\PassOptionsToPackage{numbers}{natbib}
\PassOptionsToPackage{sort}{natbib}

\documentclass{article}

\usepackage{geometry}
\usepackage[utf8]{inputenc} 
\usepackage[T1]{fontenc}   
\usepackage{fullpage}
\usepackage[numbers]{natbib}
\newcommand{\comment}[1]{}
\date{}

\usepackage{microtype}
\usepackage{graphicx}
\usepackage{subcaption}
\usepackage{booktabs} 

\usepackage{hyperref}

\usepackage{amsmath}
\usepackage{amssymb}
\usepackage{mathtools}
\usepackage{amsthm}

\usepackage[capitalize,noabbrev]{cleveref}

\theoremstyle{plain}
\newtheorem{theorem}{Theorem}

\newtheorem{lemma}{Lemma}[section]
\newtheorem{corollary}{Corollary}
\theoremstyle{definition}

\newtheorem{assumption}{Assumption}

\allowdisplaybreaks

\usepackage[utf8]{inputenc} 
\usepackage[T1]{fontenc}   
\usepackage{hyperref}       
\usepackage{url}            
\usepackage{booktabs}       
\usepackage{amsfonts}       
\usepackage{nicefrac}       
\usepackage{microtype}      
\usepackage{algorithm}
\usepackage{algorithmic}
\usepackage{algorithmic}
\usepackage{algorithmic}
\usepackage{algorithmic}
\usepackage{lipsum}
\usepackage{enumitem}
\usepackage{amsmath, amssymb, amsthm, cleveref}
\usepackage{wrapfig}
\crefname{equation}{}{}
\Crefname{equation}{}{}

\usepackage[utf8]{inputenc}
\usepackage{mathtools}
\usepackage{graphicx}
\usepackage{caption}
\usepackage{bbm}
\usepackage{array}
\usepackage{arydshln}

\usepackage{multirow}
\crefname{equation}{}{}
\Crefname{equation}{}{}

\usepackage{url}            
\usepackage{booktabs}       
\usepackage{amsfonts}       
\usepackage{nicefrac}       
\usepackage{xcolor}
\usepackage{microtype}      
\usepackage{amssymb,amsmath, amsthm}
\usepackage[utf8]{inputenc}
\usepackage{mathtools}
\usepackage{tikz}
\usetikzlibrary{calc}
\usepackage{graphicx}
\usepackage{caption}
\usepackage{bbm}
\usepackage{bm}
\usepackage{array}
\usepackage{arydshln}
\usepackage{xspace}

\usepackage{algorithmic}
\usepackage[colorinlistoftodos,prependcaption]{todonotes}

\newcommand{\ps}[1]{ \textcolor{blue}{{\small [PS: #1]}}}
\newcommand{\yj}[1]{ \textcolor{purple}{\small[\textsc{YJ:} #1]}}

\definecolor{gr}{rgb}{0.25, 0.25, 0.25}

\newcommand{\ibd}{\mathbf{I}_d}

\newcommand{\hb}{\mathbf{H}}

\newcommand{\nub}{\overline{\nu}}

\newcommand{\rb}{\mathbf{r}}
\newcommand{\ec}{M}

\newcommand{\ic}{m}

\newcommand{\kbar}{{\overline{K}}}

\newcommand{\lff}{f}
\newcommand{\lf}{F_\ic}

\newcommand{\tilo}{\tilde{\mathcal{O}}}

\newcommand{\tb}{\mathbf{t}}
\newcommand{\db}{\mathbf{d}}

\newcommand{\ssbl}{\overline{\mathbf{S}}}
\newcommand{\qbl}{\overline{\mathbf{q}}}
\newcommand{\qbt}{\widetilde{\mathbf{q}}}

\newcommand{\rbl}{\overline{\mathbf{r}}}
\newcommand{\dbl}{\overline{\mathbf{d}}}
\newcommand{\rbt}{\widetilde{\mathbf{r}}}

\newcommand{\wb}{\mathbf{w}}
\newcommand{\xb}{\mathbf{x}}

\newcommand{\tbb}{\mathbf{T}}
\newcommand{\hbbt}{\widehat{\mathbf{H}}}
\newcommand{\hbbl}{\widetilde{\mathbf{H}}}

\newcommand{\ib}{\mathbf{I}}

\newcommand{\bdat}{\mathcal{B}_\ic}

\newcommand{\defeq}{\vcentcolon=}

\newcommand{\lr}{\eta}

\newcommand{\lrt}{\widetilde{\eta}}

\newcommand{\expt}{\mathbb{E}}

\newcommand{\nbr}[1]{{\left\| {#1} \right\|}}
\newcommand{\bigto}[1]{\widetilde{\mathcal{O}}\left({#1}\right)}

\newcommand\inner[2]{\langle #1, #2 \rangle}

\newcommand{\cycp}{CyCP\xspace}

\usepackage{nicefrac}

\newcommand{\st}{\mathcal{S}}

\theoremstyle{plain}

\newtheorem*{thm*}{Theorem}

\newtheorem{coro}{Corollary}

\newcommand{\norm}[1]{\left\lVert#1\right\rVert^{2}}

\graphicspath{{Figures/}}

\crefname{equation}{}{}
\Crefname{equation}{}{}
\crefname{clm}{claim}{claims}
\Crefname{clm}{Claim}{Claims}
\Crefname{coro}{Corollary}{Corollaries}
\Crefname{sec}{Section}{Sections}
\crefname{app}{appendix}{appendices}
\Crefname{app}{Appendix}{Appendices}
\crefname{prop}{proposition}{propositions}
\Crefname{prop}{Proposition}{Propositions}
\Crefname{propty}{Property}{Properties}
\crefname{figure}{fig.}{figures}
\Crefname{figure}{Fig.}{Figures}
\crefname{defn}{definition}{definitions}
\Crefname{defn}{Definition}{Definitions}
\crefname{fact}{fact}{facts}
\Crefname{fact}{Fact}{Facts}
\crefname{appendix}{appendix}{appendices}
\Crefname{appendix}{Appendix}{Appendices}
\crefname{algo}{algorithm}{algorithms}
\Crefname{algo}{Algorithm}{Algorithms}
\crefname{algorithm}{algorithm}{algorithms}
\Crefname{algorithm}{Algorithm}{Algorithms}
\crefname{tbl}{table}{table}
\Crefname{tbl}{Table}{Table}
\crefname{table}{table}{table}
\Crefname{table}{Table}{Table}
\crefname{algorithm}{algorithm}{algorithms}
\Crefname{algorithm}{Algorithm}{Algorithms}

\crefname{conj}{conjecture}{conjectures}
\Crefname{conj}{Conjecture}{Conjectures}
\crefname{obs}{observation}{observations}
\Crefname{obs}{Observation}{Observations}

\newcommand{\bx}{{\bf x}}

\newcommand{\by}{{\bf y}}

\newcommand{\mc}{\mathcal}
\newcommand{\mco}{\mathcal O}
\newcommand{\mbb}{\mathbb}
\newcommand{\mbf}{\mathbf}
\newcommand{\mbe}{\mathbb E}

\newcommand{\lp}{\left(}
\newcommand{\rp}{\right)}

\newcommand{\lcb}{\left\{}
\newcommand{\rcb}{\right\}}

\newcommand{\lnr}{\left\|}
\newcommand{\rnr}{\right\|}
\newcommand{\lan}{\left\langle}
\newcommand{\ran}{\right\rangle}

\newcommand{\G}{\nabla}

\newcommand{\nn}{\nonumber}

\title{\textbf{On the Convergence of Federated Averaging \\
with Cyclic Client Participation}}

\author{
Yae Jee Cho \\
\small Carnegie Mellon University\\
\small \texttt{\href{mailto:yaejeec@andrew.cmu.edu}{yaejeec@andrew.cmu.edu}} 
\and
Pranay Sharma \\
\small Carnegie Mellon University\\
\small \texttt{\href{mailto:pranaysh@andrew.cmu.edu}{pranaysh@andrew.cmu.edu}} 
\and
Gauri Joshi \\
\small Carnegie Mellon University\\
\small \texttt{\href{mailto:gaurij@andrew.cmu.edu}{gaurij@andrew.cmu.edu}}\\
\and
Zheng Xu \\
\small Google Research\\
\small \texttt{\href{mailto:xuzheng@google.com}{xuzheng@google.com}} 
\and
Satyen Kale\\
\small Google Research\\
\small \texttt{\href{mailto:satyenkale@google.com}{satyenkale@google.com}} 
\and
Tong Zhang\\
\small Google Research and HKUST\\
\small \texttt{\href{mailto:tozhang@google.com}{tozhang@google.com}} 
}

\begin{document}
\maketitle


\begin{abstract}
Federated Averaging (FedAvg) and its variants are the most popular optimization algorithms in federated learning (FL). Previous convergence analyses of FedAvg either assume full client participation or partial client participation where the clients can be uniformly sampled. However, in practical cross-device FL systems, only a subset of clients that satisfy local criteria such as battery status, network connectivity, and maximum participation frequency requirements (to ensure privacy) are available for training at a given time. As a result, client availability follows a \textit{natural cyclic pattern}. 
We provide (to our knowledge) the first theoretical framework to analyze the convergence of FedAvg with cyclic client participation with several different client optimizers such as GD, SGD, and shuffled SGD. Our analysis discovers that cyclic client participation can achieve a faster asymptotic convergence rate than vanilla FedAvg with uniform client participation under suitable conditions, providing valuable insights into the design of client sampling protocols. 
\end{abstract}

\section{Introduction} \label{sec:intro}
Federated learning (FL) is a distributed learning framework that enables edge clients (e.g., mobile phones, tablets) to collaboratively train a machine learning (ML) model without sharing their local data~\cite{mcmahan2017communication}. 
In cross-device FL \citep{kairouz2019advances}, millions of mobile devices are orchestrated by a central server for training, and only a subset of client devicess will participate in each communication round due to intermittent connectivity and resource constraints~\cite{bonawitz2019towards}.

Federated Averaging (FedAvg) \citep{mcmahan2017communication} and its variants \citep{reddi2020adaptive,sahu2019federated,wang2021field} are the most popular algorithms in FL. In each communication round of the generalized FedAvg framework \citep{reddi2020adaptive,wang2021field}: 1) the server broadcasts the current model to a subset of clients, 2) clients update the model with local data and send back the local model update, and 3) the server aggregates clients' model updates and computes the new global model. This algorithm is popular in practice for various reasons including the compatibility with FL system implementation \citep{bonawitz2019towards} and additional privacy techniques such as differential privacy \citep{mcmahan2017learning} and secure aggregation \citep{bonawitz2016practical}.

The convergence of (generalized) FedAvg (also known as local SGD) has been studied in many recent works \citep{li2019on,woodworth2020minibatch,wang2022unreasonable,karimireddy2019scaffold} due to its popularity in practice. 
While these analyses tackle the theoretical challenge of data heterogeneity, they assume either full client participation where all clients will participate every round, or partial client participation where the clients are chosen uniformly at random from the entire set of clients. However, in practical cross-device FL systems, clients can only participate in training when local criteria such as being idle, plugged in for charging,
and on an unmetered network are satisfied \citep{bonawitz2019towards,hard2018federated,paulik2021federated,huba2022papaya}. Works like \citet{yang2018applied,eichner2019semi,zhu2021diurnal} observe client participation to have a diurnal pattern, and  \citet{balle2020privacy,kairouz2019advances,wang2021field} discuss the difficulty of controlling the sampling of clients for participation. Motivated by differential privacy \citep{kairouz2021practical}, \citet{mcmahan2022federated} seeks to limit the contribution of each client by allowing it to participate at most once in a large time window. For these reasons, clients typically participate in training with a \textit{cyclic pattern} in practical FL systems. 



In this work, we provide the first (to the best of our knowledge) convergence analysis of federated averaging with cyclic client participation. We consider that clients are implicitly divided into groups, and the groups become available to the server in a cyclic order. We show that for a global PL objective~\cite{haddadpour2019local}, instead of the standard $\mathcal{O}\left({1/T}\right)$ rate of error convergence achieved by FedAvg, where $T$ is the number of communication rounds, cyclic client participation can achieve a faster $\bigto{1/T^2}$ convergence under suitable conditions, where $\bigto{\cdot}$ subsumes all log-terms and constants. This key insight is similar to that obtained by a recent work \cite{yun2022shuf} on the convergence of mini-batch and local-update shuffle SGD, which shows the fast convergence of local data shuffling at clients under the full (rather than cyclic and partial) client participation setting (see \Cref{sec:relshuff} for more details). 

Our analysis framework covers several cases of cyclic participation and different client optimizers: 
1) it includes the subsampling of a subset of clients from each group that becomes cyclically available, 2) it captures how the number of groups within a cycle or the data heterogeneity characteristics of the client groups affect convergence, and 3) it covers different client local procedures including gradient descent (GD), stochastic gradient descent (SGD), and shuffled SGD (SSGD). As a result of this generality, several well-studied FedAvg variants such as standard FedAvg with partial client participation~\cite{li2019on, div2022parcli} 
, minibatch RR and local RR~\cite{yun2022shuf} can become special cases of our framework. We show that our bounds match with the bounds from prior works in these special cases, corroborating the validity of our results. We also present preliminary experimental results to demonstrate that cyclic client participation indeed achieves better performance in terms of test accuracy and training loss convergence compared to standard FedAvg. 

\section{Related Work}
\label{sec:related}


\subsection{Client Participation in FL.} 
Due to the large total number of clients in cross-device FL, it is inevitable to select only a subset of clients per training round. Therefore, there has been a plethora of work related to client participation in FL~\cite{kairouz2019advances,li2020federated_spmag}. Most work has focused on analyzing FedAvg with unbiased partial client participation~\cite{yang2021achieving_iclr,div2022parcli} and showing a convergence rate of $\mc{O}\left({1/T}\right)$. While some work in FL has also considered biased partial client participation for flexible client participation~\cite{ruan2020flexible} or loss-dependent client participation~\cite{yjc2020csfl,jack2019afl}, cyclic participation patterns have not been considered in these previous work. 

Another related line of work is the analyses on arbitrary client participation presented in recent work~\cite{wang2022arbicli,dimit2021arbifl}. \citet{wang2022arbicli} proposes classes of different client participation patterns where cyclic client participation goes under the regularized participation class. However, due to the generality of the formulation, their analysis does not capture important characteristics such as how the ordering of the clients or the number of client groups within a cycle affects the convergence. \citet{dimit2021arbifl} analyzes FedAvg with clients sending their local updates in an asynchronous manner, where each client has its own-defined cycle interval for sending its updates. However, such framework does not simulate the cyclic pattern that a realistic FL system observes where groups of clients sequentially become available to the server.   

Cyclic client participation has only recently been viewed in FL through the lens of privacy~\cite{kairouz2021privswor,choquette2022multi} and communication-efficiency~\cite{zhu2022agesel}. While \citet{kairouz2021privswor} shows that cyclic client participation can improve privacy guarantees in FL, its convergence properties are not examined. \citet{zhu2022agesel} shows that selecting clients based on their participation frequencies can speed up convergence with the rate $\mathcal{O}\left({1/TV}\right)$ where $V$ is a constant depending on the variance arising from the data heterogeneity with partial client participation. However, the exact rate of the convergence speed-up is unclear due to the lack of bounds for the variable $V$. In contrast to this prior work, we provide the convergence for cyclic client participation in FL where the speed-up rate is clear (at the rate $\bigto{1/T^2}$) and the conditions under which it can be achieved are identified. This speedup relies on analyzing FedAvg with cyclic participation from the perspective of shuffling-based methods which we explain in more detail below.

\comment{
    \begin{itemize}
        \item Initial works \cite{konecny2016federated, mcmahan2017communication}; survey papers \cite{kairouz2019advances, li2020federated_spmag}
        \item Local SGD guarantees: \cite{stich2018local, khaled2020tighter, koloskova2020unified}
        \item partial client participation: \cite{yang2021achieving_iclr}
    \end{itemize}}
\subsection{Shuffling-based methods.} \label{sec:relshuff}
The initial progress on shuffling-based methods was made by \cite{gurbuzbalaban2019convergence_siamopt, gurbuzbalaban2021random_mathprog} for strongly-convex quadratics.
The general idea in these, and subsequent works is that since shuffling-based methods involve using each component function \textit{exactly once} in each epoch, the progress made by these methods within an epoch approximates that of full-batch gradient descent. 

The literature on shuffling-based methods mainly focuses on three kinds of epochs: (i) random reshuffling (RR), where the data is shuffled after every epoch, (ii) shuffle once (SO), where the data is shuffled just once at the beginning, and (iii) incremental gradient (IG) method, in which the data is not shuffled at all, and follows a predetermined order in each epoch. We shall see in \cref{sec:pf} (and more so in \cref{theo1:GD}) that cyclic client participation essentially approximates an incremental gradient method at the level of the server, with each client interpreted as a sample.


Recent work has established upper and lower bounds for shuffle SGD under these shuffling schemes. For RR (and SO), \cite{safran2020good_colt} showed a lower bound of $\mco (\frac{1}{n^2 K^2} + \frac{1}{K^3})$ for strongly convex quadratic ($K$ is the number of epochs, $n$ is the number of samples), while \cite{rajput2020closing_icml} showed $\mco (\frac{1}{n K^2})$ lower bound for general strongly-convex $F$ with smooth $\{ f_i \}$. Matching upper bounds have been achieved in the large epoch regime by \cite{ahn2020sgd_neurips, mishchenko2020random_neurips} for smooth PL functions and \cite{nagaraj2019sgd_icml} for smooth strongly convex functions. For IG, \cite{nguyen2021shuffl_jmlr} showed $\mco(\nicefrac{1}{K^2})$ rate for strongly-convex $F$ with smooth $\{ f_i \}$. The improved dependence on $K$ in all these works requires $K$ to be larger than $\mco(\kappa^a)$, where $\kappa$ is the condition number of the problem, and $a \in [1,2]$. This large epoch requirement has been shown to be essential in \cite{safran2021random_neurips}.



\section{Problem Formulation}\vspace{-0.3em}
\label{sec:pf}

\textbf{System Model and Objectives.} Consider a cross-device FL setting where we have $M$ total clients. Each client $\ic\in[\ec]$ has its local training dataset $\bdat$ and its corresponding local empirical loss function $F_\ic(\wb)=\frac{1}{|\bdat|}\sum_{\xi \in \bdat} \ell(\wb,\xi)$, where $\ell(\wb,\xi)$ is the loss value for the model $\wb \in \mathbb{R}^d$ at data sample $\xi$. The optimization task is identical to that of standard FL~\cite{mcmahan2017communication, kairouz2019advances} where the global objective is $F(\wb)=\frac{1}{M}\sum_{\ic=1}^{\ec}F_\ic(\wb)$ and the server aims to find the model that achieves $\min_{\wb}F(\wb)$. Throughout the paper, all vector and matrix norms are Euclidean and spectral norms, respectively. 

\begin{wrapfigure}{!r}{0.5\textwidth} \small
\centering \vspace{-2em}
\includegraphics[width=0.48\textwidth]{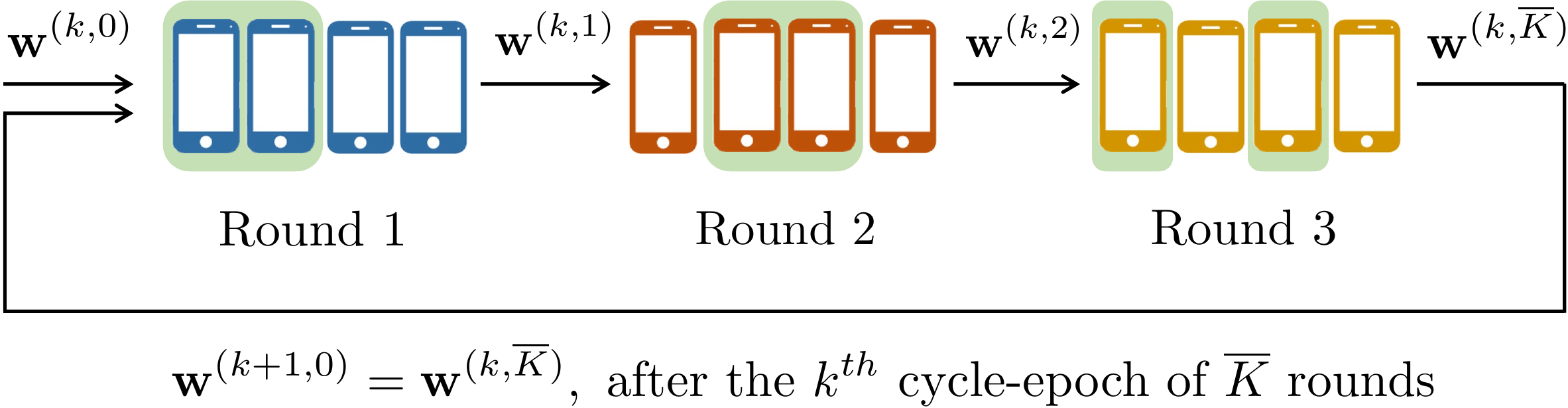}
\caption{Illustration of cyclic client participation (CyCP) with $M = 12$ clients divided into $\kbar = 3$ groups. In each communication round, $N = 2$ clients are selected for training from the client group available at that time. All groups are traversed once in a \emph{cycle-epoch} consisting of $\kbar$ communication rounds. \label{fig:swor_illust}} \vspace{-2em}
\end{wrapfigure}

\textbf{Cyclic Client Participation (\cycp).} We consider that the $M$ clients are divided into $\kbar$ non-overlapping client groups such that each group contains $M/\kbar$ clients, as illustrated in \Cref{fig:swor_illust}. The client groups are denoted by $\sigma(i),~i\in[\kbar]$, where each $\sigma(i)$ contains the associated clients' indices. The groups and the order in which they are traversed by the server (say, $\sigma(1),~...~,\sigma(\kbar)$) are pre-determined and fixed throughout training to simulate a cyclic structure of client participation. In each communication round, once a client group $\sigma(i)$ becomes available, the server selects a subset of $N$ clients from $\sigma(i)$ uniformly at random without replacement. As a result of the cyclic structure, and subsampling within each available client group, once selected, a client cannot participate in training at least for the next  $\kbar-1$ rounds. For brevity, we call this cyclic client participation framework as \cycp throughout the paper.

Observe that the CyCP framework reflects several practical FL scenarios mentioned in \cref{sec:intro}. Each client can participate at most once in $\kbar$ consecutive communication rounds in CyCP, which satisfies the privacy requirements~\cite{kairouz2021privswor,choquette2022multi}. A timer on each client can be used to enforce that clients can only participate in training again after a certain period, and this period corresponds to $\kbar$. Even without enforcing the timer-based criterion, CyCP captures the natural participation pattern due to clients coming from different time zones or preference of charging their devices~\cite{paulik2021federated,huba2022papaya,yang2018applied}. 

We introduce the term ``cycle-epoch'' to refer to the interval in which the server goes through all the client groups $\{ \sigma(i),~i\in[\kbar] \}$ once sequentially. In other words, each cycle-epoch consists of $\kbar$ communication rounds. Formally, we set $k$ as the index for the cycle-epoch, and $i\in[\kbar]$ as the index for the currently available client group within a cycle-epoch. The server sends the global model $\wb^{(k,i-1)}$ to the set $\st^{(k,i)}$ of $N$ clients, selected from the client group $\sigma(i)$, to perform local training. We consider three different types of client local updates which we explain in detail below.  


\textbf{Client Local Update.} Each client $m \in \st^{(k,i)}$ initializes its local model as $\wb_m^{(k,i-1,0)}=\wb^{(k,i-1)}$ and performs local update(s).
 The global model is updated as:
\begin{align}
\wb^{(k,i)}=\wb^{(k,i-1)}+\Delta^{(k,i-1)}
\end{align}
where $\Delta^{(k,i-1)}$ is the aggregate of local updates from clients in $\st^{(k,i)}$. We consider three different client local update procedures for our \cycp framework. 
\vspace{-1em}
\begin{itemize}[leftmargin=*]
    \item[(i)] Local Gradient Descent (GD): Selected clients perform a single GD step to update their local model. Therefore,
    \begin{align}
        \Delta^{(k,i-1)}=-\frac{\eta}{N} \sum_{m\in\st^{(k,i)}}\nabla F_m(\wb^{(k,i-1)}) \nn
    \end{align}
The resulting global algorithm is referred to as FedSGD in \citep{mcmahan2017communication}.

   \item[(ii)] Local Stochastic Gradient Descent (Local SGD): To avoid the cost of computing full gradients, 
   each client performs $\tau$ local updates to its model using stochastic gradients $\nabla F_m (\wb_m^{(k,i-1,l)} ,\xi_m^{(k,i-1,l)})$ 
   computed using a minibatch $\xi_m^{(k,i-1,l)}$ sampled uniformly at random from client $\ic$'s local dataset $\mathcal{B}_\ic$. 
   Thus, the client's model update is 
   \begin{align}
    \Delta^{(k,i-1)}\hspace{-0.3em}=\hspace{-0.3em}-\frac{\lr}{N}\hspace{-0.5em}\sum_{m\in\mathcal{S}^{(k,i)}}\hspace{-0.3em}\sum_{l=0}^{\tau-1} \nabla F_m(\wb_m^{(k,i-1,l)},\xi_m^{(k,i-1,l)})\nn
     \end{align} 
     with \scalebox{0.9}{$\wb_m^{(k,i-1,l+1)}=\wb_m^{(k,i-1,l)}-\lr\nabla F_m(\wb_m^{(k,i-1,l)},\xi_m^{(k,i-1,l)})$.}
   
   \item[(iii)] Local Shuffled SGD (SSGD): Recent works in FL \cite{yun2022shuf,malinovsky2021random} propose the use of local SSGD, where clients partition their local datasets into $B$ disjoint \emph{components}, that is, the local loss at client $m$ can be expressed as $F_m(\wb)=\frac{1}{B}\sum_{l=0}^{B-1} F_{m,l}(\wb)$. We define $\mathcal{P}_B$ to be the set of all permutations of $\{0,...,B-1\}$. In each round, the client performs local updates by going over all the components, in an order decided by the random permutation $\pi_m^k\sim\text{Unif}(\mathcal{P}_B)$. 
   The resulting model update is 
\begin{align}
\Delta^{(k,i-1)}=-\frac{\lr}{N}\sum_{m\in\mathcal{S}^{(k,i)}}\sum_{l=0}^{B-1}\nabla F_{m,\pi_m^k(l)}(\wb_m^{(k,i-1,l)}) \nn
\end{align}
with \scalebox{0.9}{$\wb_m^{(k,i-1,l+1)}=\wb_m^{(k,i-1,l)}-\lr\nabla F_{m,\pi_m^k(l)}(\wb_m^{(k,i-1,l)})$.}
\end{itemize} 
Further details of our framework of FL with \cycp are shown in \Cref{algo1}. 


\textbf{Special Cases of \cycp.} The \cycp framework covers different algorithms such as standard FedAvg with partial client participation or minibatch RR and local RR presented in \cite{yun2022shuf}. When $\kbar=1$, the \cycp setting becomes \emph{standard FedAvg} with partial client participation where in each round, $N$ clients are sampled from the same entire client population. 
When $\kbar=1$ and $N=M$, \cycp with local GD becomes identical to minibatch RR \cite{yun2022shuf} with $M$ clients, each with a single component. Both converge exponentially fast to the optimum.
Another special case is when we have $\kbar=1$ with $N=M$ but for local SSGD, in which case we have local RR \cite{yun2022shuf} with $B$ components at each client, with synchronization of the updates happening for every $B$ components. We show in \Cref{sec:theo} that our theoretical results match the bounds accordingly for these special cases.

\setlength{\textfloatsep}{1em}
\begin{algorithm}[!h] \small
\caption{\small \cycp Framework in FL}\label{algo1}
\renewcommand{\algorithmicloop}{\textbf{Global server do:}}
\begin{algorithmic}[1]
\STATE \textbf{Input: }Global Model $\wb^{(1,0)}$, Client groups $\sigma(i),i\in[\kbar]$ 
\STATE \textbf{Output: }Global Model $\wb^{K+1,0}$
\STATE \colorbox{green!8}{\hrulefill For $k\in[K]$ cycle-epochs do:~~~~\small{\# \textit{Cyclic Participation}}}
\\
\STATE \colorbox{green!8}{\hrulefill \hspace*{1em} \hrulefill For $i\in[\kbar]$ do:~~~~~~~~~~~~~{\small{\# $T=K\kbar$ \textit{comm. rounds}}}}\\
\STATE \hspace*{2em} Sample $N$ clients from client set $\sigma(i)$ uniformly at random w/o replacement to get client set $\mathcal{S}^{(k,i)}$. 
\STATE \hspace*{2em} Send global model $\wb^{(k,i-1)}$ to clients in $\mathcal{S}^{(k,i)}$.\\
\STATE \hspace*{2em} {Clients $\ic\in\mathcal{S}^{(k,i)}$ in parallel do:}\\
\STATE \hspace*{3em} {$\wb_m^{(k+1,0)} \leftarrow$ \textbf{LocalUpdate}($m,~\wb^{(k,i-1)}, case$)}
\\
\STATE \hspace*{2em} $\wb^{(k,i)}=\frac{1}{N}\sum_{m\in\st^{(k,i)}}\wb_\ic^{(k+1,0)}$\\
\STATE \hspace*{1em} $\wb^{(k+1,0)}=\wb^{(k,\kbar)}$\\
\hspace*{1em}\\
\textbf{LocalUpdate}($m,~\wb,~case$):
\STATE \hspace*{1em} {Set local model $\wb_\ic=\wb$}\\
\STATE {\hspace*{1em} {if $case==LocalGD$:}}
\STATE {\hspace*{2em}Update $\wb_\ic\leftarrow\wb_\ic-\lr\nabla F_\ic(\wb_\ic)$} \\
\STATE {\hspace*{1em} {elif $case==Local SGD$:}}
\STATE {\hspace*{2em}For $j\in[\tau]$ do:}\\
\STATE {\hspace*{3em}Sample mini-batch $\xi$ from local dataset $\mathcal{B}_m$} \\
\STATE {\hspace*{3em}Update $\wb_\ic\leftarrow\wb_\ic-\lr\nabla F_\ic(\wb_\ic,~\xi_\ic)$} \\
\STATE {\hspace*{1em} {elif $case==Shuffled SGD$:}}\\
\STATE {\hspace*{2em}Sample $\pi_m^k\sim\text{Unif}(\mathcal{P}_B)$}\\
\STATE {\hspace*{2em}For $j\in[B]$ do:}
\STATE {\hspace*{3em}Update $\wb_\ic\leftarrow\wb_\ic-\lr\nabla F_{\ic,{\pi_\ic^k(j-1)}}(\wb_\ic)$} 
\end{algorithmic} 
\end{algorithm}

\section{Convergence Analysis}
\label{sec:theo}
In this section, we provide and compare the convergence bounds for \cycp in FL for the three client local update methods described above, and provide insights into how the achieved complexities with \cycp ($\kbar>1$) compare with standard FedAvg ($\kbar=1$). 
All the proofs are deferred to \Cref{app:CCP_localGD}-\ref{app:ccp+ssgd}.

\subsection{Assumptions}
\label{sec:theo-ass}
First, we present the assumptions used for the convergence guarantees in this work. 
\begin{assumption}[Smoothness of $F_m(\wb),~\forall~m$] The clients' local objective functions $F_1(\wb),~...,F_M(\wb)$, are all $L$-smooth, that is, $\|\nabla F_m(\wb)-\nabla F_m(\wb')\| \leq L \|\wb-\wb'\|$ for all $m$, $\wb$ and $\wb'$.
\label{as1}
\end{assumption} 
\begin{assumption}[$\mu$-Polyak-Łojasiewicz $F(\wb)$]
For some $\mu > 0$, the global objective satisfies $\frac{1}{2}\|\nabla F(\wb)\|^2\geq\mu(F(\wb)- \min_{\wb'} F(\wb'))$ for all $\wb$. \label{as2}
\end{assumption}
\Cref{as1}, \ref{as2} are common in the optimization and FL literature~\cite{haddadpour2019convergence,karimi2020linpl,haddadpour2019local,gower2021sgdpl}. While we restrict ourselves to PL functions for brevity and for ease of comparison with prior work \cite{yun2022shuf}, the analyses can be generalized to general nonconvex functions using techniques proposed in \cite{li2021convergence_KL}. 

Next, we present the assumptions over the client groups $\sigma(1), \dots, \sigma(\kbar)$.


\begin{assumption}[Intra-group \& Inter-Group Data Heterogeneity] There exist constants $\gamma,~\alpha\geq0$, such that for all $\wb$, for all $i\in[\kbar]$ and for all $m\in\sigma(i)$, $\| \nabla F_m(\wb)-\frac{1}{|\sigma(i)|}\sum_{m\in\sigma(i)}\nabla F_m(\wb) \| \leq \gamma$, and $\| \frac{1}{|\sigma(i)|}\sum_{m\in\sigma(i)}\nabla F_m(\wb)-\nabla F(\wb) \| \leq \alpha$.
\label{as3}
\end{assumption}

\Cref{as3} 
bounds the data heterogeneity across clients \textit{within a group} by $\gamma$ and the data heterogeneity \textit{across groups} by $\alpha$.
%
\Cref{as3} also implies the commonly used data heterogeneity assumption used in previous  FL literature~\cite{yu2019linear,koloskova2020unified,wang2020tackling,wang2020slowmo,reddi2020adaptive} as follows: 
\begin{lemma}
If \Cref{as3} is true, there exists $\nu=\gamma+\alpha\geq0$ such that $\nbr{\nabla F_m(\wb)-\frac{1}{M}\sum_{i=1}^M\nabla F_i(\wb)}\leq \nu$ for all clients $m\in[M]$, and for all $\wb$. \label{lem1}
\end{lemma}
Using \Cref{as3} instead of the standard assumption allows us to derive tighter convergence bounds in terms of $\gamma$ and $\alpha$, and separate the effect of the two kinds of heterogeneity, as we discuss in subsequent sections.

\subsection{Convergence for \cycp with Local GD}
First, we start with providing the convergence of the global model in \cycp with local GD.
\begin{theorem}[Convergence with CyCP+GD] With Assumptions \ref{as1}, \ref{as2}, \ref{as3}, the choice of step-size $\lr=\log(MT^2/\kbar^2)/\mu N T$, and number of communication rounds $T\geq 7\kappa\kbar\log{(MT^2/\kbar^2)}$ where $\kappa=L/\mu$: 
\begin{align}
 &\expt[F(\wb^{(K,0)})]-F^*\leq\frac{\kbar^2(F(\wb^{(0,0)})-F^*)}{MT^2}+\tilo\left(\frac{\kappa^2(\kbar-1)^2\alpha^2}{\mu T^2}\right)+\tilo\left(\frac{\kbar\kappa\gamma^2}{\mu NT}\left(\frac{M/\kbar-N}{M/\kbar-1}\right)\right), \label{eq:theo-locgd-comm}
\end{align}
where $\tilo(\cdot)$ subsumes all log-terms and constants. 
\label{theo1:GD}
\end{theorem}  
Although it might appear that the bound becomes worse with increasing $\kbar$ due to it appearing in the numerators of the terms, since $T=K \kbar$ (see \Cref{algo1}), a large $\kbar$ has no adverse impact on the convergence.


\textbf{Convergence Dependence on $\gamma^2$ and $\alpha$.} \Cref{theo1:GD} shows that CyCP+GD converges at the rate of $\tilo\left(\frac{\kbar\kappa\gamma^2}{\mu NT}\left(\frac{M/\kbar-N}{M/\kbar-1}\right)\right)$ which depends on $\gamma^2$, the intra-group data heterogeneity. Consequently, a large intra-group data heterogeneity $\gamma$ leads to worse convergence. Conversely, if $\gamma \simeq 0$, CyCP+GD can achieve $\tilo\left(1/T^2\right)$ convergence, due to the $\tilo\left(1/T\right)$ domintnat term becoming zero. Hence, in the \cycp settings where clients within the same group have similar data distributions (i.e., $\gamma$ is close to $0$)
, CyCP+GD can yield a faster convergence rate compared to standard FedAvg ($\kbar=1$). An example of a setting where this can naturally occur in realistic FL scenarios is when the cyclic patterns follow the diurnal-nocturnal pattern, also shown in~\cite{zhu2021diurnal}. It is also worth noting that the term with inter-group data heterogeneity $\alpha$ in \Cref{theo1:GD} decays at the rate of $\bigto{1/T^2}$.
Therefore, in \cycp settings, the intra-group data heterogeneity $\gamma$ has a more significant contribution to the convergence error than the inter-group data heterogeneity $\alpha$.

\textbf{Convergence Dependence on $\kbar$.} \Cref{theo1:GD} also shows that even for $\gamma\neq0$, CyCP+GD can gain a $\tilo\left(1/T^2\right)$ convergence rate when $\kbar = M/N$. 
This is a faster rate than the standard FedAvg (the setting with $\kbar=1$) which has $\tilo\left(1/T\right)$ rate. While the convergence rates for the cases of $\kbar=1$ and $\kbar=M/N$ are clear from \Cref{theo1:GD}, it is yet unclear what happens in the middle regime of $1<\kbar<M/N$. For this, we compare the total cost of CyCP+GD and standard FedAvg to achieve $\epsilon$ error. We define the total communication and computation cost in one communication round of GD (which involves computing $N$ gradients and communicating $N$ vectors to the server) as $c_{\text{GD}}$
. Then taking into account only the dominant term in \Cref{eq:theo-locgd-comm}, the total cost $C_{\text{GD}}$ to achieve an $\epsilon$ error is
\begin{align}
C_{\text{GD}}(\epsilon)=\bigto{\frac{c_{\text{GD}}\kbar\gamma^2}{\epsilon N}\left(\frac{M/\kbar-N}{M/\kbar-1}\right)} \label{eq:totalc:gd}
\end{align}
We compare $C_{\text{GD}}(\epsilon)$ with $\kbar>1$ and $\kbar=1$ denoted as $C_{\text{GD}|\kbar>1}(\epsilon),~C_{\text{GD}|\kbar=1}(\epsilon)$ respectively and derive the following result:
\begin{corollary} For the total cost defined as \Cref{eq:totalc:gd}, for $\kbar < M/N$, we have that $C_{\text{GD}|\kbar>1}(\epsilon)>C_{\text{GD}|\kbar=1}(\epsilon)$. 
\label{theo:1-1:kbar}
\end{corollary}

\Cref{theo:1-1:kbar} shows that with the number of groups set to the middle range, i.e., $1<\kbar<M/N$, \cycp does not incur a smaller cost compared to standard FedAvg ($\kbar=1$). Hence, for CyCP+GD to incur a lower cost
compared to standard FedAvg, the \textit{necessary condition} is having $\kbar=M/N$. There may be some scenarios in which $\kbar$ is a naturally occurring quantity that the server does not have control over. It is worth noting that in these cases, $N$ which is the number of selected clients per round, can be chosen accordingly by the server to pay a lower cost than standard FedAvg.

\paragraph{Matching Bounds with Minibath RR~\cite{yun2022shuf}.} Recall that with $\kbar=1$ and $N=M$, CyCP+GD in \Cref{algo1} becomes analogous to the minibatch RR algorithm with $M$ clients where each client has a single component. In this case, our bound in \Cref{theo1:GD} is only left with the first term that decays with the rate $\bigto{1/MT^2}$ which exactly matches minibatch RR's bound in [Theorem 1]~\cite{yun2022shuf} which shows exponential convergence.

\subsection{Convergence  for \cycp with Local SGD}
Next, we present the convergence for \cycp with local SGD. Local SGD introduces additional technical challenges compared to GD for deriving the convergence analysis and requires the following additional assumption over the stochastic gradients:

\begin{assumption}[Bounded Variance]
For local objective $F_m(\wb)$, the local stochastic gradient $\nabla F_m(\wb,\xi_\ic)$ computed using a mini-batch $\xi_m$, sampled uniformly at random from $\bdat$, has bounded variance, that is, 
$\expt[\|\nabla F_m(\wb,\xi_\ic)-\nabla \lf(\wb)\|^2]\leq\sigma^2$, for all $m\in[M]$. 
\label{as5}
\end{assumption}
\Cref{as5} is commonly used in the stochastic optimization literature \cite{stich2018local, basu2019qsparse, li2019on, ruan2020flexible}. Now we present the convergence bound for local SGD.
\begin{theorem}[Convergence with CyCP+SGD] With Assumptions \ref{as1}, \ref{as2}, \ref{as3}, and \ref{as5} and step-size $\lr=\log(MT^2/\kbar^2)/\tau\mu NT$, for $T\geq 10\kappa\kbar\log{(MT^2/\kbar^2)}$ communication rounds where $\kappa=L/\mu$, the convergence error is bounded as:
\begin{align}
\begin{aligned}
    \expt[F(\wb^{(K,0)})]-F^*\leq\frac{\kbar^2(F(\wb^{(0,0)})-F^*)}{MT^2}+\bigto{\frac{\kappa^2 \kbar(\kbar-1)\alpha^2}{\mu T^2}}+\bigto{\frac{\kbar\kappa\gamma^2}{\mu N T}\left(\frac{M/\kbar-N}{M/\kbar-1}\right)} \\
    +\bigto{\frac{\kbar\kappa\sigma^2}{\mu\tau N T}}+\bigto{\frac{\kappa^2(\tau-1)\nu^2}{\mu\tau N^2T^2}}
 \label{eq:8-0-0}
\end{aligned}
\end{align}
where $\tilo(\cdot)$ subsumes all log-terms and constants. \label{theo2:locSGD}
\end{theorem}
Again, although it might appear that the bound becomes worse with increasing $\kbar$, since $T=K \kbar$ in \Cref{algo1},
a large $\kbar$ has no adverse impact on the convergence.


\textbf{Convergence Dependence on $\gamma^2$ and $\sigma^2$.} In \Cref{theo2:locSGD}, the dominant $\bigto{1/T}$ terms are dependent on two factors: the intra-group data heterogeneity $\gamma^2$ (which also appeared for local GD) and the stochastic gradient variance $\sigma^2$. Due to this, for $\sigma > 0$, even with $\kbar=M/N$ or $\gamma^2 \simeq 0$, we do not achieve the $\bigto{1/T^2}$ convergence rate, as we did in the local GD case. Hence, even with \cycp, the best we can achieve when using local SGD is the convergence rate of $\bigto{1/T}$. Seeing this result, one might wonder if there is any advantage at all of performing Local SGD with \cycp. We answer this question below by comparing the cost of CyCP with Local SGD to that of standard FedAvg.

\textbf{Does \cycp ($\kbar>1$) with Local SGD Ever Help for FL?} At first glance of \Cref{theo2:locSGD} one may think that \cycp does not improve the convergence rate with local SGD case due to stochastic gradient variance appearing in one of the dominant terms $\bigto{\frac{\kbar\kappa\sigma^2}{\mu\tau N T}} + \bigto{\frac{\kbar\kappa\gamma^2}{\mu N T}\left(\frac{M/\kbar-N}{M/\kbar-1}\right)}$. However, we show in \Cref{theo:2-1:kbar} that this is not always the case. Similar to how we defined $c_{GD}$ in the previous section, we definte the total communication and computation cost in one communication round with Local SGD as $c_{\text{SGD}}$. Formally, taking into account only the dominant terms in \Cref{eq:8-0-0}, the total cost to achieve an $\epsilon$ error for the local SGD case is:
\begin{align}
C_{\text{SGD}}(\epsilon)\hspace{-0.2em}=\hspace{-0.2em}\bigto{\hspace{-0.2em}\frac{c_{\text{SGD}}\kbar\gamma^2}{\epsilon N}\hspace{-0.2em}\left(\frac{M/\kbar-N}{M/\kbar-1}\right)\hspace{-0.2em}}\hspace{-0.2em}+\hspace{-0.2em}\bigto{\hspace{-0.2em}\frac{c_{\text{SGD}}\sigma^2\kbar}{\epsilon N\tau}\hspace{-0.2em}} \label{eq:totalc:lsgd}
\end{align}
We denote the costs for \cycp and standard FedAvg as $C_{\text{SGD}|\kbar=1},~C_{\text{SGD}|\kbar>1}$ respectively. Now we show the conditions to have $C_{\text{SGD}|\kbar>1}<C_{\text{SGD}|\kbar=1}$, i.e., have \cycp incur a lower cost than standard FedAvg.

\begin{corollary}
Suppose we have $N = M/\kbar$, and the intra-group data heterogeneity $\gamma$ satisfies $\gamma^2 \geq M\sigma^2/N\tau$. Then, we get $C_{\text{SGD}|\kbar} \leq C_{\text{SGD}|\kbar=1}$. \label{theo:2-1:kbar}
\end{corollary}

\Cref{theo:2-1:kbar} shows that \cycp+SGD can indeed incur a lower cost to achieve $\epsilon$ error compared to standard FedAvg ($\kbar=1$) when the intra-group data heterogeneity is sufficiently larger than the stochastic gradient variance divided by the number of local iterations. Note that the condition $\gamma^2 \geq M\sigma^2/N\tau$ in \Cref{theo:2-1:kbar} can be satisfied by increasing the minibatch size $b$ (which decreases the variance $\sigma^2$) or increasing the number of local iterations $\tau$.
\vspace{-1em}

\paragraph{Matching Bounds with Standard FedAvg.} For $\kbar=1$, CyCP+SGD recovers Standard FedAvg with Local SGD, and our bound in \Cref{eq:8-0-0} with full client participation follows the order of $\bigto{\kappa^2 \nu^2/\mu M T^2}+\bigto{\kappa\sigma^2/\mu\tau M T}$. We show that this bound matches the last iterate bound in \cite{qu2020lin} which assumes full client participation, and bounded norm of the stochastic gradient with parameter $G$ that reads
$\bigto{\kappa^2 G^2/\mu T^2} + \bigto{\kappa \sigma^2/\mu \tau T}$ where their learning rate doesn't decay with $M$ as our case, leading to the lack of the $1/M$ in their bounds. The difference in the first term is due to their work assuming the bounded norm of the stochastic gradient $G$ while we assume only the bounded variance of the stochastic gradient.  


\subsection{Convergence for \cycp with Local SSGD}
For the last scenario of \cycp, we present results on the convergence properties of the global model with local SSGD. We slightly modify the local loss definition of each client as  $F_m(\wb)=\frac{1}{B}\sum_{l=0}^{B-1} F_{m,l}(\wb)$ so that each client has $B$ loss components.
For local SSGD, clients perform local updates sequentially over $F_{m,\pi_m(l)}(\wb),~l\in[0,...,B-1]$ where $\pi_m\sim\text{Unif}(\mathcal{P}_B)$ is a random permutation over the $B$ components and $\pi_m(l)$ denotes the $l$-th element of this permutation. For SSGD, in lieu of \cref{as5}, we need the following intra-client component heterogeneity assumption, which is commonly used in the shuffled SGD literature~\cite{yun2022shuf,malinovsky2021random}: 

\begin{assumption}[Intra-Client Component Heterogeneity]
There exists a constant $\overline{\nu}\geq0$ such that for each client $m\in[M]$, and each component $l\in[0,...,B-1]$ of its local dataset, $\| \nabla F_{m,l}(\wb)-\nabla F_m(\wb) \| \leq \overline{\nu}$, for all $\wb$. \label{as6}
\end{assumption}

Now we present our convergence results for the SSGD case. 
\begin{theorem}[Convergence with CyCP+SSGD] With Assumptions \ref{as1}, \ref{as2}, \ref{as3}, and \ref{as6}, and $\lr=\log(MBT^2/\kbar^2)/\mu BT$ and cycle-epoch $T\geq 10\kappa\kbar\log{(MBT^2/\kbar^2)}$ where $\kappa=L/\mu$, with probability at least $1-\delta$, the convergence error is bounded as:\vspace{-0.5em}
\begin{align}
\begin{aligned}
    \expt[F(\wb^{(K,0)})]-F^*\leq\frac{\kbar^2(F(\wb^{(0,0)})-F^*)}{MBT^2}+\bigto{\frac{\kappa^2 (B-1)^2\nu^2}{\mu B^2 T^2}}+\bigto{\frac{\kappa^2(\kbar-1)^2\alpha^2}{\mu T^2}}
    \\
    +\bigto{\frac{\kappa^2 \overline{\nu}^2}{\mu T^2}\left(\frac{(B^{3/2}-1)^2}{B^4}+\frac{(B-1)^2}{B^3}\right)}+\bigto{\frac{\kappa\kbar\gamma^2}{\mu N T}\left(\frac{M/\kbar-N}{M/\kbar-1}\right)}
    \end{aligned} \label{eq:8-0-1}
\end{align}
where $\tilo(\cdot)$ subsumes all log-terms and constants. \label{theo3:ssgd}
\end{theorem}


Again, since $T = K \kbar$, increasing $\kbar$ does not impact the bound above adversely.

\textbf{Dependency on $\kbar$ is Identical to CyCP+GD.} \Cref{theo3:ssgd} shows that the dominant term is $\bigto{\frac{\kappa\kbar\gamma^2}{\mu N T}\left(\frac{M/\kbar-N}{M/\kbar-1}\right)}$ which is the same dominant term for the local GD's convergence rate in \Cref{theo1:GD}. This dominant term exists as long $\gamma\neq0$ or $\kbar\neq M/N$. Hence, for $1<\kbar<M/N$, as with local GD (see \Cref{theo:1-1:kbar}), \cycp+SSGD does not yield a lower cost than standard FedAvg.
Also like CyCP+GD, $\kbar=M/N$ is \textit{necessary} for CyCP+SSGD to get any cost reduction compared to standard FedAvg. We give a more detailed comparison between the costs of the two different local update procedures below.

\paragraph{Can CyCP+SSGD be better than CyCP+GD?} We have seen above that CyCP+SSGD and CyCP+GD converge at the same rate due to the same dominant term which is non-zero for $\gamma^2\neq0$ and $\kbar<M/N$. For $\kbar=M/N$, however, the dominant term goes to $0$ and CyCP+SSGD and CyCP+GD become comparable. Again, we define the total communication and computation cost in one communication round with local SSGD as $c_{\text{SSGD}}$. Then, for $\kbar=M/N$, we have that the total cost for CyCP+GD and CyCP+SSGD to achieve an $\epsilon$ error denoted as $C_{\text{GD}|\kbar=M/N}(\epsilon),~C_{\text{SSGD}|\kbar=M/N}(\epsilon)$ respectively, is
\begin{flalign}
&C_{\text{GD}|\kbar=M/N}(\epsilon)\hspace{-0.1em}=\hspace{-0.1em}\bigto{\hspace{-0.1em}\frac{c_{\text{SSGD}}}{\sqrt{\epsilon}}\left(\hspace{-0.1em}\frac{\kbar}{\sqrt{M}}+{\kbar\alpha}\hspace{-0.1em}\right)\hspace{-0.1em}} \label{eq:ccpgd}\\
&\begin{aligned}
&C_{\text{SSGD}|\kbar=M/N}(\epsilon)\hspace{-0.1em}=\hspace{-0.1em}\bigto{\hspace{-0.1em}\frac{c_{\text{SSGD}}}{\sqrt{\epsilon}}\hspace{-0.1em}\left(\hspace{-0.1em}\hspace{-0.1em}\frac{\kbar}{\sqrt{MB}}+{\nu}+{\kbar\alpha}+\frac{\nub}{\sqrt{B}}\hspace{-0.1em}\right)\hspace{-0.1em}}
\end{aligned}\label{eq:ccpssgd}
\end{flalign}
With these costs, we have that CyCP+SSGD only incurs a lower cost than the CyCP+GD case for $\kbar=M/N$ when 
\begin{align}
1-\frac{1}{\sqrt{B}}-\nu-\frac{\nub}{\sqrt{B}}>0,~B>1 \label{eq:condssgd}
\end{align}
The only certain condition in which \Cref{eq:condssgd} can be satisfied is when $\nu\simeq\nub\simeq0$. Hence, with $\kbar=M/N$, only when there is close to 0 data heterogeneity across clients and intra-client component heterogeneity within each client is when CyCP+SSGD can incur a lower cost compared to CyCP+GD. This also aligns well with the theoretical results presented in \cite{yun2022shuf,woodworth2020local}.

\textbf{Matching Bounds with Special Cases of CyCP+SSGD.} Special cases of CyCP+SSGD become analogous to different algorithms such as CyCP+GD or local RR \cite{yun2022shuf}. With $B=1$, since $\nub=0$ , we recover CyCP+GD (\Cref{theo1:GD}). Another special case is when $\kbar=1,~N=M$ for CyCP+SSGD. In this case, for each communication round, we have full client participation where each client performs SSGD over its local components, which becomes analogous to the local RR algorithm proposed and theoretically analyzed in \cite{yun2022shuf}. Note that in this case, the bound in \Cref{theo3:ssgd} matches the bound for local RR presented in \cite{yun2022shuf}. We present a more detailed comparison in the following paragraph. 

\paragraph{Is CyCP+SSGD better than Local RR?} Since local RR assumes full client participation, 
for a fair comparison, we compare local RR with CyCP+SSGD with $\kbar=M/N$. 
The dominant term in CyCP+SSGD's convergence bound in \Cref{theo3:ssgd} becomes zero and we are left with the terms having a convergence rate of $\bigto{1/T^2}$. We compare the two algorithms in terms of the total cost to achieve $\epsilon$ error, where the total cost for local RR is: 
\begin{flalign}
C_{\text{LocalRR}}(\epsilon)=\bigto{\frac{\kbar c_{\text{SSGD}}}{\sqrt{\epsilon}}\left(\frac{1}{\sqrt{MB}}+\nu+\frac{\overline{\nu}}{\sqrt{B}}\right)} \label{eq:localrrfcp}
\end{flalign}
Note that for local RR, both the computation and communication cost is $\kbar$ times that of CyCP+SSGD, since in local RR all clients' participate in each communication round while in CyCP+SSGD only $M/\kbar$ clients participate per communication round. With CyCP+SSGD's cost in \Cref{eq:ccpssgd} for $\kbar=M/N$ and local RR's cost in \Cref{eq:localrrfcp}, we have the following theoretical result that compares the two methods:
\begin{corollary} For a sufficiently large $M$ such that 
\begin{align}
M>N\left(1+\frac{\alpha}{\gamma+\frac{\nub}{\sqrt{B}}}\right)
\end{align}
where $\kbar=M/N$ for CyCP+SSGD, CyCP+SSGD is \textbf{always} better than local RR in terms of the total cost taken to gain epsilon error.   \label{theo:compssgdlrr}
\end{corollary} \vspace{-0.5em}

\Cref{theo:compssgdlrr} shows that with $\kbar=M/N$, and sufficiently large $M$, CyCP+SSGD is always preferred over local RR to achieve a lower cost. Since $\kbar=M/N$, a larger $M$ indicates a larger $\kbar$. Observe that the lower bound on $M$ in \Cref{theo:compssgdlrr} becomes smaller for a smaller $\alpha$ and a larger $\gamma$. This indicates that it is preferred that the client groups have smaller inter-group data heterogeneity but larger intra-group data heterogeneity, for CyCP+SSGD to beat local RR.


\section{Experimental Results} 
\paragraph{Setup.} We train ML models on standard datasets using FedAvg with \cycp for different client local updated procedures to see how cyclicity affects the performance of FL. We experiment with image classification using an MLP for the FMNIST~\cite{xiao2017fmnist} dataset and EMNIST dataset~\cite{cohen2017emnist} with 62 labels where we have 100 and 500 clients in total and select 5 and 10 clients per communication round respectively. We use the Dirichlet distribution $\text{Dir}_{K}(\alpha)$ \citep{hsu2019noniid} to partition the data across clients where $\alpha$ determines the degree of the data heterogeneity across clients. Smaller $\alpha$ indicates larger data heterogeneity. We experiment with three different seeds for the randomness in the dataset partition across clients and present the averaged results. Due to space constraints, further details of the experiments and results showing the training losses are presented in \Cref{app:exp}.

\begin{figure}[!h] \small
\centering
\begin{subfigure}{.157\textwidth}
\includegraphics[width=1\textwidth]{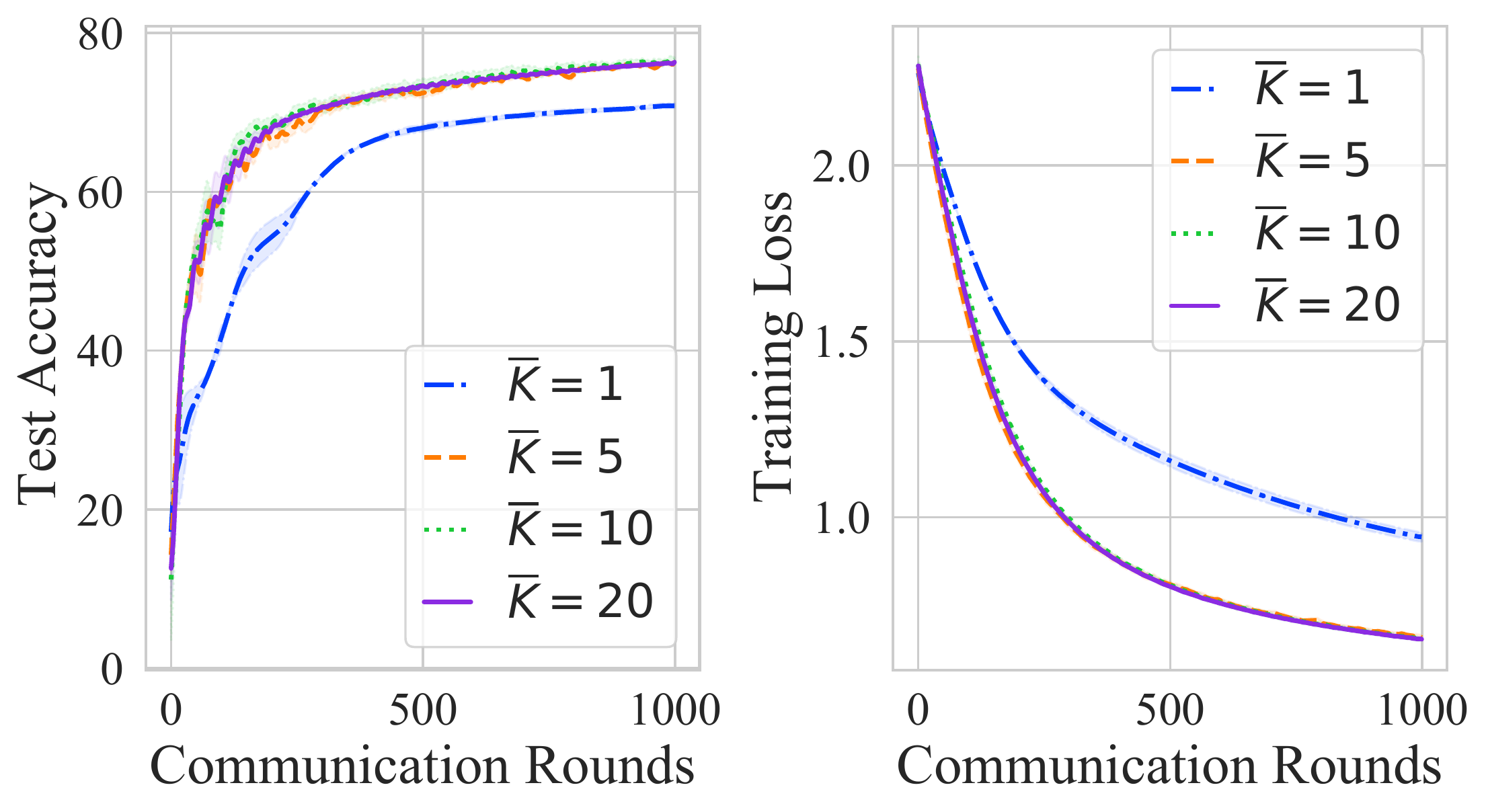} \caption{\small GD}
\end{subfigure}
\begin{subfigure}{.157\textwidth}
\includegraphics[width=1\textwidth]{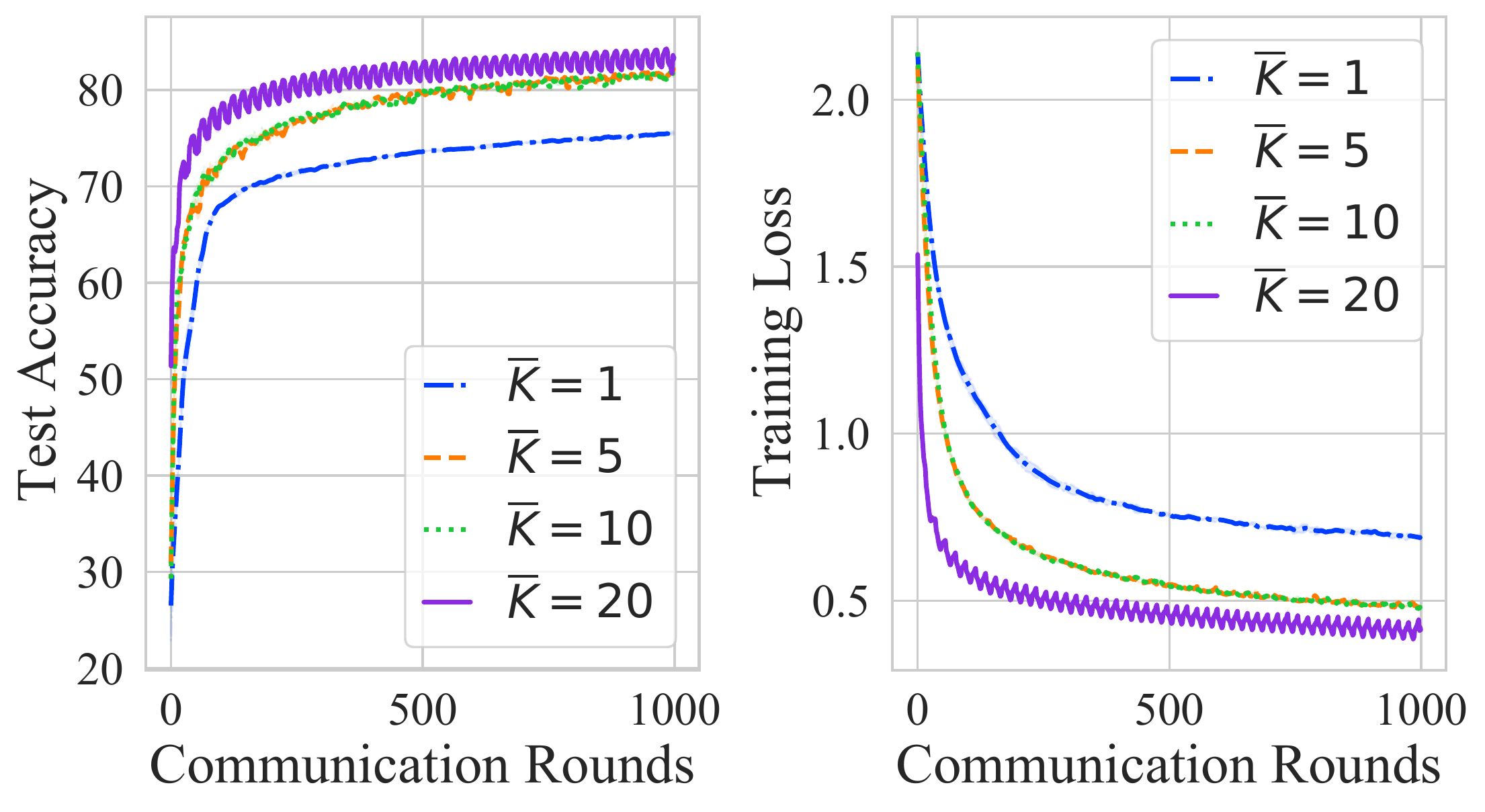} \caption{\small SGD }
\end{subfigure}
\begin{subfigure}{.157\textwidth}
\includegraphics[width=1\textwidth]{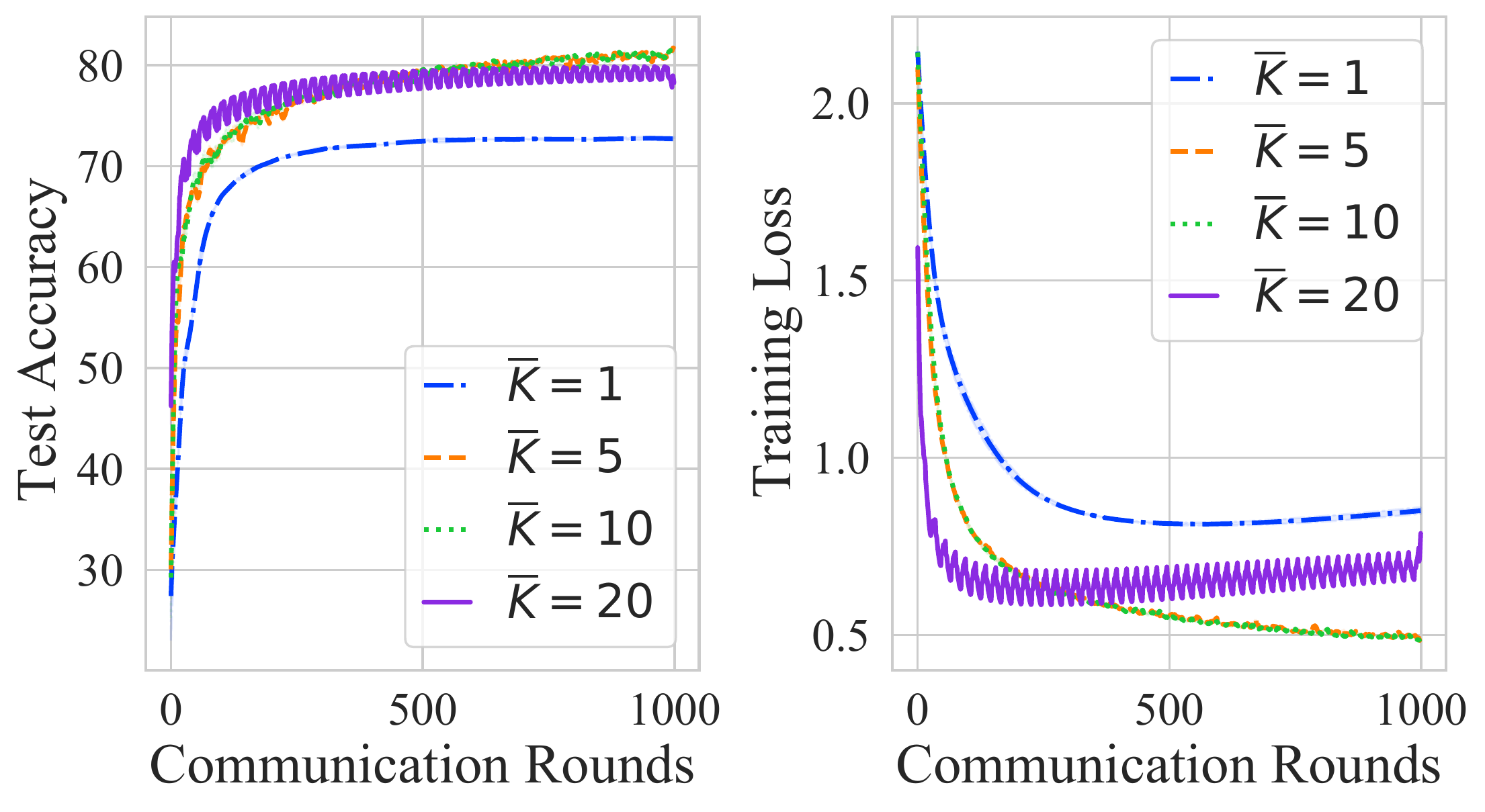} \caption{\small SSGD}
\end{subfigure} \vspace{-1em}
 \caption{\small Test accuracy for FMNIST for high data heterogeneity ($\alpha=0.5$). \cycp ($\kbar>1$) shows a higher test accuracy performance of $5$-$10$\% improvement compared to $\kbar=1$ (Standard FedAvg) for all different client local procedures. 
 }\label{fig:fmnist_1} 
\end{figure}

\begin{figure}[!h]\small
\centering
\begin{subfigure}{.157\textwidth}
\includegraphics[width=1\textwidth]{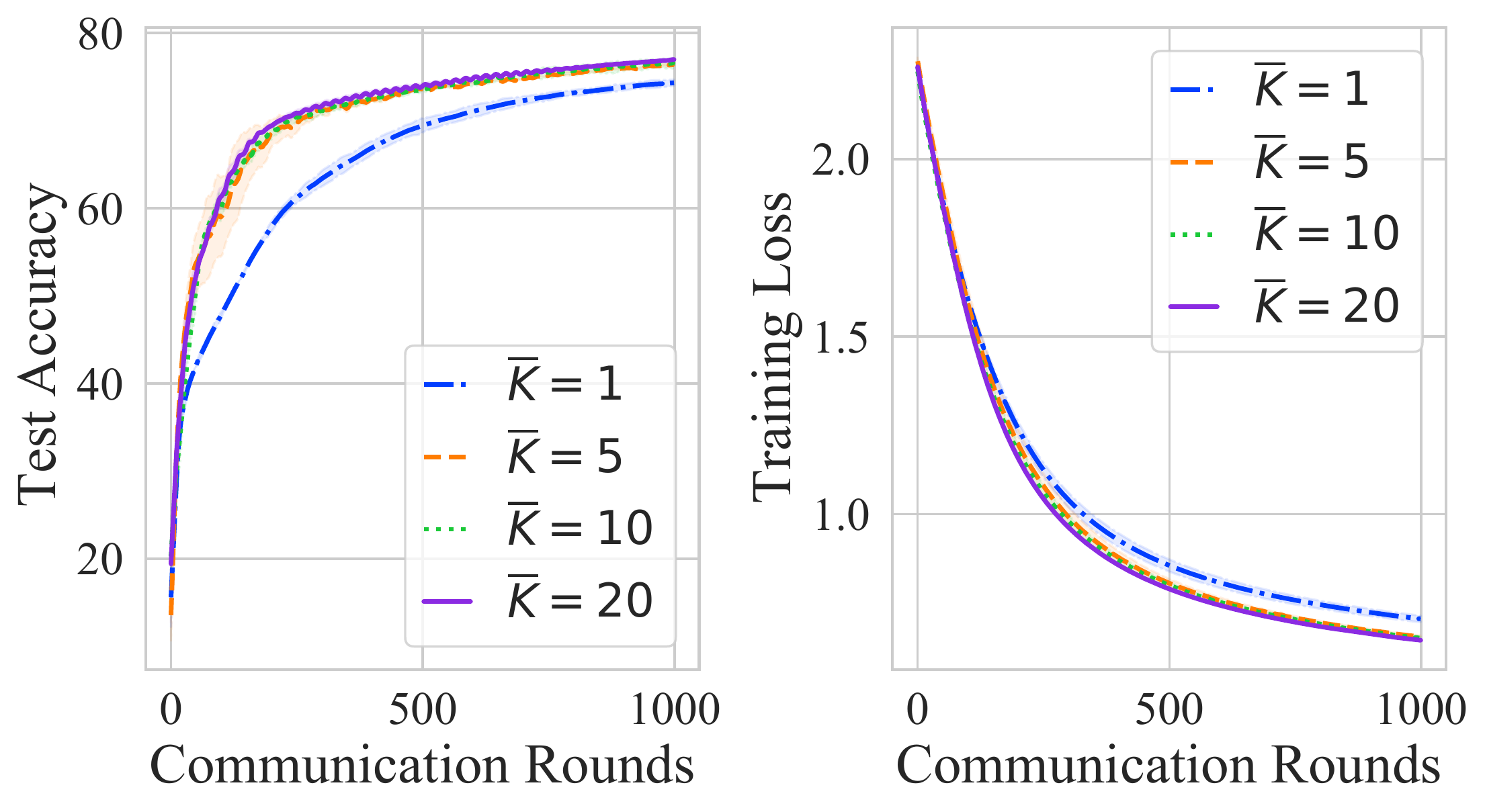} \caption{\small GD}
\end{subfigure}
\begin{subfigure}{.157\textwidth}
\includegraphics[width=1\textwidth]{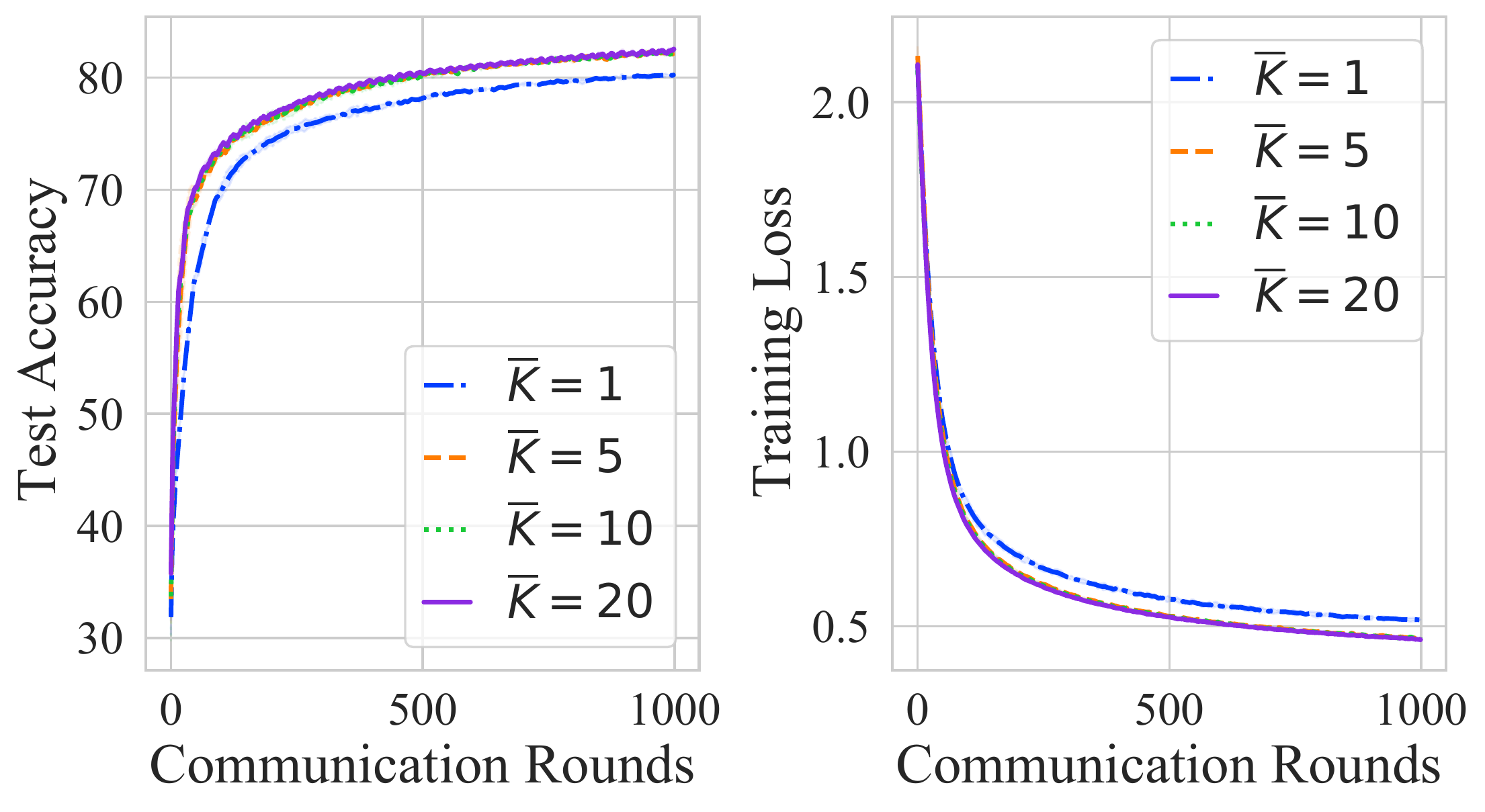} \caption{\small SGD}
\end{subfigure}
\begin{subfigure}{.157\textwidth}
\includegraphics[width=1\textwidth]{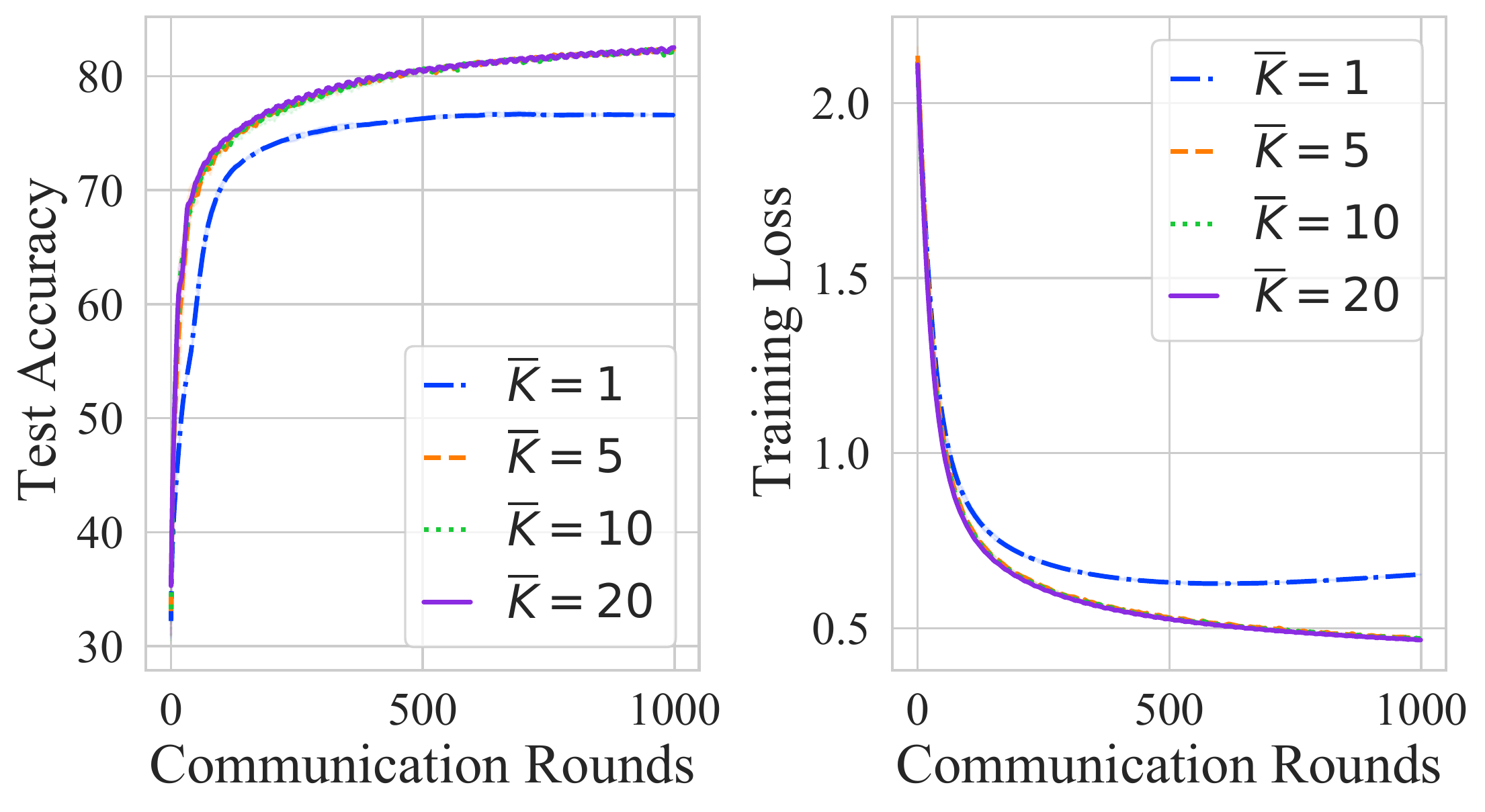} \caption{\small SSGD}
\end{subfigure} 
 \caption{\small Test accuracy for FMNIST for low data heterogeneity ($\alpha=2.0$). Being consistent with the high data heterogeneity case in \Cref{fig:fmnist_1}, \cycp ($\kbar>1$) shows a higher test accuracy performance compared to $\kbar=1$ (Standard FedAvg) for all different client local procedures with the improvement of $2$-$8$\%. However, the performance gap between \cycp and standard FedAvg is lower than when there is higher data heterogeneity. 
 }
 \label{fig:fmnist_2}  
\end{figure}

\paragraph{Effect of $\kbar$ and Data Heterogeneity.} We show in \Cref{fig:fmnist_1} the test accuracy for the FMNIST dataset for different $\kbar$ values and client local procedures for high data heterogeneity ($\alpha=0.5$). Recall that $\kbar=1$ is analogous to standard FedAvg and $1<\kbar\leq M/N$ (where $M/N=20$ for the FMNIST case) implies cyclic client participation. A higher $\kbar$ represents the server visiting more client groups within a single cycle. We show that for high data heterogeneity, for all different client local procedures, a higher $\kbar$ achieves better test accuracy by approximately $5$-$10\%$ improvement. Although our theoretical results suggest that \cycp sees improvement in convergence for only $\kbar=M/N$, we observe improvement even for $M/N>\kbar>1$. This can be due to our theoretical results being on PL-objectives while the landscape of DNN may not necessarily fall into this category~\cite{qu2020pllandscape}. For lower data heterogeneity results shown in \Cref{fig:fmnist_2}, the improvement for $\kbar>1$ compared to standard FedAvg is approximately $2$-$8$\%. The improvement is less than that for the high data heterogeneity case which aligns with the theoretical results in \Cref{sec:theo} which shows that increasing $\kbar$ decreases the dominant term that is dependent on the intra-group data heterogeneity. In \Cref{fig:emnist}, we show that the performance gap between \cycp and standard FedAvg is even higher due to the data heterogeneity being even higher than the FMNIST case.

\begin{figure}[!t] \small
\centering
\begin{subfigure}{.157\textwidth}
\includegraphics[width=1\textwidth]{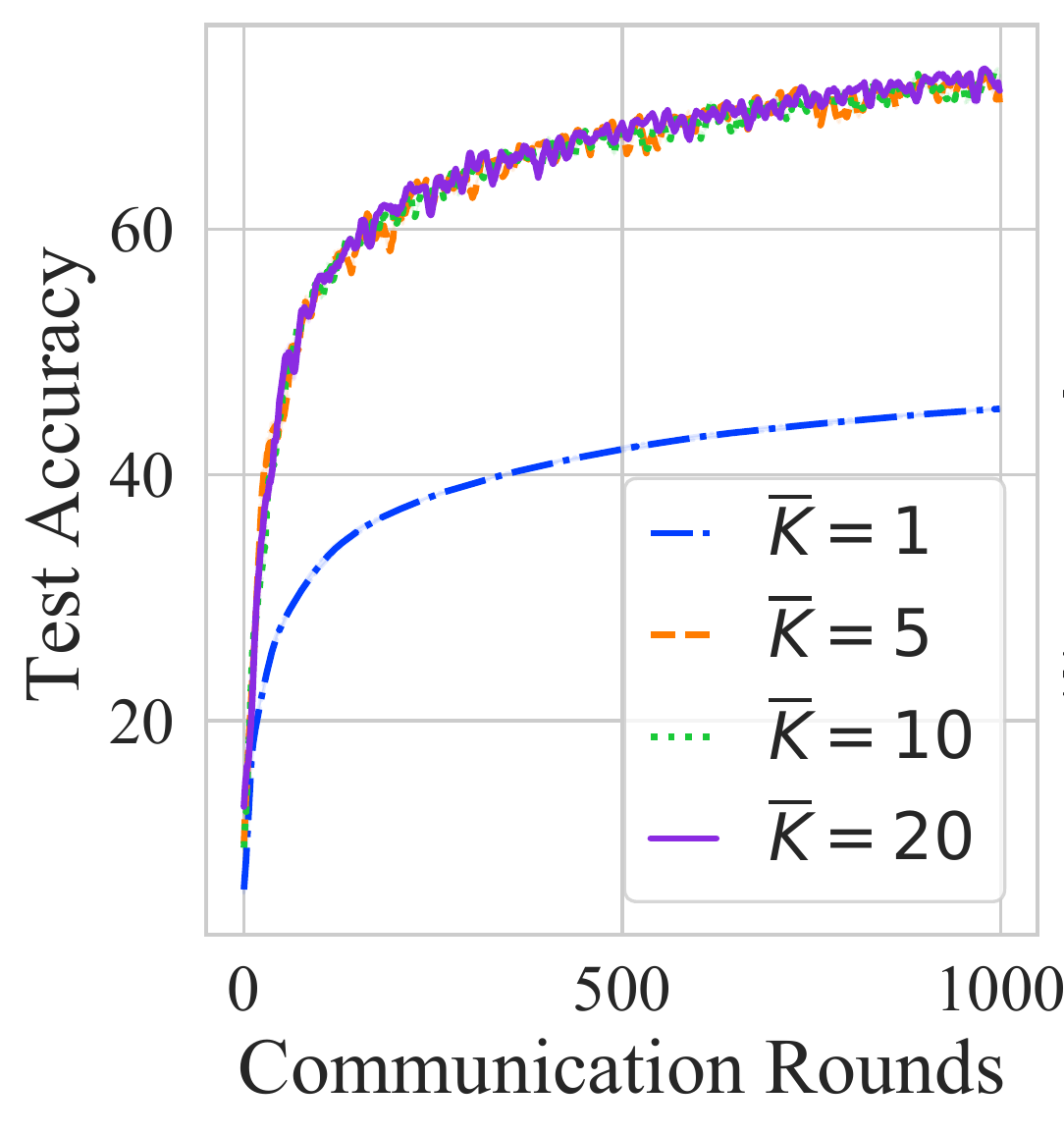} \caption{\small GD}
\end{subfigure}\hspace*{-0.2em}
\begin{subfigure}{.157\textwidth}
\includegraphics[width=1\textwidth]{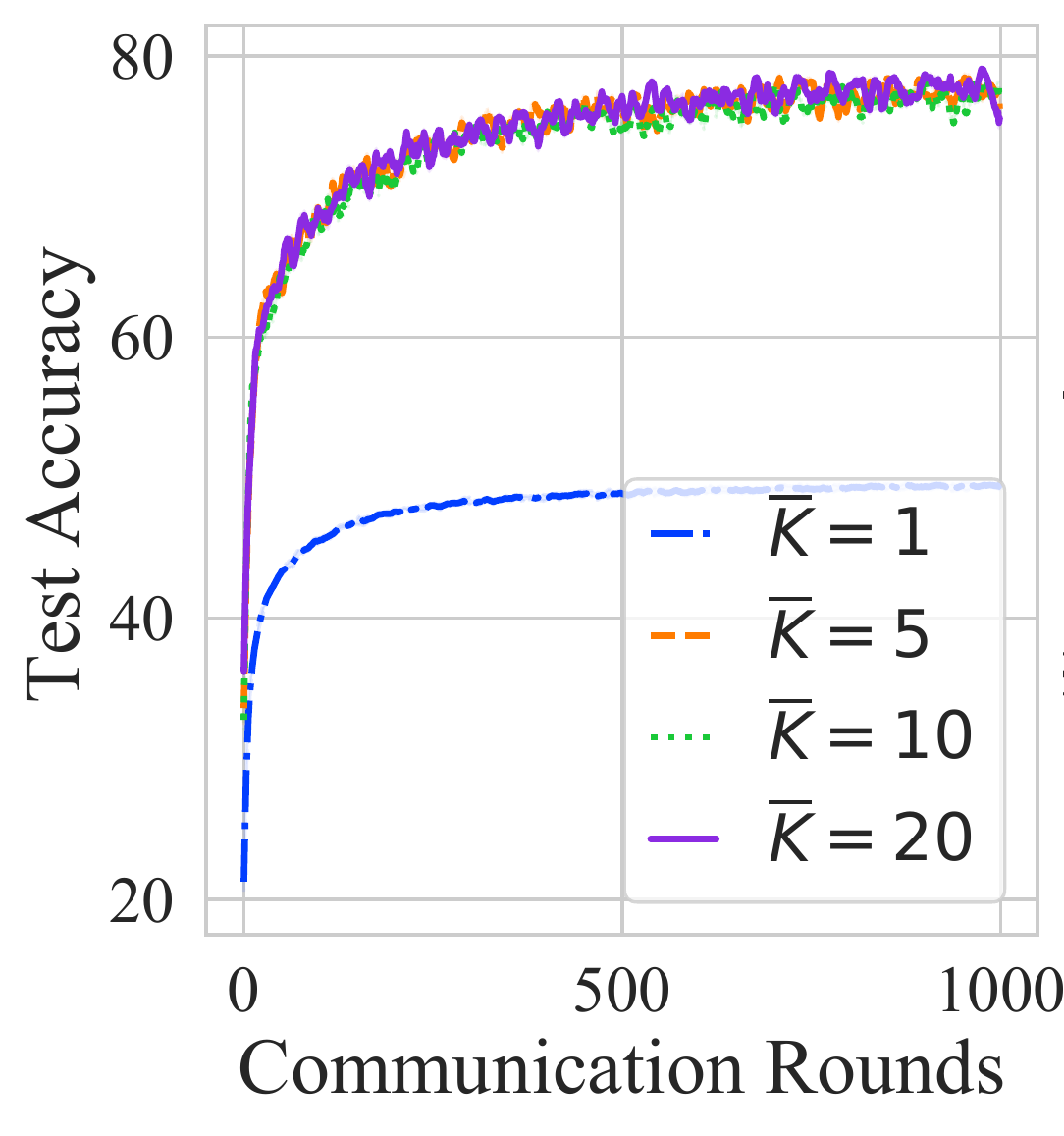} \caption{\small Local SGD}
\end{subfigure}\hspace*{-0.2em}
\begin{subfigure}{.157\textwidth}
\includegraphics[width=1\textwidth]{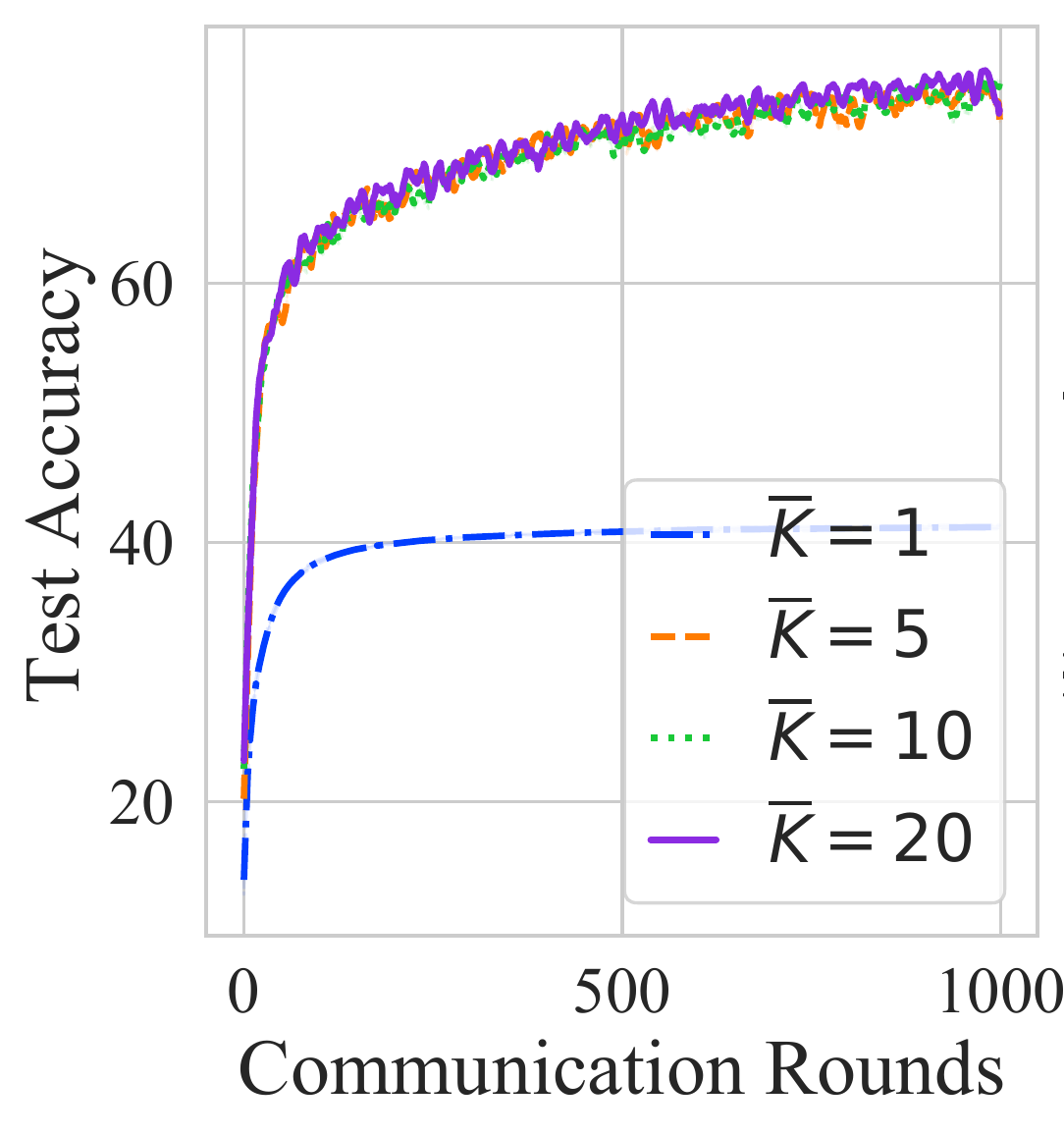} \caption{\small Shuffled SGD}
\end{subfigure}  \vspace{-0.5em}
 \caption{\small Test accuracy for EMNIST with high data heterogeneity ($\alpha=0.05$). \cycp shows a higher test accuracy performance to Standard FedAvg for all different client local procedures. The improvement gap is larger than the FMNIST case which is due to the EMNIST dataset partitioned with more data heterogeneity.}
 \label{fig:emnist}
\end{figure}

\textbf{Difference Across Client Local Procedures.} One distinct characteristic that can be observed in \Cref{fig:fmnist_1}(b)-(c) and \Cref{fig:emnist}(b)-(c) which are the results for the local SGD and SSGD with high data heterogeneity is that for the highest $\kbar$ there are oscillations in the test accuracy curve. This is due to the \cycp where as we increase $\kbar$, the inter-group data heterogeneity also becomes higher causing oscillation as the server sequentially visits the group for training. This behavior has also been observed in previous work where the server trains sequentially in a cyclic manner from different groups in \cite{zhu2021diurnal}. We show similar behavior of oscillation in the training loss curves shown in \Cref{fig:fmnist_3} and \Cref{fig:fmnist_4} in \Cref{app:exp}.

\section{Concluding Remarks}
Cyclic client participation is frequently observed in practical FL systems~\cite{kairouz2021privswor,zhu2021diurnal}, but its effect on the convergence of FedAvg is not yet well-understood. In this paper, we formulate a new framework to analyze the convergence of FedAvg with cyclic client participation for PL-objectives. Our analysis covers different client local procedures such as GD, SGD, and shuffled SGD. The analysis allows us to understood how FedAvg convergence is affected by different characteristics of the system such local update procedures, data heterogeneity within and across groups of clients that cyclically become available, and the length of the cycle $\kbar$. We discover conditions in which cyclic client participation converges faster than standard FedAvg. We also provide comparisons across different client local procedures and algorithms that are special cases of our framework with cost analyses to achieve an $\epsilon$ error. Interesting future work includes extending the analysis to non-PL-objectives and general optimizers such as including momentum, and to cases where the client groups are not disjoint and not fixed throughout training.

\clearpage
\newpage
\comment{
\textbf{Existing upper bounds on Non-\cycp with Local SGD}
One of the most common algorithms studied in FL is the non-\cycp with local SGD. Note that with $\kbar=1$ our algorithm becomes analogous to this algorithm, and the bound in \Cref{eq:8-0-0}, with $\lr=\log(MT^2)/\tau\mu N T$ \Cref{theo2:locSGD} becomes
\begin{align}
&\begin{aligned}
    &F(\wb^{(K,0)})-F^*\leq\frac{(F(\wb^{(0,0)})-F^*)}{MT^2}+\bigto{\frac{\kappa^2 \nu^2}{\mu M T^2}}\\
    &+\bigto{\frac{\kappa\gamma^2}{\mu N T}\left(\frac{M-N}{M-1}\right)}+\bigto{\frac{\kappa\sigma^2}{\mu\tau N T}}\\
    &+\bigto{\frac{\kappa^2(\tau-1)\nu^2}{\mu\tau N^2T^2}}
    \end{aligned} \label{eq:11-0-0}
\end{align} where $T\geq10\kappa\log{(M\kbar^2T^2)}$.

The first last iterate bound we compare with our bound is \cite[Theorem 1]{} which assumes full client participation, and bounded norm of the stochastic gradient with parameter $G$. Note that the third term in \Cref{eq:11-0-0} goes to 0 for full client participation and the first two terms in \Cref{eq:11-0-0} are due to the lower bound on $K$ and our small learning rate which is not directly comparable with  \cite[Theorem 1]{}. However, we do observe that the last two terms in \Cref{eq:11-0-0} that appear due to the noise coming from the local update steps match their last iterate bound that reads
\begin{align}
    \expt[F(\wb^{(K,0)})]-F^*\leq&\bigto{\frac{\kappa\sigma^2}{\mu M T}}+\bigto{\frac{\kappa^2\tau^2G^2}{\mu T^2}}
\end{align}
with an additional improvement by $1/\tau$ due to the learning rate being inversely proportional to $\tau$ and the difference in the last term's $\nu^2/\tau M^2$ and $G^2$ due to them assuming bounded norm of the stochastic gradient while we assume only the bounded variance of the stochastic gradient. 

The next bound we compare with is \cite[Theorem1]{} which assumes partial client participation, and the same assumptions regarding the noise from local updates with our work. Admittedly, this work bounds the squared norm of the global objective gradient with the minimum over the iterates, and a direct comparison   With learning rate set as $\lr=1/L\tau \sqrt{T}$ in their work, the last iterate bound reads
\begin{align}
\begin{aligned}
\min_{k\in\{0,...,T-1\}}\expt\nbr{\nabla F(\wb^{(k,0)})}^2\leq&\frac{L(F(\wb^{(0,0)})-F^*)}{\sqrt{T}}\\
+\bigto{\frac{\nu^2}{ T}}+\bigto{\frac{\nu^2}{N\sqrt{T}}\left(\frac{M-N}{M-1}\right)}\\
+\bigto{\frac{\sigma^2}{\tau N \sqrt{T}}}
\end{aligned}
\end{align}
}
\comment{
\ps{If there is a single machine with $M B$ samples, running for $K$ epochs, the rate of RR would be $O(1/(M B K^2))$. When can we achieve this rate?}\\
\yj{We achieve the same rate as Local RR when it has $M=1,~N=MB$ set in its Theorem 2, for our case it will be $M=1, N=1, B=NB,~\kbar=1$. You can see by plugging in these parameters that they achieve the same rate. Note that $\tau,~\nu$ in their theorem is $\nu,~\nub$ in ours respectively.}\\
\ps{Based on what you wrote, the expression above reduces roughly to
\begin{align}
    F(\wb^{(K,0)})-F^* \leq \frac{F(\wb^{(0,0)})-F^*}{NBK^2} + \bigto{\frac{\kappa^2 \nu^2}{\mu K^2}} + \bigto{\frac{\kappa^2 \overline{\nu}^2}{\mu N B K^2}}
\end{align}
The second term is worse than $\frac{1}{NBK^2}$. In local RR, this problem doesn't seem to arise. More importantly, in this case, our result should reduce to shuffle SGD.
}
\yj{I think it does reduce to the rate you mentioned because for Local RR Theorem 2, if we replace the parameters accordingly ($M=1,N=MB, B=N$), it achieves $\bigto{\frac{\kappa^2 \nu^2}{\mu K^2}}$ the same as what we have in our Theorem 4.4. I think you should clarify when you talk about  LocalRR whether you mean the $B=N$ scenario or not because that makes a different from our SWOR algorithm.}

\ps{Suppose, $N = 1, \kbar = M$, this setting is the same as going through all the $MB$ samples in a sequence. In this case, RR gives $\frac{1}{MBK^2}$ rate. What we get is
\begin{align}
\begin{aligned}
    F(\wb^{(K,0)})-F^* & \leq\frac{F(\wb^{(0,0)})-F^*}{MBK^2}+\bigto{\frac{\kappa^2 \nu^2}{\mu M^2 K^2}}+\bigto{\frac{\kappa^2\nu^2}{\mu M K^2}} +\bigto{\frac{\kappa^2 \overline{\nu}^2}{\mu M^2 B K^2}} \\
    &= \frac{F(\wb^{(0,0)})-F^*}{MBK^2} + \bigto{\frac{\kappa^2\nu^2}{\mu M K^2}} + \bigto{\frac{\kappa^2 \overline{\nu}^2}{\mu M^2 B K^2}} \label{eq:13}
\end{aligned}
\end{align}
The thing to remember is that depending on the shuffling scheme in single-client SGD, the convergence rate varies: given $n$ samples and $K$ epochs RR and shuffle-once (SO) give $\frac{1}{n K^2}$ rate, while incremental gradient (IG) gives $\frac{1}{K^2}$. Our current scheme sits somewhere in between IG and RR. We shuffle client samples every time, but the order in which clients are processed is fixed over time. That could explain our rate in \eqref{eq:13}, which lies somewhere between $\frac{1}{M K^2}$ and $\frac{1}{MBK^2}$. It would be interesting to see if we can get $\frac{1}{MBK^2}$ rate by shuffling client clusters once in the beginning.
}

\ps{
\begin{itemize}
    \item In some regimes, the $\frac{\kappa\sigma^2}{\mu\tau N K}$ term would be dominant. However, this term is worse than the $\frac{\sigma^2}{\mu\tau N K}$ bound in \cite[Corollary~2.3]{stich2018local}, \cite[Lemma~15]{koloskova2020unified}. Some reviewers will most likely point this out and we need to explain/improve this. 
    \begin{itemize}
        \item One potential explanation is that both these works (\cite[Corollary~2.3]{stich2018local}, \cite[Lemma~15]{koloskova2020unified}) analyze convergence in terms of a weighted average of iterates, rather than the last iterate. \cite[Section~3.2.1]{yun2022shuf} also discusses this point.
        \item Separate client and server learning rates might help here.
    \end{itemize}
    \yj{$\rightarrow$ this can be explained by the fact that the referred work assumes strongly-convex functions while we assume PL-functions. We can add the analysis version for strongly-convex functions.}
    \item The last two terms can be subsumed in the penultimate term. We also need the local-updates term of LocalSGD with our last term. 
    \begin{itemize}
        \item It seems that our local-updates error depends only on heterogeneity $\nu$, while in localSGD, it depends on both stochastic gradient noise and heterogeneity. Can we explain this difference?
        \item Our dependence on $\kappa$ is worse (quadratic, as opposed t linear in \cite[Lemma~15]{koloskova2020unified}), but we also have no. of participating clients $N$ in the denominator, which they don't. However, we should find some other papers which utilize two-sided learning rates and might have even better higher-order terms.
    \end{itemize}
    \yj{$\rightarrow$ The dependency on $\sigma$ doesn't appear using this analysis technique because we don't get the $\sigma$ term when deriving the terms in \Cref{eq:5-0-3}. The dependency on $\kappa$ is worse we are getting $K^2$ in the denominator in the other terms which the referred work doesn't as well as the $N$. Moreover, we are using the last iterate bound instead of that of the average of the iterates.}
    \item A more general question is regarding how to extend this shuffling based analysis to more general optimizers, like momentum, or the general optimizers considered in FedNova \cite{wang2020tackling}.
    \item We should compare with the results in \cite{mohtashami22characterizing}.
\end{itemize}
}
}

\newpage
\bibliography{dist_sgd}
\bibliographystyle{icml2023}
\newpage
\appendix

\onecolumn

\section{Additional Experimental Details and Results} \label{app:exp}

\paragraph{Additional Experimental Setup Details.} All experiments are conducted on clusters equipped with one NVIDIA TitanX GPU. The algorithms are implemented in PyTorch 1. 11. 0. The code used for all experiments is included in the supplementary material. For all experiments, we do a grid search over the required hyperparameters to find the best-performing ones and then fix the hyperparameters and only change $\kbar$. Specifically, we do a grid search over the learning rate: $\lr\in\{0.05, 0.01, 0.005, 0.001\}$, batch size: $b\in\{32, 64, 128\}$, and local iterations: $\tau\in\{5, 10, 30, 50\}$ to find the hyper-parameters with the highest test accuracy for each benchmark. For the deep multi-layer perceptron used for our experiments, we use a network with 2 hidden layers of units $[64,30]$ with dropout after the first hidden layer where the input is the normalized flattened image and the output consists of the label space. For all experiments, the data is partitioned to $80\%/10\%/10\%$ for training/validation/test data, where the training data then is again partitioned across the clients heterogeneously.

\paragraph{Training Losses for the Results in \Cref{fig:fmnist_1} and \Cref{fig:fmnist_2}.} We present the training loss curves for the test accuracy results shown in \Cref{fig:fmnist_1} and \Cref{fig:fmnist_2} in \Cref{fig:fmnist_3} and \Cref{fig:fmnist_4} respectively. We see that the implications are consistent to what we have observed for the test accuracy where a higher $\kbar>1$ leads to faster convergence. The convergence improvement gap is larger when we have high data heterogenetiy as shown in \Cref{fig:fmnist_3} compared to the improvement gap for lower data heterogeneity shown in \Cref{fig:fmnist_4}. Moreover, for the SGD and SSGD client local procedures for high dataheterogeneity (\Cref{fig:fmnist_3}(b)-(c)) for the highest $\kbar=20$, we see the oscilliations that were also observed in the test accuracy curves. This is due to the cyclic participation of the clients which the client groups have heterogeneous data. Such oscilliation is not observed for lower data heterogeneity and lower $\kbar$ values.

\begin{figure}[!h] \small
\centering
\begin{subfigure}{.2\textwidth}
\includegraphics[width=1\textwidth]{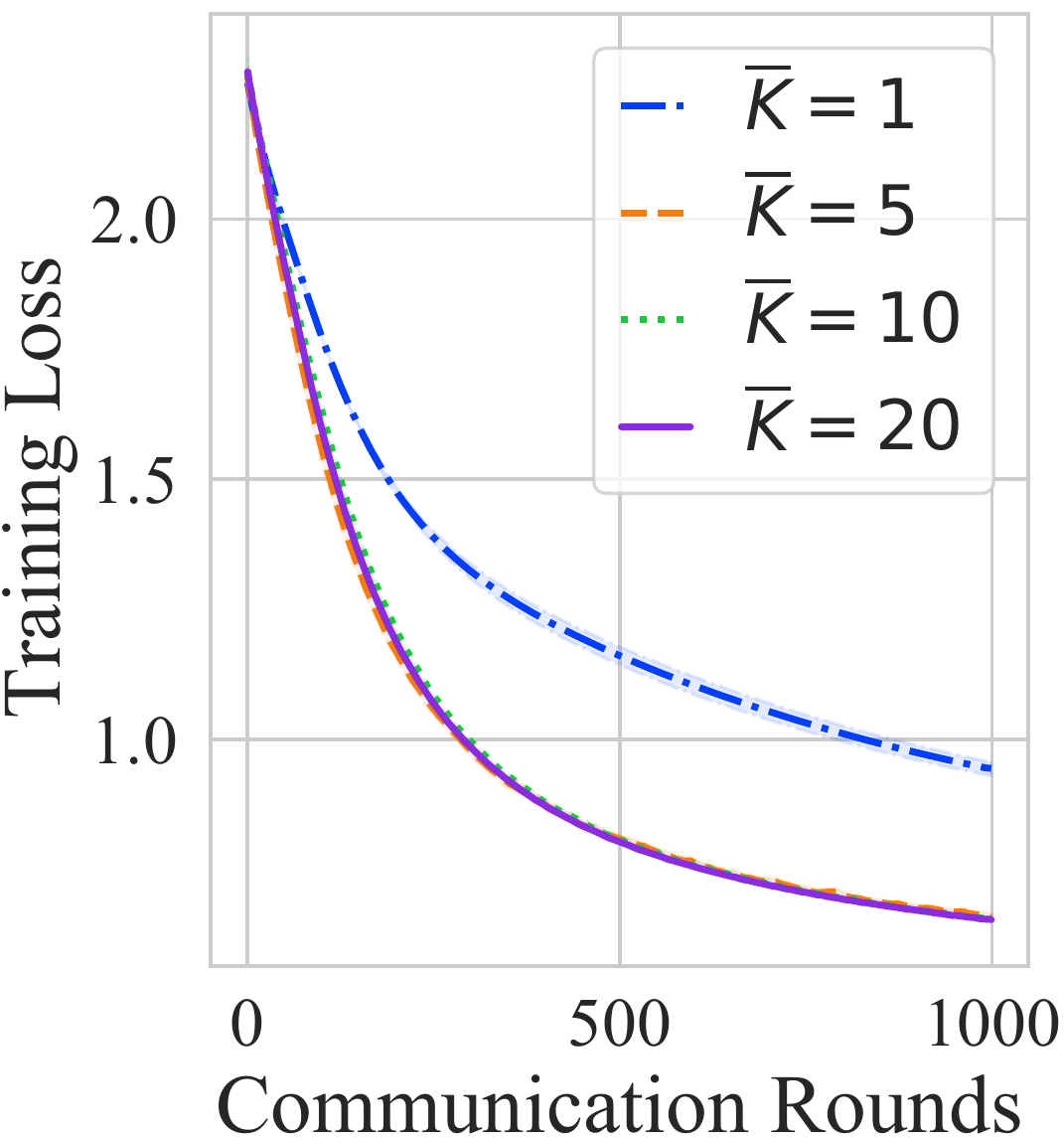} \caption{\small GD}
\end{subfigure}
\begin{subfigure}{.2\textwidth}
\includegraphics[width=1\textwidth]{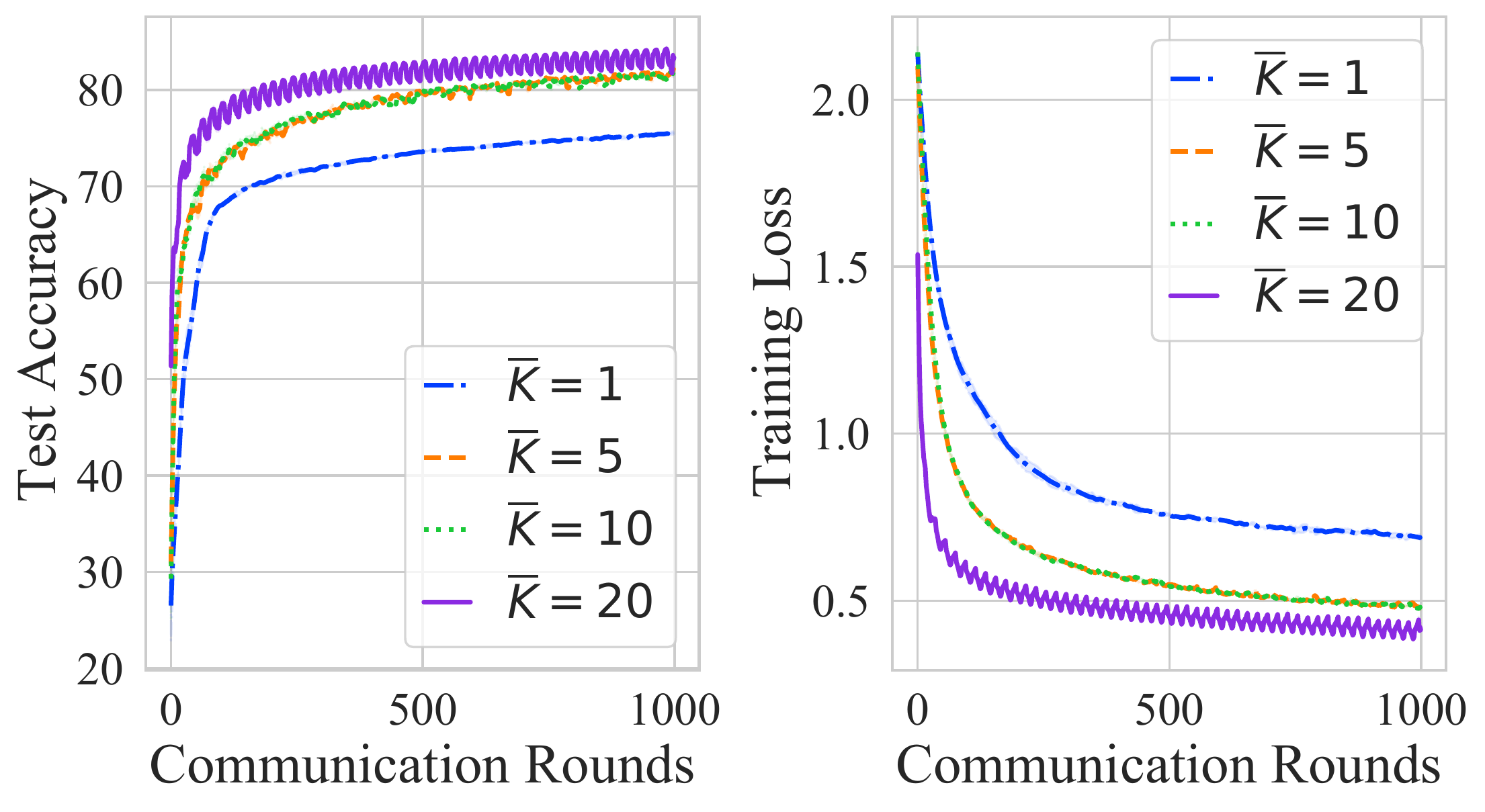} \caption{\small SGD }
\end{subfigure}
\begin{subfigure}{.2\textwidth}
\includegraphics[width=1\textwidth]{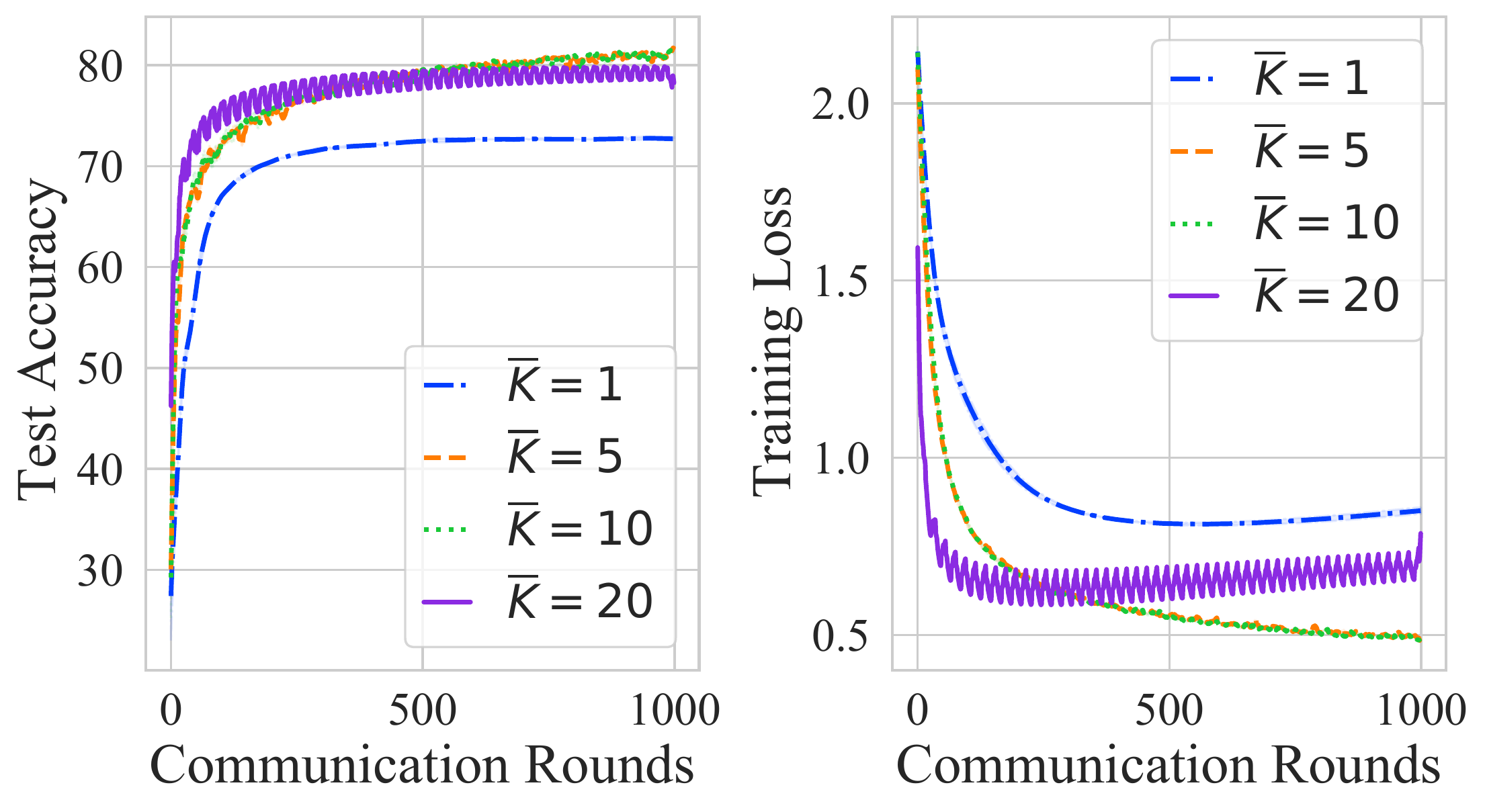} \caption{\small SSGD}
\end{subfigure}
 \caption{\small Training loss for FMNIST for high data heterogeneity ($\alpha=0.5$).
 }\label{fig:fmnist_3}
\end{figure}

\begin{figure}[!h]\small
\centering
\begin{subfigure}{.2\textwidth}
\includegraphics[width=1\textwidth]{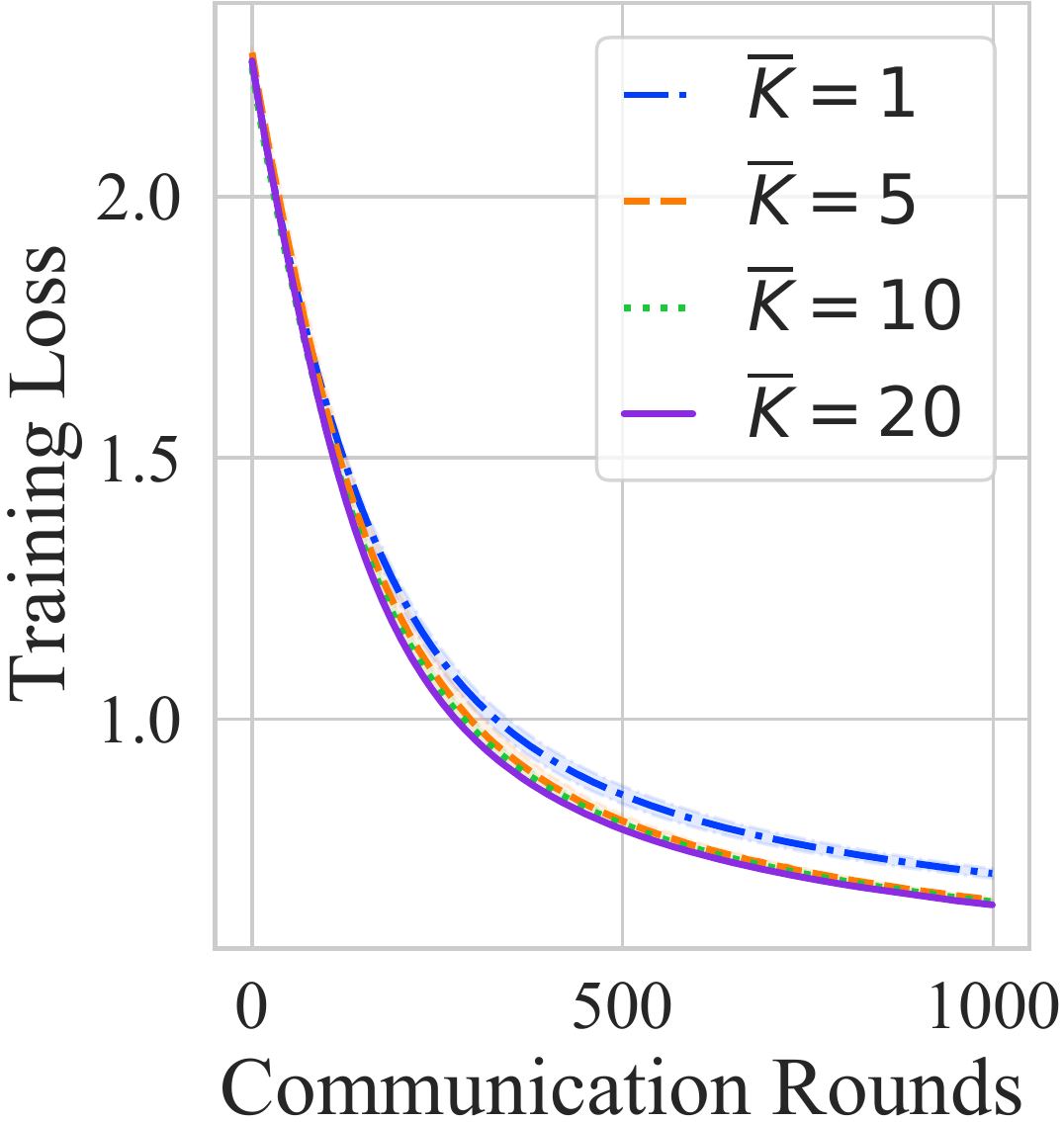} \caption{\small GD ($\alpha=2.0$)}
\end{subfigure}
\begin{subfigure}{.2\textwidth}
\includegraphics[width=1\textwidth]{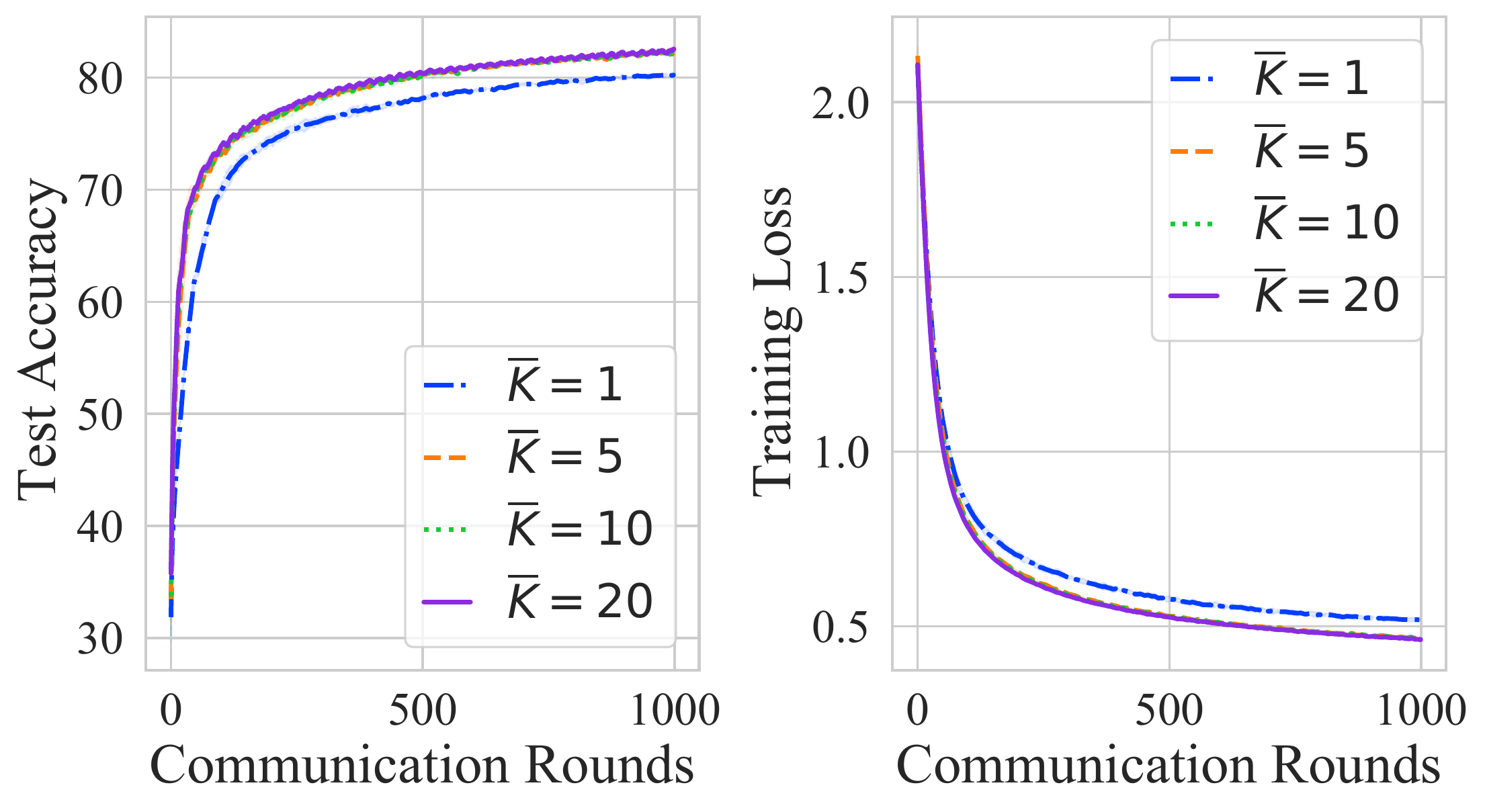} \caption{\small SGD ($\alpha=2.0$)}
\end{subfigure}
\begin{subfigure}{.2\textwidth}
\includegraphics[width=1\textwidth]{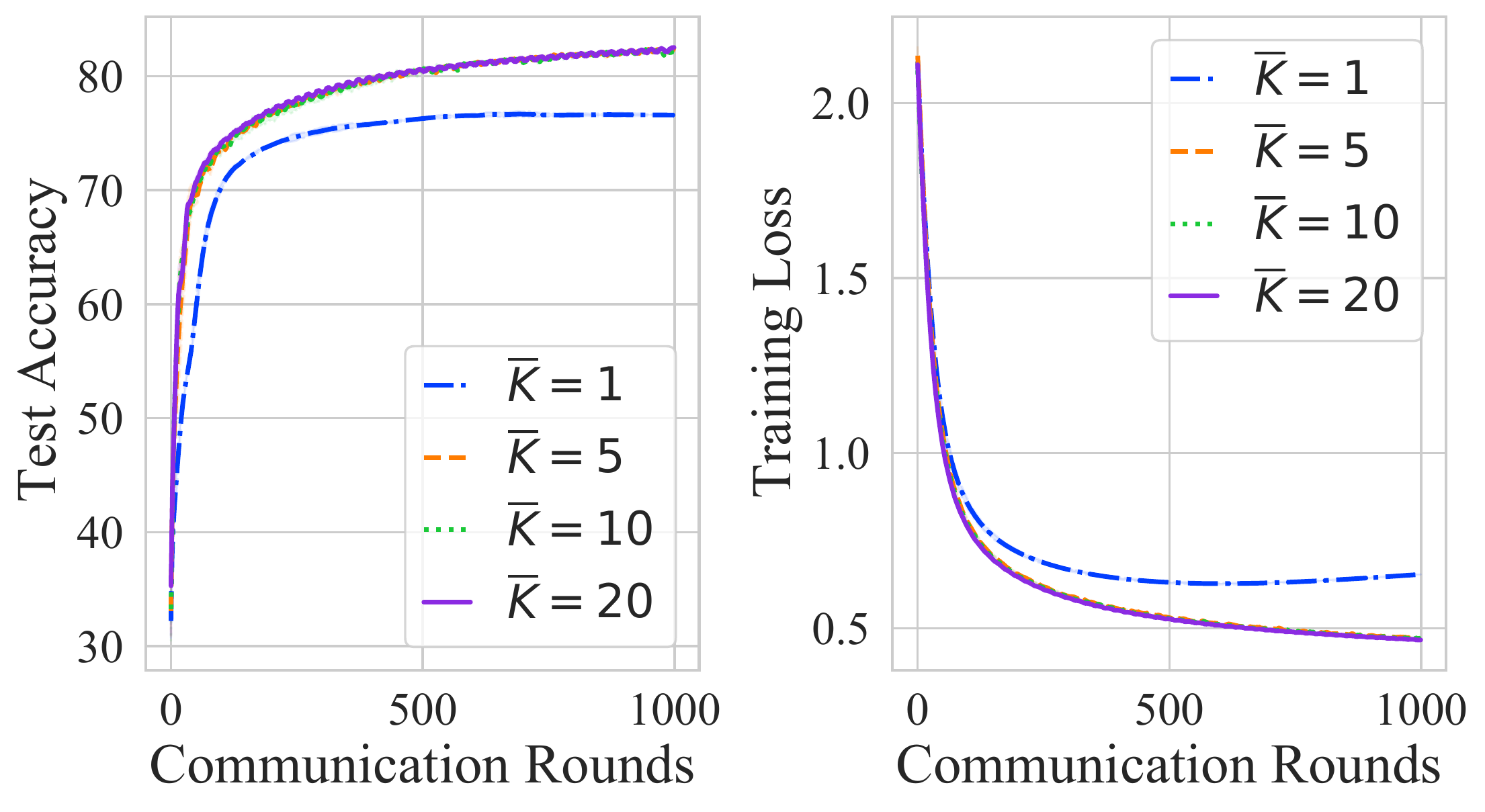} \caption{\small SSGD ($\alpha=2.0$)}
\end{subfigure} 
 \caption{\small Training loss for FMNIST for low data heterogeneity ($\alpha=2.0$). 
 }
 \label{fig:fmnist_4}  \vspace{-1em}
\end{figure}





\newpage
\section{Useful Inequalities}


\begin{lemma}[Young's inequality]
\label{lem:Young}
Given two same-dimensional vectors $\mbf u, \mbf v \in \mbb R^d$, the Euclidean inner product can be bounded as follows:
$$\lan \mbf u, \mbf v \ran \leq \frac{\norm{\mbf u}}{2 \gamma} + \frac{\gamma \norm{\mbf v}}{2}$$
for every constant $\gamma > 0$.
\end{lemma}

\comment{
\begin{lemma}[Strong Concavity]
A function $g: \mc X \times \mc Y$ is strongly concave in $\by$, if there exists a constant $\mu > 0$, such that for all $\bx \in \mc X$, and for all $\by, \by' \in \mc Y$, the following inequality holds.
$$g(\bx, \by) \leq g(\bx, \by') + \langle \nabla_y g(\bx, \by'), \by' - \by \rangle - \frac{\mu}{2} \norm{\by - \by'}.$$
\end{lemma}}

\begin{lemma}[Jensen's inequality]
\label{lem:jensens}
Given a convex function $f$ and a random variable $X$, the following holds.
$$f \lp \mbe [X] \rp \leq \mbe \left[ f(X) \right].$$
\end{lemma}

\begin{lemma}[Sum of squares]
\label{lem:sum_of_squares}
For a positive integer $K$, and a set of vectors $\xb_1, \hdots, \xb_K$, the following holds:
\begin{align*}
    \norm{\sum_{k=1}^K \xb_k} \leq K \sum_{k=1}^K \norm{\xb_k}.
\end{align*}
\end{lemma}

\comment{
\begin{lemma}[Quadratic growth condition \cite{karimi2020linpl}]
\label{lem:quad_growth}
If function $g$ satisfies Assumptions \ref{as1}, \ref{as2}, then for all $x$, the following conditions holds
\begin{align*}
    g(x) - \min_{z} g(z) & \geq \frac{\mu}{2} \norm{x_p - x}, \\
    \norm{\G g(x)} & \geq 2 \mu \lp g(x) - \min_z g(z) \rp.
\end{align*}
\end{lemma}}

\comment{
\begin{lemma}
\label{lem:smooth_convex}
For $L$-smooth, convex function $g$, the following inequality holds
\begin{align}
    \mbe \lnr \G g (\by) - \G g (\bx) \rnr^2 \leq 2 L \left[ g (\by) - g (\bx) - \G g(\bx)^\top (\by - \bx) \right].
\end{align}
\end{lemma}}

\begin{lemma}[Variance for without replacement sampling] 
\label{lemma:WOR_sampling}
For a positive integer $K$ and a set of vectors $\xb_1,...,\xb_K$ with the mean of the vectors being $\overline{\xb}=\frac{1}{K}\sum_{k=1}^K\xb_k$ and we have a mini-batch $\mathcal{N}$ with size $N$  sampled uniformly at random without replacement from $[K]$, then we have
\begin{align}
    \expt\left[\nbr{\frac{1}{N}\sum_{k\in\mathcal{N}}\xb_k-\overline{\xb}}\right]^2=\frac{(K-N)}{NK(K-1)}\sum_{k=1}^K\nbr{\xb_k-\overline{\xb}}^2
\end{align}
\end{lemma}

\newpage

\section{Proofs for the CyCP+Local GD Case}
\label{app:CCP_localGD}
The proof of \cref{theo1:GD} is presented in this section. For simplicity, the proof is presented as follows: first, in \cref{sec:noisy_GD}, the model update steps are shown to be noisy gradient descent steps. Next, in \cref{sec:int_results_GD}, we present some intermediate results, which shall be used in the analysis, followed by the proof of \cref{theo1:GD} in \cref{sec:proof_theo1_GD}. Finally, in \cref{sec:proofs_int_results_GD}, we present the proofs of the intermediate results.

We define the $\sigma$-algebra generated by the randomness in the algorithm till cycle epoch $k$ as follows: $\mc F_k \triangleq \sigma \lcb \{ \wb^{(1,i)} \}_{i=1}^\kbar, \{ \wb^{(2,i)} \}_{i=1}^\kbar, \dots, \{ \wb^{(k-1,i)} \}_{i=1}^\kbar \rcb$. We use $\expt_k [\cdot]$ as the shorthand for the expectation $\expt [\cdot | \mc F_k]$.

\subsection{Global Model Updates as Noisy Gradient Descent Step}
\label{sec:noisy_GD}

With \cycp, with each client doing a single GD update, (case (i) in \Cref{algo1}), the update rule for the global model ($k \in [K], i \in [\kbar]$) is given as
\begin{align}
    \wb^{(k,i)}-\wb^{(k,i-1)}=-\eta \sum_{m\in\st^{(k,i)}}\nabla F_m(\wb^{(k,i-1)}) \label{eq:gd}
\end{align}
Recall that $\wb^{(k+1,0)} = \wb^{(k,\kbar)}$. Therefore, we can unroll \eqref{eq:gd} to get the following result.

\begin{lemma}
\label{lemma:noisy_GD}
\begin{align}
    \expt[\wb^{(k+1,0)} | \mc F_k]-\wb^{(k,0)} &= -\lr\kbar N\nabla F(\wb^{(k,0)}) + \lr^2 \expt[\rbl^{(k,0)} | \mc F_k], \label{eq:3-0-1b}
\end{align}
where  $\rbl^{(k,0)} \triangleq \sum_{i=1}^{\kbar-1}\left(\prod_{j=i+2}^{\kbar} (\ibd-\lr\ssbl^{(k,j)})\right)\ssbl^{(k,i+1)}\sum_{j'=1}^i \qbl^{(k,j')}$, with 
\begin{align*}
    \left.\begin{matrix}
        \ssbl^{(k,i)} \defeq \sum_{m\in\st^{(k,i)}} \int_{0}^{1}\nabla^2 F_m (\wb^{(k,0)} + t (\wb^{(k,i-1)} - \wb^{(k,0)})) \,dt \\
        \qbl^{(k,i)} \defeq \sum_{m\in\st^{(k,i)}} \nabla F_m(\wb^{(k,0)})
    \end{matrix}\right\}, \text{ for all } i \in [\kbar].
\end{align*}
\end{lemma}

\begin{proof}
Note that \Cref{eq:gd} is a rescaled version of the update rule in the main paper in the sense that the update is scaled up by $N$ but instead we downscale the step-size by setting the learning rate as $\lr=\log(MK^2)/\mu N \kbar K$. We can reformulate the gradient of the local objective as:
\begin{align}
    \nabla F_\ic(\wb^{(k,i-1)})&=\nabla F_\ic(\wb^{(k,0)})+\nabla F_\ic(\wb^{(k,i-1)})-\nabla F_\ic(\wb^{(k,0)}) \nn \\
    \Rightarrow \nabla F_\ic(\wb^{(k,i-1)})-\nabla F_\ic(\wb^{(k,0)}) &= \underbrace{\int_{0}^{1}\nabla^2 F_m (\wb^{(k,0)} + t (\wb^{(k,i-1)} - \wb^{(k,0)})) \,dt}_{\defeq \hb_{m}^{(k,i)}} \ (\wb^{(k,i-1)} - \wb^{(k,0)}) \label{eq:0-1}
\end{align}
Note that $\hb_{m}^{(k,i)}$ exists due to our assumption that the local objectives $F_m(\cdot),~m\in[M]$ are differentiable and $L$-smooth (see Appendix D.2 in \cite{yun2022shuf}), and we can thus show that $\|\hb_{m}^{(k,i)}\|\leq L,~\forall m,k,i~$. 
Using the expression in \Cref{eq:0-1} we get
\begin{align}
    \wb^{(k,i)}-\wb^{(k,i-1)}=-\eta \sum_{m\in\st^{(k,i)}}\left(\nabla F_\ic(\wb^{(k,0)})+ \hb_{m}^{(k,i)}(\wb^{(k,i-1)}-\wb^{(k,0)})\right).
\end{align}
Defining $\qbl^{(k,i)}\defeq\sum_{m\in\st^{(k,i)}}\nabla F_m(\wb^{(k,0)})$ and $\ssbl^{(k,i)}\defeq\sum_{m\in\st^{(k,i)}}\hb_m^{(k,i)}$ we have
\begin{align}
    \wb^{(k,i)}-\wb^{(k,i-1)} &= -\eta\qbl^{(k,i)}-\eta\ssbl^{(k,i)}(\wb^{(k,i-1)}-\wb^{(k,0)}) \nn \\
    \Rightarrow \wb^{(k,i)}-\wb^{(k,0)} &= \lp \mathbf{I} - \eta\ssbl^{(k,i)} \rp(\wb^{(k,i-1)}-\wb^{(k,0)}) -\eta\qbl^{(k,i)}. \label{eq:3-0}
\end{align}
Unrolling \Cref{eq:3-0}, we get
\begin{align}
    \wb^{(k+1,0)}-\wb^{(k,0)}=-\lr\sum_{i=1}^{\kbar}\qbl^{(k,i)}+\lr^2\underbrace{\sum_{i=1}^{\kbar-1}\left(\prod_{j=i+2}^{\kbar} (\ibd-\lr\ssbl^{(k,j)})\right)\ssbl^{(k,i+1)}\sum_{j'=1}^i \qbl^{(k,j')}}_{\defeq \rbl^{(k,0)}}  \label{eq:3-0-1}
\end{align}
Conditioning on $\mc F_k$, the only randomness in \eqref{eq:3-0-1} is owing to the random client sets $\{ \st^{(k,i)} \}_{i}$.  Therefore,
\begin{align}
    \expt[\wb^{(k+1,0)} | \mc F_k]-\wb^{(k,0)} &= -\lr\sum_{i=1}^\kbar\expt_k\left[\sum_{m\in\st^{(k,i)}}\nabla F_m(\wb^{(k,0)})\right]+\lr^2\expt_k\left[\rbl^{(k,0)}\right], \nn
\end{align}
Computing the expectation finishes the proof.
\end{proof}

Interpreting $\lr\kbar N$ as the effective learning rate, the update in \eqref{eq:3-0-1b} is a noisy gradient descent step. The bulk of the proof is concerned with bounding the noise term $\rbl^{(k,0)}$.

\subsection{Intermediate Results}
\label{sec:int_results_GD}

\begin{lemma}[Bound on the Sum of Gradients over Client Groups]
If the client functions $F_m$ satisfy \cref{as3}, then for arbitrary $i\in[\kbar]$ and any $\wb$
\begin{align}
\nbr{\sum_{j=1}^i\frac{1}{M/\kbar}\sum_{m\in\sigma(j)}\nabla F_m(\wb)}\leq i\alpha +i\nbr{\nabla F(\wb)}. \nn
\end{align} \label{lem0-0}
\end{lemma}

\begin{lemma}[Bounds on the Error Terms arrising due to \cycp]
\label{lemma:error_terms_GD}
Under \cref{as1}, \ref{as3}, the error in \eqref{eq:3-0-1b} can be bounded as
\begin{align}
    \nbr{\expt_k[\rbl^{(k,0)}]}
    & \leq \frac{3N^2L\kbar(\kbar-1)}{5} \lp \alpha + \nbr{\nabla F(\wb^{(k,0)})} \rp, \label{eq:3-1-0} \\
    \expt_k\left[\nbr{\rbl^{(k,0)}}^2\right] & \leq \frac{(\kbar-1)^2 \kbar^2 N^3 L^2}{2} \left[ \left(\frac{M/\kbar-N}{M/\kbar-1}\right)\gamma^2 + 2N \left(\nbr{\nabla F(\wb^{(k,0)})}^2+{\alpha^2}\right) \right]. \label{eq:3-2-3}
\end{align}
\end{lemma}
Looking at the bound in \eqref{eq:3-2-3}, we observe that in addition to the terms from the bound in \eqref{eq:3-1-0}, we also get the additional dependence on intra-group heterogeneity $\gamma$, due to subsampling of clients within each group.

\subsection{Proof of \cref{theo1:GD}}
\label{sec:proof_theo1_GD}

For ease of reference, we restate \cref{theo1:GD} here.

\begin{thm*}[Convergence of the Global Model with CyCP+Local GD (Case (i) in \Cref{algo1})]
Suppose the local client functions $\{ F_m \}$ satisfy \cref{as1} and \cref{as3}, while the global loss function satisfies \cref{as2}. If the learning rate satisfies $\lr \leq \frac{1}{7 L N \kbar}$, then the iterates generated by \Cref{algo1} with one local GD step at the clients satisfies
\begin{align}
    & \expt[F(\wb^{(K,0)})]-F^* \leq \left(1- \lr \kbar N \mu \right)^K (F(\wb^{(0,0)})-F^*) + \mco \lp \frac{\lr^2 L^2 N^2 (\kbar-1)^2 \alpha^2}{\mu} \rp + \mco \lp \kappa \lr \kbar \left(\frac{M/\kbar-N}{M/\kbar-1}\right) \gamma^2 \rp, \nn
\end{align}
where $\kappa=L/\mu$ is the condition number.
Choosing step-size $\lr=\log(MK^2)/\mu N\kbar K$ with $K\geq 7\kappa\log{(MK^2)}$, the convergence error is bounded as:
\begin{align}
    \expt[F(\wb^{(K,0)})]-F^*\leq\frac{F(\wb^{(0,0)})-F^*}{MK^2}+\tilo\left(\frac{\kappa^2(\kbar-1)^2\alpha^2}{\mu \kbar^2K^2}\right)+\tilo\left(\frac{\kappa\gamma^2}{\mu NK}\left(\frac{M/\kbar-N}{M/\kbar-1}\right)\right) \nn
\end{align}
where $\tilo(\cdot)$ subsumes logarithmic terms and numerical constants.
\end{thm*}

\begin{coro}
Denoting the total number of communication rounds in \cref{algo1} as $T=\kbar K$, in terms of $T$ the bound above becomes
\begin{align}
    \expt[F(\wb^{(K,0)})]-F^*\leq\frac{\kbar^2(F(\wb^{(0,0)})-F^*)}{MT^2}+\tilo\left(\frac{\kappa^2(\kbar-1)^2\alpha^2}{\mu T^2}\right)+\tilo\left(\frac{\kbar\kappa\gamma^2}{\mu NT}\left(\frac{M/\kbar-N}{M/\kbar-1}\right)\right). \nn
\end{align}
where since $K\geq 7\kappa\log{(MK^2)}$, we have $T\geq7\kappa\kbar\log{(M\kbar^2T^2)}$ and one cannot increase $\kbar$ without increasing $T$ accordingly due to its lower bound depending on $\kbar$.
\end{coro}

\begin{proof}
Using the $L$-smoothness property (\cref{as1}) of the global objective $F(\wb)$ we have
\begin{align}
    \expt_k[F(\wb^{(k+1,0)})]-F(\wb^{(k,0)})\leq\inner{\nabla F(\wb^{(k,0)})}{\expt_k[\wb^{(k+1,0)}]-\wb^{(k,0)}}+\frac{L}{2}\expt_k\left[\nbr{\wb^{(k+1,0)}-\wb^{(k,0)}}^2\right] \label{eq:3-0-2}
\end{align}
where the expectation here is over the selected client sets $\st^{(k,i)}$, for all $i\in[\kbar]$.
First, we bound the inner product term in \eqref{eq:3-0-2}. Using \cref{lemma:noisy_GD}, \cref{lemma:error_terms_GD}, we have
\begin{align}
    & \inner{\nabla F(\wb^{(k,0)})}{\expt_k[\wb^{(k+1,0)}]-\wb^{(k,0)}} \nn \\
    & = \inner{\nabla F(\wb^{(k,0)})}{-\lr\kbar N\nabla F(\wb^{(k,0)})+\lr^2\expt_k[\rbl^{(k,0)}]} \tag{Using \cref{lemma:noisy_GD}} \\
    & \leq -\lr\kbar N \nbr{\nabla F(\wb^{(k,0)})}^2 + \lr^2\nbr{\nabla F(\wb^{(k,0)})} \frac{3 N^2 L \kbar(\kbar-1)}{5} \left( \nbr{\nabla F(\wb^{(k,0)})} + \alpha \right) \tag{Using \cref{lemma:error_terms_GD}} \\
    &= -\lr\kbar N\left(1-\frac{3\lr (\kbar-1) NL}{5}\right)\nbr{\nabla F(\wb^{(k,0)})}^2+\frac{3\lr^2N^2L\kbar(\kbar-1)\alpha}{5}\nbr{\nabla F(\wb^{(k,0)})} \label{eq:3-1-1}
\end{align}
We can bound the last term in \Cref{eq:3-1-1} as
\begin{align}
    \frac{3\lr^2N^2L\kbar(\kbar-1)\alpha}{5}\nbr{\nabla F(\wb^{(k,0)})} 
    = \left(\frac{\lr^{1/2}N^{1/2}\kbar^{1/2}}{5}\nbr{\nabla F(\wb^{(k,0)})}\right)\left(3\lr^{3/2}N^{3/2}L\kbar^{1/2}(\kbar-1)\alpha\right) \nn \\
     \leq \frac{\lr N \kbar}{50}\nbr{\nabla F(\wb^{(k,0)})}^2+\frac{9\lr^3 \alpha^2 L^2 N^3 \kbar(\kbar-1)^2}{2}~~~~(\because\Cref{lem:Young}) \label{eq:3-1-2}
\end{align}
Plugging \Cref{eq:3-1-2} back into \Cref{eq:3-1-1} we have
\begin{flalign}
\begin{aligned}
     \inner{\nabla F(\wb^{(k,0)})}{\expt_k[\wb^{(k+1,0)}]-\wb^{(k,0)}} \leq -\lr\kbar N\left(\frac{49}{50}-\frac{3\lr (\kbar-1) NL}{5}\right)\nbr{\nabla F(\wb^{(k,0)})}^2 + \frac{9\lr^3 \alpha^2 L^2 N^3 \kbar(\kbar-1)^2}{2}.
    \end{aligned} \label{eq:3-1-3}
\end{flalign}
Now we aim to bound the last term in the RHS of \eqref{eq:3-0-2}. We have
\begin{align}
    \frac{L}{2}\expt_k\left[\nbr{\wb^{(k+1,0)}-\wb^{(k,0)}}^2\right] & \overset{\eqref{eq:3-0-1}}{=} \frac{L}{2}\expt_k\left[\nbr{-\lr\sum_{i=1}^{\kbar}\qbl^{(k,i)}+\lr^2\rbl^{(k,0)}}^2\right] \nn \\
    & \leq L\lr^2\expt_k\left[\nbr{\sum_{i=1}^{\kbar}\qbl^{(k,i)}}^2\right]+L\lr^4\expt_k\left[\nbr{\rbl^{(k,0)}}^2\right] \label{eq:3-1-5}
\end{align}
The last term is already bounded in \cref{lemma:error_terms_GD}, \eqref{eq:3-2-3}. Now, we bound the first term, 
\begin{flalign}
    & \expt_k \left[\nbr{\sum_{i=1}^{\kbar}\qbl^{(k,i)}}^2\right] = \expt_k\left[\nbr{\sum_{i=1}^{\kbar}\sum_{m\in\st^{(k,i)}}\nabla F_\ic(\wb^{(k,0)})}^2\right]=N^2\expt_k\left[\nbr{\sum_{i=1}^{\kbar}\frac{1}{N}\sum_{m\in\st^{(k,i)}}\nabla F_\ic(\wb^{(k,0)})}^2\right] \nn \\
    &= N^2\expt_k\left[\nbr{\sum_{i=1}^{\kbar} \lcb \frac{1}{N}\sum_{m\in\st^{(k,i)}}\nabla F_\ic(\wb^{(k,0)}) - \frac{1}{M/\kbar}\sum_{m\in\sigma(i)}\nabla F_\ic(\wb^{(k,0)}) + \frac{1}{M/\kbar}\sum_{m\in\sigma(i)}\nabla F_\ic(\wb^{(k,0)}) \rcb}^2\right] \nn \\
    &= N^2 \expt_k\left[\nbr{\sum_{i=1}^{\kbar} \lcb \frac{1}{N}\sum_{m\in\st^{(k,i)}}\nabla F_\ic(\wb^{(k,0)})-\frac{1}{M/\kbar}\sum_{m\in\sigma(i)}\nabla F_\ic(\wb^{(k,0)}) \rcb }^2\right] + N^2\nbr{\frac{1}{M/\kbar}\sum_{i=1}^\kbar\sum_{m\in\sigma(i)}\nabla F_\ic(\wb^{(k,0)})}^2 \nn \\
    & \leq N\kbar^2 \left[ \left(\frac{M/\kbar-N}{M/\kbar-1}\right)\gamma^2 + N \nbr{\nabla F(\wb^{(k,0)})}^2 \right]. \label{eq:3-1-6}
\end{flalign}
following steps analogous to the proof of \eqref{eq:3-2-2}. Finally, we can plug in \Cref{eq:3-1-3}, \cref{eq:3-1-5} and \Cref{eq:3-1-6} into \Cref{eq:3-0-2} to get
\begin{align}
    & \expt_k[F(\wb^{(k+1,0)})]-F(\wb^{(k,0)}) \nn \\
    & \leq -\lr\kbar N\left(\frac{49}{50}-\frac{3\lr (\kbar-1) NL}{5}-L\lr N\kbar- \lr^3\kbar (\kbar-1)^2 N^3L^3\right)\nbr{\nabla F(\wb^{(k,0)})}^2 + \frac{9\lr^3 \alpha^2 L^2 N^3 \kbar(\kbar-1)^2}{2} \nn \\
    &+ \lr^4\kbar^3(\kbar-1)N^4L^3\alpha^2 + \frac{1}{2} \lr^4 \kbar^2 (\kbar-1)^2 N^3L^3\left(\frac{M/\kbar-N}{M/\kbar-1}\right)\gamma^2 + L \lr^2 N \kbar^2 \left(\frac{M/\kbar-N}{M/\kbar-1}\right)\gamma^2 \nn \\
    &\leq -\frac{\lr\kbar N}{2}\nbr{\nabla F(\wb^{(k,0)})}^2 +\frac{47\lr^3 L^2 N^3 \kbar(\kbar-1)^2}{10} \alpha^2 + \frac{11L\lr^2N\kbar^2}{5}\left(\frac{M/\kbar-N}{M/\kbar-1}\right) \gamma^2  \\
    &
\begin{aligned}\expt_k[F(\wb^{(k+1,0)})]-F^*\leq (1-\lr \kbar N\mu)(F(\wb^{(k,0)})-F^*)+\frac{47\lr^3 L^2 N^3 \kbar(\kbar-1)^2}{10} \alpha^2 \\
    + \frac{11L\lr^2N\kbar^2}{5}\left(\frac{M/\kbar-N}{M/\kbar-1}\right) \gamma^2
    \end{aligned} \label{eq:3-2-4}
\end{align}
where the last inequality follows from \Cref{as2}. Unrolling \Cref{eq:3-2-4}, and using $\lr = \frac{\log(MK^2)}{\mu \kbar K N}$, we get
\begin{flalign}
    & \expt[F(\wb^{(K,0)})]-F^* \nn \\
    & \leq\left(1-\frac{\log(MK^2)}{K}\right)^K(F(\wb^{(0,0)})-F^*) + \frac{47\lr^3  L^2 N^3 \kbar(\kbar-1)^2}{10\lr\kbar N\mu} \alpha^2  + \frac{11L\lr^2N\kbar^2}{5\lr\kbar N\mu}\left(\frac{M/\kbar-N}{M/\kbar-1}\right) \gamma^2 \nn \\
    & \leq \frac{F(\wb^{(0,0)})-F^*}{MK^2}+\frac{47\lr^2  L^2 N^2(\kbar-1)^2}{10\mu} \alpha^2 + \frac{11L\lr\kbar}{5\mu}\left(\frac{M/\kbar-N}{M/\kbar-1}\right) \gamma^2 \nn \\
    &= \frac{F(\wb^{(0,0)})-F^*}{MK^2}+\frac{47\log^2(MK^2)  \kappa^2(\kbar-1)^2}{10\mu\kbar^2 K^2} \alpha^2 + \frac{11L\log(MK^2)}{5\mu^2NK}\left(\frac{M/\kbar-N}{M/\kbar-1}\right) \gamma^2 \nn \\
    &= \frac{F(\wb^{(0,0)})-F^*}{MK^2}+\tilo\left(\frac{\kappa^2(\kbar-1)^2 \alpha^2}{\mu \kbar^2K^2}\right)+\tilo\left(\frac{\kappa}{\mu NK}\left(\frac{M-N\kbar}{M-\kbar}\right) \gamma^2\right). \nn
\end{flalign}
Also, since we have total $T=\kbar K$ communication rounds, we can also express the bound in terms of $T$ as follows.
\begin{flalign}
    \expt[F(\wb^{(K,0)})]-F^*\leq
    \frac{\kbar^2(F(\wb^{(0,0)})-F^*)}{MT^2}+\tilo\left(\frac{\kappa^2(\kbar-1)^2 \alpha^2}{\mu T^2}\right)+\tilo\left(\frac{\kbar\kappa}{\mu NT}\left(\frac{M-N\kbar}{M-\kbar}\right) \gamma^2 \right). \nn
\end{flalign}

\end{proof}

\paragraph{Proof for \Cref{theo:1-1:kbar}}
We reiterate \Cref{eq:totalc:gd}, the total cost $C_{\text{GD}}$ to achieve an $\epsilon$ error:
\begin{align}
C_{\text{GD}}(\epsilon)=\frac{c_{\text{GD}}\kbar\gamma^2}{\epsilon N}\left(\frac{M/\kbar-N}{M/\kbar-1}\right). \nn
\end{align}
To get $C_{\text{GD}|\kbar>1}(\epsilon)<C_{\text{GD}|\kbar=1}(\epsilon)$ we have that the inequality
\begin{align}
\kbar\left(\frac{M/\kbar-N}{M/\kbar-1}\right)-\left(\frac{M-N}{M-1}\right)<0 \nn 
\end{align}
should be true for $\kbar>1$. Rearranging the terms, we get
\begin{align}
    N(M-1)\kbar^2+(N-M^2)\kbar+M^2-MN>0
\end{align}
Since $(M(M-2N)+N)^2=(N-M^2)^2-4N(M^2-MN)(M-1)\geq 0$ we have that for $\kbar>M(M-N)/N(M-1)$ we have $\kbar\left(\frac{M/\kbar-N}{M/\kbar-1}\right)-\left(\frac{M-N}{M-1}\right)<0$
Rearranging the terms slightly, we get
$$M < N + N \kbar \lp 1 - \frac{1}{M} \rp.$$
Recall that $\kbar \in [\frac{M}{N}]$. For $\kbar = M/N$, we get
$$N + N \frac{M}{N} \lp 1 - \frac{1}{M} \rp = M-1+N > M$$
for any $N>1$. However, if $\kbar = \frac{M}{N}-1$, we get
$$N + N \lp \frac{M}{N} - 1 \rp \lp 1 - \frac{1}{M} \rp = M-1+N-N+\frac{N}{M} =  M-1+\frac{N}{M} < M$$
for any $N>1$. Consequently, the only case when $\kbar > 1$ gives benefit over $\kbar = 1$ is when $\kbar = M/N$, meaning full client participation in every client group.

\subsection{Proofs on Intermediate Lemmas}
\label{sec:proofs_int_results_GD}

\begin{proof}[Proof of \cref{lem0-0}]
\begin{align}
    \nbr{\sum_{j=1}^i\frac{1}{M/\kbar}\sum_{m\in\sigma(j)}\nabla F_m(\wb)}\leq \sum_{j=1}^i\nbr{\frac{1}{M/\kbar}\sum_{m\in\sigma(j)}\nabla F_m(\wb)-\nabla F(\wb) +\nabla F(\wb)} \nn \\
    \leq\sum_{j=1}^i\nbr{\frac{1}{M/\kbar}\sum_{m\in\sigma(j)}\nabla F_m(\wb)-\nabla F(\wb)}+i\nbr{\nabla F(\wb)} \leq i\alpha+i\nbr{\nabla F(\wb)}. \tag{using \cref{as3}}
\end{align}
\end{proof}

\begin{proof}[Proof of \cref{lemma:error_terms_GD}]
First, we bound $\nbr{\expt_k[\rbl^{(k,0)}]}$.
\begin{align}
    \nbr{\expt_k[\rbl^{(k,0)}]} &= \nbr{ \sum_{i=1}^{\kbar-1}\expt_k \left[\left(\prod_{j=i+2}^{\kbar} (\ibd-\lr\ssbl^{(k,j)})\right)\ssbl^{(k,i+1)}\sum_{j'=1}^i \qbl^{(k,j')}\right]} \nn \\
    & \leq \sum_{i=1}^{\kbar-1}\nbr{\expt_k\left[\left(\prod_{j=i+2}^{\kbar} (\ibd-\lr\ssbl^{(k,j)})\right)\ssbl^{(k,i+1)}\sum_{j'=1}^i \qbl^{(k,j')}\right]} \nn \\
    & \leq \sum_{i=1}^{\kbar-1}\expt_k\left[\underbrace{\nbr{\prod_{j=i+2}^{\kbar}(\ibd-\lr\ssbl^{(k,j)})}}_{A_1}\underbrace{\nbr{\ssbl^{(k,i+1)}}}_{A_2}\underbrace{\nbr{\sum_{j'=1}^i \qbl^{(k,j')}}}_{A3}\right]~~~~\text{(}\because\text{Submultiplicativity of Norms)} \label{eq:3-0-1-1}
\end{align}
Next, we bound $A_1,A_2,A_3$ separately as follows.
\begin{align}
    A_1 &= \nbr{\prod_{j=i+2}^{\kbar} (\ibd-\lr\ssbl^{(k,j)})}
    \leq \prod_{j=i+2}^{\kbar} \nbr{(\ibd-\lr\ssbl^{(k,j)})} \nn \\
    &= \prod_{j=i+2}^{\kbar} \nbr{I_d-\lr\sum_{m\in\st^{(k,j)}}\hb_m^{(k,j)}} \leq (1+\lr N L)^\kbar\leq e^{1/7}\leq 6/5, \label{eq:3-0-4}
\end{align}
where in \Cref{eq:3-0-4} we use \cref{as1}, and $\lr \leq \frac{1}{7 L N \kbar}$. Next, we bound $A_2$.
\begin{align}
    A_2=\nbr{\ssbl^{(k,i+1)}}=\nbr{\sum_{m\in\st^{(k,i+1)}}\hb_m^{(k,i+1)}}\leq NL \label{eq:3-0-4-1}
\end{align}
Next, we bound $A_3$ as follows:
\begin{align}
    A_3 = \nbr{\sum_{j'=1}^i \expt_k\left[\qbl^{(k,j')}\right]} &= \nbr{\sum_{j'=1}^i \expt_k\left[\sum_{m\in\st^{(k,j')}} \nabla F_m(\wb^{(k,0)})\right]} = N \nbr{\sum_{j'=1}^i \frac{1}{M/\kbar} \sum_{m \in \sigma(j')} \nabla F_m(\wb^{(k,0)})} \nn \\
    & \leq i N \alpha + i N \nbr{\nabla F(\wb)}. \label{eq:3-0-4-2}
\end{align}
where the last inequality follows from \Cref{lem0-0}. Substituting the bounds from \eqref{eq:3-0-4}-\eqref{eq:3-0-4-2} in \eqref{eq:3-0-1-1}, we get the bound in \eqref{eq:3-1-0}. 

Next, we derive the bound in \Cref{eq:3-2-3}.
\begin{align}
    \expt_k\left[\nbr{\rbl^{(k,0)}}^2\right] &= \expt_k\left[\nbr{\sum_{i=1}^{\kbar-1}\left(\prod_{j=i+2}^{\kbar} (\ibd-\lr\ssbl^{(k,j)})\right)\ssbl^{(k,i+1)}\sum_{j'=1}^i \qbl^{(k,j')}}^2\right] \nn \\
    & \leq (\kbar-1)\sum_{i=1}^{\kbar-1}\expt_k\left[\nbr{\left(\prod_{j=i+2}^{\kbar} (\ibd-\lr\ssbl^{(k,j)})\right)\ssbl^{(k,i+1)}\sum_{j'=1}^i \qbl^{(k,j')}}^2\right] \nn \\
    & \leq (\kbar-1)\sum_{i=1}^{\kbar-1}\expt_k \left[ \nbr{\prod_{j=i+2}^{\kbar} (\ibd-\lr\ssbl^{(k,j)})}^2 \nbr{\ssbl^{(k,i+1)}}^2 \nbr{\sum_{j'=1}^i \qbl^{(k,j')}}^2 \right]. \label{eq:3-2-3-0}
\end{align}
Observe that
\begin{align}
    \nbr{(\ibd-\lr\ssbl^{(k,j)})}\leq 1+\lr\nbr{\sum_{m\in\st^{(k,i)}}\hb_m^{(k,i)}}\leq 1+\lr\sum_{m\in\st^{(k,i)}}\nbr{\hb_m^{(k,i)}}\leq 1+\lr NL,~\forall j\in[\kbar]. \label{eq:3-2-3-1}
\end{align}
Using \eqref{eq:3-0-4}, \eqref{eq:3-0-4-1} and \eqref{eq:3-2-3-1} in \eqref{eq:3-2-3-0}, we get
\begin{align}
    \expt_k\left[\nbr{\rbl^{(k,0)}}^2\right] \leq \frac{36 (\kbar-1) N^2 L^2}{25} \sum_{i=1}^{\kbar-1}\expt_k\left[\nbr{\sum_{j'=1}^i \qbl^{(k,j')}}^2\right] \label{eq:3-2-0}
\end{align}
Lastly, we bound the expected value in \Cref{eq:3-2-0}:
\begin{align}
    &\sum_{i=1}^{\kbar-1}\expt_k\left[\nbr{\sum_{j'=1}^i \qbl^{(k,j')}}^2\right]=N^2\sum_{i=1}^{\kbar-1}\expt_k\left[\nbr{\sum_{j'=1}^i\frac{1}{N} \sum_{m\in\mathcal{S}^{(k,j')}}\nabla F_m(\wb^{(k,0)})}^2\right] \nn \\
    &=N^2\sum_{i=1}^{\kbar-1}\expt_k\nbr{\sum_{j'=1}^i\lcb \frac{1}{N} \sum_{m\in\mathcal{S}^{(k,j')}}\nabla F_m(\wb^{(k,0)})-\frac{1}{M/\kbar}\sum_{m\in\sigma(j')}\nabla F_m(\wb^{(k,0)})+\frac{1}{M/\kbar}\sum_{m\in\sigma(j')}\nabla F_m(\wb^{(k,0)}) \rcb}^2 \nn \\
    &= N^2\sum_{i=1}^{\kbar-1}\expt_k\left[\nbr{\sum_{j'=1}^i \lcb \frac{1}{N} \sum_{m\in\mathcal{S}^{(k,j')}}\nabla F_m(\wb^{(k,0)})-\frac{1}{M/\kbar}\sum_{m\in\sigma(j')}\nabla F_m(\wb^{(k,0)}) \rcb}^2\right] \nn \\
    & \qquad \qquad +N^2\sum_{i=1}^{\kbar-1}\nbr{\frac{1}{M/\kbar}\sum_{j'=1}^i\sum_{m\in\sigma(j')}\nabla F_m(\wb^{(k,0)})}^2 \nn \tag{Cross-terms are zero}\\
    & \leq N^2\sum_{i=1}^{\kbar-1}i\sum_{j'=1}^i\expt_k\left[\nbr{\frac{1}{N} \sum_{m\in\mathcal{S}^{(k,j')}}\nabla F_m(\wb^{(k,0)})-\frac{1}{M/\kbar}\sum_{m\in\sigma(j')}\nabla F_m(\wb^{(k,0)})}^2\right] \tag{Using \cref{lem:sum_of_squares}} \\
    & \qquad \qquad + N^2 \sum_{i=1}^{\kbar-1} \nbr{\frac{1}{M/\kbar} \sum_{j'=1}^i \sum_{m \in \sigma(j')} \nabla F_m(\wb^{(k,0)})}^2 \nn \\
    &= N^2\sum_{i=1}^{\kbar-1}i\sum_{j'=1}^i \frac{1}{N}\left(\frac{M/\kbar-N}{M/\kbar-1}\right) \frac{1}{M/\kbar} \sum_{m'\in\sigma(j')}\nbr{\nabla F_{m'}(\wb)-\frac{1}{M/\kbar}\sum_{m\in\sigma(j')}\nabla F_m(\wb)}^2 \tag{Using without replacement sampling} \\
    & \qquad \qquad +N^2\sum_{i=1}^{\kbar-1}\nbr{\frac{1}{M/\kbar}\sum_{j'=1}^i\sum_{m\in\sigma(j')}\nabla F_m(\wb^{(k,0)})}^2 \nn \\
    &\leq N^2\sum_{i=1}^{\kbar-1}i\sum_{j'=1}^i\frac{1}{N}\left(\frac{M/\kbar-N}{M/\kbar-1}\right)\gamma^2 + N^2\sum_{i=1}^{\kbar-1}\nbr{\frac{1}{M/\kbar}\sum_{j'=1}^i\sum_{m\in\sigma(j')}\nabla F_m(\wb^{(k,0)})}^2 \tag{Using \cref{as3}} \\
    & \leq \sum_{i=1}^{\kbar-1} i^2 N \left(\frac{M/\kbar-N}{M/\kbar-1}\right)\gamma^2 + {N^2}\sum_{i=1}^{\kbar-1} {2i^2} \left(\nbr{\nabla F(\wb^{(k,0)})}^2+{\alpha^2}\right) \tag{using \Cref{lem0-0}} \\
    & \leq \frac{1}{3} (\kbar-1) \kbar^2 N \left[ \left(\frac{M/\kbar-N}{M/\kbar-1}\right)\gamma^2 + 2N \left(\nbr{\nabla F(\wb^{(k,0)})}^2+{\alpha^2}\right) \right]. \label{eq:3-2-2} 
\end{align}
Plugging \Cref{eq:3-2-2} into \Cref{eq:3-2-0} we get
\begin{align}
    \expt_k \left[ \nbr{\rbl^{(k,0)}}^2 \right] & \leq \frac{12 (\kbar-1)^2 \kbar^2 N^3 L^2}{25} \left[ \left(\frac{M/\kbar-N}{M/\kbar-1}\right)\gamma^2 + 2N \left(\nbr{\nabla F(\wb^{(k,0)})}^2+{\alpha^2}\right) \right]. \nn
\end{align}
which concludes the proof.
\end{proof}

\newpage
\section{Proofs for the CyCP+Local SGD Case}
\label{app:ccp+locsgd}
The proof of \cref{theo2:locSGD} is presented in this section. For simplicity, the proof is presented as follows: first, in \cref{sec:noisy_GD_localSGD}, the model update steps are shown to be noisy gradient descent steps. Next, in \cref{sec:int_results_localSGD}, we present some intermediate results, which shall be used in the analysis, followed by the proof of \cref{theo2:locSGD} in \cref{sec:proof_theo2_localSGD}. Finally, in \cref{sec:proofs_int_results_localSGD}, we present the proofs of the intermediate results.

Similar to what we have did in \cref{sec:proof_theo1_GD}, we define the $\sigma$-algebra generated by the randomness in the algorithm till cycle epoch $k$ as follows: $\mc F_k \triangleq \sigma \lcb \{ \wb^{(1,i)} \}_{i=1}^\kbar, \{ \wb^{(2,i)} \}_{i=1}^\kbar, \dots, \{ \wb^{(k-1,i)} \}_{i=1}^\kbar \rcb$. We use $\expt_k [\cdot]$ as the shorthand for the expectation $\expt [\cdot | \mc F_k]$.

\subsection{Global Model Updates as Noisy Gradient Descent Step}
\label{sec:noisy_GD_localSGD}

Recall that the global model update rule is
\begin{align}
    \wb^{(k,i)}-\wb^{(k,i-1)}=-\lr\sum_{m\in\mathcal{S}^{(k,i)}}\sum_{l=0}^{\tau-1} \nabla F_m(\wb_m^{(k,i-1,l)},\xi_m^{(k,i-1,l)}) \label{eqn:locsgd}
\end{align}
where $\nabla F_m(\wb_m^{(k,i-1,l)},\xi_m^{(k,i-1,l)})=\frac{1}{b}\sum_{\xi \in \xi_m^{(k,i-1,l)}}\nabla \lff(\wb_m^{(k,i-1,l)}, \xi)$ is the stochastic gradient computed using a mini-batch $\xi_m^{(k,i-1,l)}$ of size $b$ that is randomly sampled from client $\ic$'s local dataset $\mathcal{B}_\ic$. Recall that $k$ is a \textit{semi-epoch} index which denotes each communication round when we have traversed the $\kbar$ groups of clients by sampling $N$ clients from each group $\sigma(i),~i\in[\kbar]$ uniformly at random without replacement. The index $i\in [\kbar]$ denotes each inner communication round. Recall that $\wb^{(k+1,0)} = \wb^{(k,\kbar)}$ for all $k \in [K]$. With \Cref{eqn:locsgd} we get
\begin{align}
      & \wb^{(k,i)}-\wb^{(k,i-1)} =- \lr \sum_{m\in\mathcal{S}^{(k,i)}}\sum_{l=0}^{\tau-1} \left(\nabla F_m(\wb_m^{(k,i-1,l)},\xi_m^{(k,i-1,l)})-\nabla F_m(\wb^{(k,i-1)})+\nabla F_m(\wb^{(k,i-1)})\right) \nn \\
      &= -\lr \sum_{m\in\mathcal{S}^{(k,i)}}\underbrace{\sum_{l=0}^{\tau-1} (\nabla F_m(\wb_m^{(k,i-1,l)},\xi_m^{(k,i-1,l)})-\nabla F_m(\wb^{(k,i-1)}))}_{\defeq\db_m^{(k,i)}}-\lr\tau\sum_{m\in\mathcal{S}^{(k,i)}}\nabla F_m(\wb^{(k,i-1)}) \label{eq:0-0} \\
      &= -\lr \underbrace{\sum_{m\in\mathcal{S}^{(k,i)}}\db_m^{(k,i)}}_{\defeq\dbl^{(k,i)}}-\lr\tau\sum_{m\in\mathcal{S}^{(k,i)}}\left(\nabla F_m(\wb^{(k,0)})+\hb_m^{(k,i)}(\wb^{(k,i-1)}-\wb^{(k,0)}\right) \tag{$\hb_m^{(k,i)}$ defined in \eqref{eq:0-1}} \\
      &= -\lrt\underbrace{(\dbl^{(k,i)}/\tau+\qbl^{(k,i)})}_{\defeq \qbt^{(k,i)}}-\lrt \underbrace{\sum_{m\in\mathcal{S}^{(k,i)}}\hb_m^{(k,i)}}_{\defeq\ssbl^{(k,i)}}(\wb^{(k,i-1)}-\wb^{(k,0)}) \label{eq:5-0}
\end{align}
where $\lrt=\tau\lr$ and $\qbl^{(k,i)}\defeq\sum_{m\in\st^{(k,i)}}\nabla F_m(\wb^{(k,0)})$. Unrolling \Cref{eq:5-0}, we get
\begin{align}
    \wb^{(k+1,0)}-\wb^{(k,0)}=-\lrt \sum_{i=1}^\kbar \qbt^{(k,i)} + \lrt^2\underbrace{\sum_{i=1}^{\kbar-1}\left(\prod_{j=i+2}^{\kbar} (\ibd-\lrt\ssbl^{(k,j)})\right)\ssbl^{(k,i+1)}\sum_{j'=1}^i \qbt^{(k,j')}}_{\defeq \rbt^{(k,0)}}
\label{eq:updatelocsgd}
\end{align}
Conditioning on $\mc F_k$, we get
\begin{align}
    \expt_k[\wb^{(k+1,0)}]-\wb^{(k,0)} &= -\frac{\lrt}{\tau}\sum_{i=1}^\kbar\expt_k\left[\dbl^{(k,i)}\right] - \lrt \expt_k \left[ \sum_{i=1}^\kbar \sum_{m\in\st^{(k,i)}}\nabla F_m(\wb^{(k,0)}) \right] + \lrt^2 \expt_k[\rbt^{(k,0)}] \nn \\
    &=-\frac{\lrt}{\tau}\sum_{i=1}^\kbar\expt_k\left[\dbl^{(k,i)}\right]-\lrt\kbar N\nabla F(\wb^{(k,0)})+\lrt^2\expt_k[\rbt^{(k,0)}] \nn \\
    &=-\frac{\lrt}{\tau}\sum_{i=1}^\kbar\expt_k\left[\sum_{m\in\st^{(k,i)}} \db_m^{(k,i)}\right]-\lrt\kbar N\nabla F(\wb^{(k,0)})+\lrt^2\expt_k[\rbt^{(k,0)}].
    \label{eq:5-0-2}
\end{align}

\subsection{Intermediate Results}
\label{sec:int_results_localSGD}

\begin{lemma}[Bound on the Norm of the Error Term Arrising Due to Local SGD Steps]
\label{lem2-4}
We can bound the norm of the error term ($\mathbf{d}_m^{(k,i)}$ defined in \Cref{eq:0-0}) arising due to the local SGD steps as follows:
\begin{flalign}
    & \nbr{\sum_{i=1}^\kbar\expt\left[\sum_{m\in\st^{(k,i)}} \db_m^{(k,i)}\right]}\leq \frac{3\lr LN\kbar(\tau-1)\tau}{5}\nbr{\nabla F(\wb^{(k,0)})}+\frac{14\lr LN\kbar(\tau-1)\tau\nu}{25}+\frac{\lr L N (\kbar-1) (\tau-1)\tau \alpha}{36}. \nn
\end{flalign}
\end{lemma}

\begin{lemma}[Bound on the Norm Square of the Error Term Arrising Due to Local SGD Steps]
\label{lem2-6}
Similar to the bound in \Cref{lem2-4} we can bound the square of the error term arising due to the local SGD steps as follows:
\begin{flalign}
    \frac{3L\lrt^2}{2\tau^2}\expt\left[\nbr{\sum_{i=1}^\kbar \dbl^{(k,i)}}^2\right]
    &\leq\frac{\lrt N\kbar}{200}\nbr{\nabla F(\wb^{(k,0)})}^2+\frac{81L\lrt^2\kbar^2N\sigma^2}{50\tau}+\frac{209\lrt^2\lr^2L^3\kbar^2N^2\tau(\tau-1)\nu^2}{50} \nn \\
    & \quad + \frac{L^3\lrt^4\kbar^4N^3\gamma^2}{20}\left(\frac{M/\kbar-N}{M/\kbar-1}\right) + \frac{9 L^3\lrt^4\kbar^3(\kbar-1)N^4\alpha^2}{100}. \nn
\end{flalign}
\end{lemma}
Both \Cref{lem2-4} and \Cref{lem2-6} bound the error term that arises due to taking local SGD steps instead of a full GD step. The bound mainly depends on the number of local steps $\tau$ and the intra-group and inter-group data heterogeneity $\gamma$ and $\alpha$. Recall that $\nu=\gamma+\alpha$ and $\lrt=\lr\tau$.

\begin{lemma}[Bound on the Norm of the Error Term Arrising Due to \cycp]
\label{lem2-5}
\begin{flalign}
    & \lrt^2 \nbr{\expt_k[\rbt^{(k,0)}]}
    \leq \frac{17\lrt N\kbar}{250}\nbr{\nabla F(\wb^{(k,0)})}+\frac{84\lrt^2\lr N^2L^2\kbar(\kbar-1)(\tau-1)\nu}{125}+\frac{16\lrt^2N^2L\kbar(\kbar-1)\alpha}{25}. \nn 
\end{flalign}
\end{lemma}
\begin{lemma}[Bound on the Norm Square of the Error Term Arrising Due to \cycp]
\label{lem2-7}
\begin{flalign}
    \frac{3L\lrt^4}{2}\expt\left[\nbr{\rbt^{(k,0)}}^2\right]
    &\leq \frac{9\lrt N\kbar}{500}\nbr{\nabla F(\wb^{(k,0)})}^2+\frac{3\lrt^2L\kbar^2N\sigma^2}{200\tau}+\frac{173\lrt^4\kbar^4N^3L^3\gamma^2}{20}\left(\frac{M/\kbar-N}{M/\kbar-1}\right) \nn \\
    & \quad + \frac{3 \lrt^2 \lr^2 L^3 \kbar^2 N^2 \tau (\tau-1)\nu^2}{50} + \frac{59 \lrt^4 \kbar^3 (\kbar-1) N^4 L^3 \alpha^2}{10}. \nn 
\end{flalign}
\end{lemma}
Both \Cref{lem2-5} and \Cref{lem2-7} bound the error term that arises due to cyclic client participation where we model the proof sketch to be a large full gradient step plus these error terms. The norm square error depends on the additional variance term dependent on $\sigma^2$, coming from the stochastic gradients.



\begin{lemma}[Bound on the Distance between the Initial Global Model at each Cycle-Epoch and its Trajectory within that Cycle-Epoch]
\label{lem2-3}
\begin{align}
&\begin{aligned}
&\sum_{r=1}^\kbar\expt\left[\nbr{\wb^{(k,r-1)}-\wb^{(k,0)}}^2\right]\\
&\leq \frac{83\lrt^2N^2\kbar^2(\kbar-1)}{20}\nbr{\nabla F(\wb^{(k,0)})}^2+\frac{51\lrt^2N\kbar^2(\kbar-1)\sigma^2}{50\tau}+\frac{3\lr^2\kbar\tau(\tau-1)\nu^2}{100} \nn \\
&+\frac{51\lrt^2N\kbar^2(\kbar-1)\gamma^2}{25} \left(\frac{M/\kbar-N}{M/\kbar-1}\right)+\frac{41\lrt^2N^2\kbar^2(\kbar-1)\alpha^2}{10}.
\end{aligned}
\end{align}
\end{lemma}
\Cref{lem2-3} bounds the distance between the initial global model given at the start of each cycle-epoch, and its trajectory through the $\kbar$ communication rounds within that cycle-epoch.

\subsection{Proof of \cref{theo2:locSGD}}
\label{sec:proof_theo2_localSGD}

\begin{thm*}[Convergence with CyCP+SGD] With Assumptions \ref{as1}, \ref{as2}, \ref{as3}, and \ref{as5} and step-size $\lr=\log(MK^2)/\tau\mu N\kbar K$, for $K\geq 10\kappa\log{(MK^2)}$ communication rounds where $\kappa=L/\mu$, the convergence error is bounded as:
\begin{align}
    &\expt[F(\wb^{(K,0)})]-F^*\leq\frac{F(\wb^{(0,0)})-F^*}{MK^2}+\bigto{\frac{\kappa^2 (\kbar-1)\alpha^2}{\mu \kbar K^2}} +\bigto{\frac{\kappa\gamma^2}{\mu N K}\left(\frac{M/\kbar-N}{M/\kbar-1}\right)}\nn \\
    &\quad +\bigto{\frac{\kappa\sigma^2}{\mu\tau N K}}+\bigto{\frac{\kappa^2(\tau-1)\nu^2}{\mu\tau N^2\kbar^2K^2}}\nn
\end{align}
where $\tilo(\cdot)$ subsumes all log-terms and constants. With $T=K\kbar$ we have in terms of communication rounds
\begin{align}
\begin{aligned}
    \expt[F(\wb^{(K,0)})]-F^*\leq\frac{\kbar^2(F(\wb^{(0,0)})-F^*)}{MT^2}+\bigto{\frac{\kappa^2 \kbar(\kbar-1)\alpha^2}{\mu T^2}}+\bigto{\frac{\kbar\kappa\gamma^2}{\mu N T}\left(\frac{M/\kbar-N}{M/\kbar-1}\right)} \nn \\
+\bigto{\frac{\kbar\kappa\sigma^2}{\mu\tau N T}}+\bigto{\frac{\kappa^2(\tau-1)\nu^2}{\mu\tau N^2T^2}}\nn
\end{aligned}
\end{align}
\end{thm*}

Throughout the proof we use $\lr=\log{(MK^2)}/\mu\tau N \kbar K$ and $K\geq 10\kappa\log{(MK^2)}$ which leads to $\lr \leq 1/(10 \tau N L \kbar)$. Again, with the $L$-smoothness property of the global objective we have that
\begin{align}
    \expt_k[F(\wb^{(k+1,0)})]-F(\wb^{(k,0)})\leq \inner{\nabla F(\wb^{(k,0)})}{\expt_k[\wb^{(k+1,0)}]-\wb^{(k,0)}} + \frac{L}{2}\expt_k \left[\nbr{\wb^{(k+1,0)}-\wb^{(k,0)}}^2\right] \label{eq:5-0-1}
\end{align}
where the expectation here is over the selected client sets $\st^{(k,i)},~i\in[\kbar]$ and stochastic gradients. First we bound the inner product term in the RHS \Cref{eq:5-0-1}. Using \eqref{eq:5-0-2},
\begin{align}
    & \inner{\nabla F(\wb^{(k,0)})}{\expt_k[\wb^{(k+1,0)}]-\wb^{(k,0)}} \nn \\
    &= \inner{\nabla F(\wb^{(k,0)})}{-\frac{\lrt}{\tau}\sum_{i=1}^\kbar\expt_k \left[\sum_{m\in\st^{(k,i)}}
    \db_m^{(k,i)}\right]-\lrt\kbar N\nabla F(\wb^{(k,0)})+\lrt^2\expt_k[\rbt^{(k,0)}]} \nn \\
    & \leq \frac{\lrt}{\tau} \nbr{\nabla F(\wb^{(k,0)})} \nbr{\sum_{i=1}^\kbar\expt_k \left[\sum_{m\in\st^{(k,i)}}
    \db_m^{(k,i)}\right]} - \lrt \kbar N\nbr{\nabla F(\wb^{(k,0)}}^2+\lrt^2\nbr{\nabla F(\wb^{(k,0)})} \nbr{\expt_k[\rbt^{(k,0)}]}.
    \label{eq:5-0-3}
\end{align}
$\nbr{\sum_{i=1}^\kbar\expt_k \left[\sum_{m\in\st^{(k,i)}} \db_m^{(k,i)}\right]}$ and $\nbr{\expt_k[\rbt^{(k,0)}]}$ are already bounded in \cref{lem2-4} and \cref{lem2-5} respectively.
Plugging these bounds in \Cref{eq:5-0-3} we have
\begin{flalign}
     & \inner{\nabla F(\wb^{(k,0)})}{\expt_k[\wb^{(k+1,0)}]-\wb^{(k,0)}} \nn \\
    & \leq \frac{\lrt}{\tau}\nbr{\nabla F(\wb^{(k,0)})}\left(\frac{3\lr LN\kbar(\tau-1)\tau}{5}\nbr{\nabla F(\wb^{(k,0)})}+\frac{14\lr LN\kbar(\tau-1)\tau\nu}{25}+\frac{\lr L N (\kbar-1) (\tau-1)\tau \alpha}{36}\right) \nn \\
    & \quad - \lrt \kbar N \nbr{\nabla F(\wb^{(k,0)}}^2+\nbr{\nabla F(\wb^{(k,0)})}\left(\frac{17\lrt N\kbar}{250}\nbr{\nabla F(\wb^{(k,0)})}+\frac{84\lrt^2\lr N^2L^2\kbar(\kbar-1)(\tau-1)\nu}{125}\right. \nn \\
    & \qquad \qquad \left. + \frac{16\lrt^2 N^2 L \kbar(\kbar-1)\alpha}{25} \right) \nn \\
    & \leq -\lrt\kbar N\left(1-\frac{3}{50}-\frac{17}{250}\right)\nbr{\nabla F(\wb^{(k,0)})}^2+\left(\frac{\lrt^{1/2}N^{1/2}\kbar^{1/2}}{25}\nbr{\nabla F(\wb^{(k,0)})}\right)(14\lrt^{1/2}\lr L N^{1/2}\kbar^{1/2}(\tau-1)\nu) \nn \\
    & \quad + \left( \frac{\lrt^{1/2}N^{1/2}\kbar^{1/2}}{40}\nbr{\nabla F(\wb^{(k,0)})}\right)(28\lrt^{3/2}\lr L^2 N^{3/2}\kbar^{3/2}(\tau-1)\nu) \nn \\
    & \quad + \left( \frac{\lrt^{1/2}N^{1/2}\kbar^{1/2}}{36}\nbr{\nabla F(\wb^{(k,0)})}\right)\left(\frac{\lrt^{1/2}\lr L N^{1/2}(\kbar-1)(\tau-1)\alpha}{\kbar^{1/2}}\right) \nn \\
    & \quad + \left( \frac{\lrt^{1/2}N^{1/2}\kbar^{1/2}}{25}\nbr{\nabla F(\wb^{(k,0)})}\right)\left(16\lrt^{3/2} L N^{3/2}\kbar^{1/2}(\kbar-1)\alpha\right) \nn~~~~~~~~{(\because~\lr \leq 1/(10 \tau N L \kbar))} \\
    & \leq - \lrt \kbar N\left(1-\frac{3}{50}-\frac{17}{250}-\frac{1}{2\times40^2}-\frac{1}{2\times36^2}-\frac{1}{2\times25^2}\right)\nbr{\nabla F(\wb^{(k,0)})}^2+98\lrt\lr^2 L^2 N\kbar(\tau-1)^2\nu^2 \nn \\
    & \quad + 392\lrt^{3}\lr^2 L^4 N^{3}\kbar^{3}(\tau-1)^2\nu^2+\frac{\lrt\lr^2 L^2 N(\kbar-1)^2(\tau-1)^2\alpha^2}{2\kbar}+{128\lrt^{3} L^2 N^{3}\kbar(\kbar-1)^2\alpha^2} \nn ~~~~~~~~~(\because~\Cref{lem:Young}) \\
    & \leq - \frac{107\lrt\kbar N}{125}\nbr{\nabla F(\wb^{(k,0)})}^2+102\lrt\lr^2 L^2 N\kbar(\tau-1)^2\nu^2 +\frac{\lrt\lr^2 L^2 N(\kbar-1)^2(\tau-1)^2\alpha^2}{2\kbar}+{128\lrt^{3} L^2 N^{3}\kbar(\kbar-1)^2\alpha^2}. \label{eq:7-0-7}
\end{flalign}
Now we bound the second term in the RHS of \Cref{eq:5-0-1} as follows
\begin{align}
    & \frac{L}{2}\expt_k \left[\nbr{\wb^{(k+1,0)}-\wb^{(k,0)}}^2\right] = \frac{L}{2} \expt_k \left[ \nbr{-\lrt \sum_{i=1}^\kbar \qbt^{(k,i)} + \lrt^2\rbt^{(k,0)}}^2\right] \nn \\
    &= \frac{L}{2} \expt_k \left[\nbr{
    -\lrt \sum_{i=1}^\kbar \dbl^{(k,i)}/\tau-\lrt \sum_{i=1}^\kbar\qbl^{(k,i)}+\lrt^2\rbt^{(k,0)}}^2\right] \tag{Using \eqref{eq:updatelocsgd}} \\
    & \leq \frac{3L\lrt^2}{2\tau^2}\expt_k \left[\nbr{\sum_{i=1}^\kbar \dbl^{(k,i)}}^2\right] + \frac{3L\lrt^2}{2}\expt_k \left[\nbr{\sum_{i=1}^\kbar \qbl^{(k,i)}}^2\right] + \frac{3L\lrt^4}{2}\expt_k \left[\nbr{\rbt^{(k,0)}}^2\right]. \label{eq:5-0-7}
\end{align}
where in \Cref{eq:5-0-7} we use $\|a+b+c\|^2\leq3(\|a\|^2+\|b\|^2+\|c\|^2)$. Substituting the bounds from \Cref{eq:3-1-6}, \cref{lem2-6} and \cref{lem2-7} into \Cref{eq:5-0-7} we get
\begin{flalign}
&\begin{aligned}
&\frac{L}{2}\expt_k \left[\nbr{\wb^{(k+1,0)}-\wb^{(k,0)}}^2\right]\leq  \frac{\lrt N\kbar}{200}\nbr{\nabla F(\wb^{(k,0)})}^2+\frac{81L\lrt^2\kbar^2N\sigma^2}{50\tau}+\frac{209\lrt^2\lr^2L^3\kbar^2N^2\tau(\tau-1)\nu^2}{50} \nn \\
&+\frac{L^3\lrt^4\kbar^4N^3\gamma^2}{20}\left(\frac{M/\kbar-N}{M/\kbar-1}\right)+\frac{9 L^3\lrt^4\kbar^3(\kbar-1)N^4\alpha^2}{100}+\frac{3L\lrt^2N\gamma^2\kbar^2}{2}\left(\frac{M/\kbar-N}{M/\kbar-1}\right) \nn \\
&+\frac{3\lrt N\kbar}{20}\nbr{\nabla F(\wb^{(k,0)})}^2+\frac{9\lrt N\kbar}{500}\nbr{\nabla F(\wb^{(k,0)})}^2+\frac{3\lrt^2L\kbar^2N\sigma^2}{200\tau}+\frac{173\lrt^4\kbar^4N^3L^3\gamma^2}{20}\left(\frac{M/\kbar-N}{M/\kbar-1}\right) \nn \\
&+\frac{3\lrt^2\lr^2L^3\kbar^2N^2\tau(\tau-1)\nu^2}{50}+\frac{59\lrt^4\kbar^3(\kbar-1)N^4L^3\alpha^2}{10}
\end{aligned}\\
&\begin{aligned}
&\leq  \lrt N\kbar\left(\frac{1}{200}+\frac{3}{20}+\frac{9}{500}\right)\nbr{\nabla F(\wb^{(k,0)})}^2+\frac{L\lrt^2\kbar^2N\sigma^2}{\tau}\left(\frac{81}{50}+\frac{3}{200}\right)\\
&+(\lrt^2\lr^2L^3\kbar^2N^2\tau(\tau-1)\nu^2)\left(\frac{209}{50}+\frac{3}{50}\right)+L\lrt^2\kbar^2N\gamma^2\left(\frac{M/\kbar-N}{M/\kbar-1}\right)\left(\frac{174}{2000}+\frac{3}{2}\right) \nn \\
&+{L^3\lrt^4\kbar^3(\kbar-1)N^4\alpha^2}\left(\frac{9}{100}+\frac{59}{10}\right)
\end{aligned}~~~~~~~~{(\because~\lr \leq 1/(10 \tau N L \kbar))}\\
&\begin{aligned}
&\leq  \frac{173\lrt N\kbar}{1000}\nbr{\nabla F(\wb^{(k,0)})}^2+\frac{41L\lrt^2\kbar^2N\sigma^2}{25\tau}+\frac{41\lrt\lr^2L^2\kbar N\tau(\tau-1)\nu^2}{250}+\frac{8L\lrt^2\kbar^2N\gamma^2}{5}\left(\frac{M/\kbar-N}{M/\kbar-1}\right)\\
&+6L^3\lrt^4\kbar^3(\kbar-1)N^4\alpha^2
\end{aligned} \label{eq:7-3-1}
\end{flalign}
Using the bound derived in \Cref{eq:7-3-1} and \Cref{eq:7-0-7} for the upper bound in \Cref{eq:5-0-1} we have
\begin{flalign}
& \expt_k[F(\wb^{(k+1,0)})]-F(\wb^{(k,0)}) \nn \\
& \leq -\frac{\lrt\kbar N}{2}\nbr{\nabla F(\wb^{(k,0)})}^2+103\lrt\lr^2 L^2 N\kbar(\tau-1)\tau\nu^2+\frac{\lrt\lr^2 L^2 N(\kbar-1)^2(\tau-1)^2\alpha^2}{2\kbar} \nn \\
& \quad + {129\lrt^{3} L^2 N^{3}\kbar^2(\kbar-1)\alpha^2}+\frac{41L\lrt^2\kbar^2N\sigma^2}{25\tau}+\frac{8L\lrt^2\kbar^2N\gamma^2}{5}\left(\frac{M/\kbar-N}{M/\kbar-1}\right)
\end{flalign}
With \Cref{as2} we have 
\begin{flalign}
    & \expt[F(\wb^{(k+1,0)})]-F^* \nn \\
    & \leq (1-\lrt\kbar N\mu)(\expt[F(\wb^{(k,0)})]-F^*)+103\lrt\lr^2 L^2 N\kbar(\tau-1)\tau\nu^2+\frac{\lrt\lr^2 L^2 N(\kbar-1)^2(\tau-1)^2\alpha^2}{2\kbar} \nn \\
    & \quad + {129\lrt^{3} L^2 N^{3}\kbar^2(\kbar-1)\alpha^2}+\frac{41L\lrt^2\kbar^2N\sigma^2}{25\tau}+\frac{8L\lrt^2\kbar^2N\gamma^2}{5}\left(\frac{M/\kbar-N}{M/\kbar-1}\right),\label{eq:7-3-2}
\end{flalign}
and unrolling \Cref{eq:7-3-2} we have
\begin{flalign}
&\begin{aligned}
    &\expt[F(\wb^{(K,0)})]-F^* \nn \\
    & \leq\left(1-\frac{\log{(MK^2)}}{K}\right)^K(F(\wb^{(0,0)})-F^*)+\frac{103\lrt\lr^2 L^2 N\kbar(\tau-1)\tau\nu^2}{\lrt\kbar N\mu}
    +\frac{\lrt\lr^2 L^2 N(\kbar-1)(\tau-1)^2\alpha^2}{2\lrt\kbar N\mu} \nn \\
    & \quad +\frac{129\lrt^{3} L^2 N^{3}\kbar^2(\kbar-1)\alpha^2}{\lrt\kbar N\mu}
    +\frac{41L\lrt^2\kbar^2N\sigma^2}{25\tau\lrt\kbar N\mu}+\frac{8L\lrt^2\kbar^2N\gamma^2}{5\lrt\kbar N\mu}\left(\frac{M/\kbar-N}{M/\kbar-1}\right)
    \end{aligned} \\
    &\begin{aligned}
    &\leq\frac{F(\wb^{(0,0)})-F^*}{MK^2}+\frac{103\lr^2 L^2(\tau-1)\tau\nu^2}{\mu}+\frac{\lr^2 L^2 (\kbar-1)(\tau-1)^2\alpha^2}{2\kbar \mu}+\frac{129\lrt^{2} L^2 N^{2}\kbar(\kbar-1)\alpha^2}{\mu} \nn \\
    & \quad + \frac{41L\lrt\kbar\sigma^2}{25\tau\mu}+\frac{8L\lrt\kbar\gamma^2}{5\mu}\left(\frac{M/\kbar-N}{M/\kbar-1}\right)
    \end{aligned}\\
    &\begin{aligned}
    &=\frac{F(\wb^{(0,0)})-F^*}{MK^2}+\frac{103\log^2{(MK^2)}\kappa^2(\tau-1)\nu^2}{\mu\tau N^2\kbar^2K^2}
    +\frac{\log^2{(MK^2)}\kappa^2(\kbar-1) (\tau-1)\alpha^2}{2\mu\tau N^2\kbar^3K^2} \nn \\
    & \quad +\frac{196 \log^2{(MK^2)}\kappa^2(\kbar-1) \alpha^2}{\mu\kbar K^2}+\frac{41\log{(MK^2)}\kappa\sigma^2}{25\mu\tau N K}+\frac{8\log{(MK^2)}\kappa\gamma^2}{5\mu N K}\left(\frac{M/\kbar-N}{M/\kbar-1}\right)
    \end{aligned}\\
    &\begin{aligned}
    &=\frac{F(\wb^{(0,0)})-F^*}{MK^2}+\bigto{\frac{\kappa^2(\tau-1)\nu^2}{\mu\tau N^2\kbar^2K^2}}+\bigto{\frac{\kappa^2 (\tau-1)(\kbar-1)\alpha^2}{\mu\tau N^2\kbar^3K^2}}+\bigto{\frac{\kappa^2(\kbar-1) \alpha^2}{\mu\kbar K^2}}+\bigto{\frac{\kappa\sigma^2}{\mu\tau N K}}\\
    & \quad + \bigto{\frac{\kappa\gamma^2}{\mu N K}\left(\frac{M/\kbar-N}{M/\kbar-1}\right)}
    \end{aligned}
\end{flalign}

\subsection{Proof for \Cref{theo:2-1:kbar}}
We reiterate \Cref{eq:totalc:lsgd}, the total cost to achieve an $\epsilon$ error for the local SGD case in \Cref{algo1} as 
\begin{align}
    C_{\text{SGD}} = \frac{c_{\text{SGD}}\kbar}{\mu \epsilon N} \left[ \gamma^2 \left( \frac{M/\kbar-N}{M/\kbar-1} \right) + \frac{\sigma^2}{\tau} \right].
\end{align}
To get $C_{\text{SGD}|\kbar>1}<C_{\text{SGD}|\kbar=1}$, we need
\begin{align}
   \frac{c_{\text{SGD}}\kbar}{\mu \epsilon N}\left[ \gamma^2 \left( \frac{M/\kbar-N}{M/\kbar-1} \right) + \frac{\sigma^2}{\tau} \right] & <  \frac{c_{\text{SGD}}}{\mu \epsilon N}\left[ \gamma^2 \left( \frac{M-N}{M-1} \right) + \frac{\sigma^2}{\tau} \right] \nn \\
    \Rightarrow (\kbar - 1) \frac{\sigma^2}{\tau} & < \gamma^2 \lp \frac{M-N}{M-1} - \kbar \frac{M-N \kbar}{M-\kbar} \rp \nn \\
    &= \gamma^2 \frac{M (\kbar -1) \lp N \kbar - M + N - \frac{N \kbar}{M} \rp}{(M-1) (M-\kbar)}. \nn
\end{align}
The right-hand side is positive for $\kbar = M/N$. However, even for $\kbar = M/N-1$, on the r.h.s. we notice that
\begin{align*}
    N \kbar - M + N - \frac{N \kbar}{M} &= M - N - M + N - \frac{N}{M} \lp \frac{M}{N} - 1 \rp < 0.  
\end{align*}
Hence, except for $\kbar = M/N$, we can see that
for $\gamma\geq A_1\sigma^2/\tau$ to achieve $C_{\text{SGD}|\kbar>1} > C_{\text{SGD}|\kbar=1}$ we need to have
\begin{align}
\left(\frac{M}{N} - 1\right) \frac{\sigma^2}{\tau} & < \frac{A_1\sigma^2}{\tau} \lp \frac{M-N}{M-1}\rp \\
\Rightarrow \frac{M-N}{N}< A_1\lp \frac{M-N}{M-1}\rp\\
\therefore \frac{M-1}{N}< A_1
\end{align}
completing the proof.

\comment{
\begin{align}
C_{\text{LocalSGD}}=\frac{(C_c+C_e)\kbar\gamma^2}{\epsilon N}\left(\frac{M/\kbar-N}{M/\kbar-1}\right) +\frac{(C_c+C_e)\sigma^2}{\epsilon N\tau}\left(\kbar-1\right) 
\end{align}
To guarantee we have $C_{\text{LocalSGD}|\kbar>1}<C_{\text{LocalSGD}|\kbar=1}$, assuming $\gamma^2\geq A_1\sigma^2/\tau$, we need to satisfy the following condition:
\begin{align}
\frac{(C_c+C_e)\sigma^2}{\epsilon N\tau}\left[A_1\left(\frac{M-N\kbar}{M/\kbar-1}-\frac{M-N}{M-1}\right)+\left(\kbar-1\right)\right] <0\\
A_1\left(\frac{M-N\kbar}{M/\kbar-1}-\frac{M-N}{M-1}\right)+\left(\kbar-1\right)<0\\
(1-M)(A_1N+1)\kbar^2+(A_1M^2-A_1N+M^2-1)\kbar+M(1-M-A_1M+NA_1)<0 \label{eq:10-0-0}
\end{align}
Since 
\begin{align}
(A_1M^2-A_1N+M^2-1)^2-4M(1-M)(A_1N+1)(1-M-A_1M+NA_1)\\
=\left[(M(M-2N)+N)A_1+(M-1)^2\right]^2\geq 0
\end{align}
We have that for $(M^2-2MN+N)A_1+(M-1)^2\geq 0$ which is equivalent to $M>2N$ we have that \Cref{eq:10-0-0} holds for $\kbar<1$ and $\kbar>M(A_1(M-N)+M-1)/(M-1)(AN+1)$, and $1<M(A_1(M-N)+M-1)/(M-1)(AN+1)\leq M/N$ for $M/N \leq A_1$. Hence we prove that for a large number of total clients $M>2N$, and large enough $A_1\geq M/N$, a $M/N\geq\kbar>A_2=M(A_1(M-N)+M-1)/(M-1)(AN+1)>1$ guarantees that $C_{\text{SGD}|\kbar=1}>C_{\text{SGD}|\kbar>A_2}$ completing the proof.
}
\subsection{Proofs on Intermediate Lemmas}
\label{sec:proofs_int_results_localSGD}

\begin{proof}[Proof of \cref{lem2-4}]

\begin{flalign}
    &\nbr{\sum_{i=1}^\kbar\expt_k \left[\sum_{m\in\st^{(k,i)}}
    \db_m^{(k,i)}\right]}\leq\frac{N\kbar}{M}\sum_{i=1}^\kbar\sum_{m\in\sigma(i)}\nbr{\expt_k[\db_m^{(k,i)}]} \nn \\
    &=\frac{N\kbar}{M}\sum_{i=1}^\kbar\sum_{m\in\sigma(i)}\nbr{\expt_k \left[\sum_{l=0}^{\tau-1} (\nabla F_m(\wb_m^{(k,i-1,l)},\xi_m^{(k,i-1,l)})-\nabla F_m(\wb^{(k,i-1)}))\right]} \nn \\
    & \leq \frac{N\kbar}{M}\sum_{i=1}^\kbar\sum_{m\in\sigma(i)}\sum_{l=0}^{\tau-1}\nbr{\expt_k \left[ \nabla F_m(\wb_m^{(k,i-1,l)})-\nabla F_m(\wb^{(k,i-1)})\right]}.
\end{flalign}
Next, using \cref{as1},
\begin{flalign}
    &\begin{aligned}
    &\sum_{l=0}^{\tau-1}\nbr{\expt_k \left[\nabla F_m(\wb_m^{(k,i-1,l)})-\nabla F_m(\wb^{(k,i-1)})\right]}\leq L\sum_{l=0}^{\tau-1}\nbr{\expt_k \left[\wb_m^{(k,i-1,l)}-\wb^{(k,i-1)}\right]}
    \end{aligned} \nn \\
    &= L\lr\sum_{l=0}^{\tau-1}\nbr{\expt_k \left[\sum_{l'=0}^{l-1}\nabla F_m(\wb_m^{(k,i-1,l')},\xi_m^{(k,i-1,l')})\right]} \nn \\
    &\begin{aligned}\leq
    L\lr\sum_{l=0}^{\tau-1}\sum_{l'=0}^{l-1}\nbr{\expt_k \left[\nabla F_m(\wb_m^{(k,i-1,l')})-\nabla F_m(\wb^{(k,i-1)})+\nabla F_m(\wb^{(k,i-1)})\right]}
    \end{aligned} \nn \\
    &\begin{aligned}\leq
    L\lr\tau\sum_{l'=0}^{\tau-1}\nbr{\expt_k \left[\nabla F_m(\wb_m^{(k,i-1,l')})-\nabla F_m(\wb^{(k,i-1)})\right]}+L\lr\sum_{l=0}^{\tau-1}\sum_{l'=0}^{l-1}\nbr{\expt_k \left[\nabla F_m(\wb^{(k,i-1)})\right]}
    \end{aligned} \nn \\
    \Rightarrow & (1-L\lr\tau)\sum_{l=0}^{\tau-1}\nbr{\expt_k \left[\nabla F_m(\wb_m^{(k,i-1,l)})-\nabla F_m(\wb^{(k,i-1)})\right]} \nn \\
    &\leq \frac{L\lr\tau(\tau-1)}{2}\nbr{\expt_k \left[\nabla F_m(\wb^{(k,i-1)})-\nabla F_m(\wb^{(k,0)})+\nabla F_m(\wb^{(k,0)})\right]} \nn \\
    &\leq \frac{L^2\lr\tau(\tau-1)}{2}\nbr{\expt_k \left[\wb^{(k,i-1)}-\wb^{(k,0)}\right]}+\frac{L\lr\tau(\tau-1)}{2}\nbr{\nabla F_m(\wb^{(k,0)})}. \tag{Using \Cref{as1}}
\end{flalign}
Using $\lr$ such that $\lr \leq \frac{1}{10 L \tau}$, we get
\begin{align}
    & \sum_{l=0}^{\tau-1}\nbr{\expt_k \left[\nabla F_m(\wb_m^{(k,i-1,l)})-\nabla F_m(\wb^{(k,i-1)})\right]} \leq \frac{5 L \lr\tau(\tau-1)}{9} \left[ L \nbr{\expt_k [\wb^{(k,i-1)}]-\wb^{(k,0)}} + \nbr{\nabla F_m(\wb^{(k,0)})} \right]. \label{eq:C1-1}
\end{align}
\paragraph{Bounding $\nbr{\expt_k [\wb^{(k,i-1)}]-\wb^{(k,0)}}$.}

Analogous to \Cref{eq:updatelocsgd} we have that 
\begin{align}
\nbr{\wb^{(k,r)}-\wb^{(k,0)}}=\nbr{-\frac{\lrt}{\tau}\sum_{i=1}^r\dbl^{(k,i)}-\lrt\sum_{i=1}^r\qbl^{(k,i)}+\lrt^2\rbt^{(k,r)}}, \nn
\end{align}
where $\rbt^{(k,r)}\defeq\sum_{i=1}^{r-1}\left(\prod_{j=i+2}^{r} (\ibd-\lrt\ssbl^{(k,j)})\right)\ssbl^{(k,i+1)}\sum_{j'=1}^i \qbt^{(k,j')}$.
Hence we have
\begin{align}
    \nbr{\expt_k \left[\wb^{(k,r)}-\wb^{(k,0)}\right]}\leq \frac{\lrt}{\tau}\nbr{\sum_{i=1}^r\expt_k \left[\dbl^{(k,i)}\right]}+\lrt\nbr{\sum_{i=1}^r\expt_k \left[\qbl^{(k,i)}\right]}+\lrt^2\nbr{\expt_k \left[\rbt^{(k,r)}\right]} \label{eq:7-0-0}
\end{align}
First, we bound the first term in the RHS of \Cref{eq:7-0-0} as follows:
\begin{flalign}
&\begin{aligned}
 \frac{\lrt}{\tau}\nbr{\sum_{i=1}^r\expt_k \left[\dbl^{(k,i)}\right]}=\frac{\lrt}{\tau}\nbr{\sum_{i=1}^r\expt_k \left[\sum_{m\in\mathcal{S}^{(k,i)}}\sum_{l=0}^{\tau-1} (\nabla F_m(\wb_m^{(k,i-1,l)},\xi_m^{(k,i-1,l)})-\nabla F_m(\wb^{(k,i-1)}))\right]}
 \end{aligned} \nn \\
&\begin{aligned}
=
\frac{\lrt}{\tau}\nbr{\sum_{i=1}^r\expt_k \left[\sum_{m\in\mathcal{S}^{(k,i)}}\sum_{l=0}^{\tau-1}\nabla F_m(\wb_m^{(k,i-1,l)})-\nabla F_m(\wb^{(k,i-1)}))\right]}
 \end{aligned} \nn \\
&\leq\frac{\lrt N\kbar}{\tau M}\sum_{i=1}^r\sum_{m\in\sigma(i)}\sum_{l=0}^{\tau-1}\nbr{\expt_k \left[\nabla F_m(\wb_m^{(k,i-1,l)})-\nabla F_m(\wb^{(k,i-1)}))\right]} \nn \\
&\leq\frac{\lrt N\kbar}{\tau M}\sum_{i=1}^r\sum_{m\in\sigma(i)}\left(\frac{5\lr L(\tau-1)\tau}{9}\nbr{\nabla F_m(\wb^{(k,0)})}+\frac{5\lr L^2(\tau-1)\tau }{9}\nbr{\expt_k \left[\wb^{(k,i-1)}-\wb^{(k,0)}\right]}\right) \tag{Using \eqref{eq:C1-1}} \\
&=\frac{5\lrt \lr LN\kbar(\tau-1)}{9M}\sum_{i=1}^r \sum_{m\in\sigma(i)}\nbr{\nabla F_m(\wb^{(k,0)})}+\frac{5\lrt\lr N L^2(\tau-1)}{9}\sum_{i=1}^r\nbr{\expt_k \left[\wb^{(k,i-1)}-\wb^{(k,0)}\right]}. \label{eq:7-0-1}
\end{flalign}
We can bound the second term in the RHS of \Cref{eq:7-0-0} using \Cref{lem0-0} to get
\begin{align}
\lrt\nbr{\sum_{i=1}^r\expt_k \left[\qbl^{(k,i)}\right]}\leq r\lrt N\nbr{\nabla F(\wb^{(k,0)})}+{\alpha \lrt N r}. \label{eq:7-0-2}
\end{align}
Third, we can bound the third term in the RHS of \Cref{eq:7-0-0} as following:
\begin{flalign}
&\begin{aligned}
\lrt^2\nbr{\expt_k \left[\rbt^{(k,r)}\right]} = \lrt^2\nbr{\expt_k \left[\sum_{i=1}^{r-1}\left(\prod_{j=i+2}^{r} (\ibd-\lrt\ssbl^{(k,j)})\right)\ssbl^{(k,i+1)}\sum_{j'=1}^i \qbt^{(k,j')}\right]}
\end{aligned} \nn \\
&\begin{aligned}
\leq \lrt^2\sum_{i=1}^{r-1} \nbr{\expt_k \left[\prod_{j=i+2}^{r} (\ibd-\lrt\ssbl^{(k,j)})\right]} \nbr{\expt_k \left[\ssbl^{(k,i+1)}\right]} \nbr{\expt_k \left[\sum_{j'=1}^i \qbt^{(k,j')}\right]}~~~~\text{(}\because\text{Submultiplicativity of Norms)}  
\end{aligned} \nn \\
&\leq\frac{6\lrt^2NL}{5}\sum_{i=1}^{r-1}\nbr{\sum_{j'=1}^i \expt_k \left[ \dbl^{(k,j')}/\tau+\qbl^{(k,j')} \right]} \tag{Using \eqref{eq:3-0-4}, \eqref{eq:3-0-4-1} and $\lrt \leq \frac{1}{7 L N \kbar}$}  \nn \\
& \leq \frac{6\lrt^2 NL}{5}\sum_{i=1}^{r-1}\left(\frac{5 \lr LN\kbar(\tau-1)}{9M}\sum_{j'=1}^i\sum_{m\in\sigma(j')}\nbr{\nabla F_m(\wb^{(k,0)})}+\frac{5\lr N L^2(\tau-1)}{9}\sum_{j'=1}^i\nbr{\expt_k \left[\wb^{(k,j'-1)}-\wb^{(k,0)}\right]}\right) \nn 
\\
& \quad + \frac{6\lrt^2NL}{5}\sum_{i=1}^{r-1}\left(iN\nbr{\nabla F(\wb^{(k,0)})}+{i\alpha N}\right) \tag{Using \eqref{eq:7-0-1}, \eqref{eq:7-0-2}} \\
& \leq \frac{2\lrt^2\lr N^2L^2\kbar(\tau-1)(r-1)}{3M}\sum_{j'=1}^{r}\sum_{m\in\sigma(j')}\nbr{\nabla F_m(\wb^{(k,0)})}+\frac{2\lrt^2\lr N^2 L^3(\tau-1)(r-1)}{3}\sum_{j'=1}^{r-1}\nbr{\expt_k \left[\wb^{(k,j'-1)}-\wb^{(k,0)}\right]} \nn \\ & \quad + \frac{3\lrt^2N^2Lr(r-1)}{5}\nbr{\nabla F(\wb^{(k,0)})}+\frac{3\lrt^2N^2Lr(r-1)\alpha}{5}. \label{eq:7-0-3}
\end{flalign}
Plugging in \Cref{eq:7-0-1}, \Cref{eq:7-0-2}, and \Cref{eq:7-0-3} into \Cref{eq:7-0-0}, and summing over $r$, we have
\begin{flalign}
    &\nbr{\expt_k \left[\wb^{(k,r)}-\wb^{(k,0)}\right]} \nn \\
    & \leq \left(\frac{5\lrt \lr LN\kbar(\tau-1)}{9M}+\frac{2\lrt^2\lr N^2L^2\kbar(\tau-1)(r-1)}{3M}\right)\sum_{i=1}^r\sum_{m\in\sigma(i)}\nbr{\nabla F_m(\wb^{(k,0)})} \nn \\
    & \quad +\left(\frac{5\lrt\lr N L^2(\tau-1)}{9}+\frac{2\lrt^2\lr N^2 L^3(\tau-1)(r-1)}{3}\right)\sum_{i=1}^r\nbr{\expt_k \left[\wb^{(k,i-1)}-\wb^{(k,0)}\right]} \nn \\
    & \quad +\left(r\lrt N+\frac{3\lrt^2N^2Lr(r-1)}{5}\right)\nbr{\nabla F(\wb^{(k,0)})}+{\alpha \lrt N r} +\frac{3\lrt^2N^2Lr(r-1)\alpha}{5} \nn \\
    \therefore & \sum_{r=1}^\kbar\nbr{\expt_k \left[\wb^{(k,r-1)}-\wb^{(k,0)}\right]} \nn \\
    & \leq \left(\frac{5\lrt \lr LN\kbar(\kbar-1)(\tau-1)}{9M}+\frac{\lrt^2\lr N^2L^2\kbar^2(\tau-1)(\kbar-1)}{3M}\right)\sum_{i=1}^\kbar\sum_{m\in\sigma(i)}\nbr{\nabla F_m(\wb^{(k,0)})} \nn \\
    & \quad +\left(\frac{5\lrt\lr N L^2\kbar(\tau-1)}{9}+\frac{\lrt^2\lr N^2 L^3\kbar(\tau-1)(\kbar-1)}{3}\right)\sum_{i=1}^\kbar\nbr{\expt_k \left[\wb^{(k,i-1)}-\wb^{(k,0)}\right]} \nn \\
    & \quad +\left(\frac{\kbar(\kbar-1)\lrt N}{2}+\frac{\lrt^2N^2L\kbar^2(\kbar-1)}{5}\right)\nbr{\nabla F(\wb^{(k,0)})}+\frac{49\kbar(\kbar-1)\lrt N \alpha}{100} \nn \\
    & \leq \frac{31\lrt N \kbar(\kbar-1)}{500M}\sum_{i=1}^\kbar\sum_{m\in\sigma(i)}\nbr{\nabla F_m(\wb^{(k,0)})} + \frac{27\lrt N \kbar^2}{50}\nbr{\nabla F(\wb^{(k,0)})} \nn \\
    & \quad +\frac{\kbar(\kbar-1)\lrt N \alpha}{2}~~~~\left(\because \frac{5\lrt\lr N L^2\kbar(\tau-1)}{9}+\frac{\lrt^2\lr N^2 L^3\kbar(\tau-1)(\kbar-1)}{3}\leq 3/500\right) \nn \\
    & \leq \frac{31\lrt N \kbar^2}{50}\nbr{\nabla F(\wb^{(k,0)})}+\frac{31\lrt N \kbar(\kbar-1)\nu}{500}+\frac{\kbar(\kbar-1)\lrt N \alpha}{2}. ~~~~~~~~{(\because~\lr \leq 1/(10 \tau N L \kbar))} \label{eq:94}
\end{flalign}
Finally, we therefore have
\begin{flalign}
    &\nbr{\sum_{i=1}^\kbar\expt_k \left[\sum_{m\in\st^{(k,i)}}
    \db_m^{(k,i)}\right]} \nn \\
    &\leq\frac{N\kbar}{M}\sum_{i=1}^\kbar\sum_{m\in\sigma(i)}\left(\frac{5\lr L(\tau-1)\tau}{9}\nbr{\nabla F_m(\wb^{(k,0)})}+\frac{5\lr L^2(\tau-1)\tau }{9}\nbr{\expt_k \left[\wb^{(k,i-1)}-\wb^{(k,0)}\right]}\right) \tag{Using \eqref{eq:C1-1}} \\
    & \leq\frac{5\lr LN\kbar(\tau-1)\tau}{9}\nbr{\nabla F(\wb^{(k,0)})}+\frac{5\lr LN\kbar(\tau-1)\tau\nu}{9} \nn \\
    & \quad +\frac{5\lr N L^2(\tau-1)\tau }{9}\left(\frac{31\lrt N \kbar^2}{50}\nbr{\nabla F(\wb^{(k,0)})}+\frac{31\lrt N \kbar(\kbar-1)\nu}{500}+\frac{\kbar(\kbar-1)\lrt N \alpha}{2}\right)
    \tag{Using \eqref{eq:94}} \\
    & \begin{aligned}
    \leq\frac{3\lr LN\kbar(\tau-1)\tau}{5}\nbr{\nabla F(\wb^{(k,0)})}+\frac{14\lr LN\kbar(\tau-1)\tau\nu}{25}+\frac{\lr L N (\kbar-1) (\tau-1)\tau \alpha}{36}.
    \end{aligned} \nn
\end{flalign}
    
\end{proof}

\begin{proof}[Proof of \cref{lem2-5}]
From \eqref{eq:7-0-3} it follows that
\begin{flalign}
    & \lrt^2 \nbr{\expt_k[\rbt^{(k,0)}]}
    \leq \frac{2\lrt^2\lr N^2L^2\kbar(\kbar-1)(\tau-1)}{3M}\sum_{j'=1}^{\kbar}\sum_{m\in\sigma(j')}\nbr{\nabla F_m(\wb^{(k,0)})}+\frac{3\lrt^2N^2L\kbar(\kbar-1)}{5}\nbr{\nabla F(\wb^{(k,0)})} \nn \\
    & \quad + \frac{2\lrt^2\lr N^2 L^3(\tau-1)(\kbar-1)}{3}\sum_{j'=1}^{\kbar-1}\nbr{\expt_k \left[\wb^{(k,j'-1)}-\wb^{(k,0)}\right]} + \frac{3\lrt^2N^2L\kbar(\kbar-1)\alpha}{5} \nn \\
    & \leq\frac{2\lrt^2\lr N^2L^2\kbar(\kbar-1)(\tau-1)}{3M}\sum_{j'=1}^{\kbar}\sum_{m\in\sigma(j')}\nbr{\nabla F_m(\wb^{(k,0)})}+\frac{3\lrt^2N^2L\kbar(\kbar-1)}{5}\nbr{\nabla F(\wb^{(k,0)})} \nn \\
    & \quad + \frac{2\lrt^2\lr N^2 L^3(\tau-1)(\kbar-1)}{3}\left(\frac{31\lrt N \kbar^2}{50}\nbr{\nabla F(\wb^{(k,0)})}+\frac{31\lrt N \kbar^2\nu}{500}+\frac{\kbar(\kbar-1)\lrt N \alpha}{2}\right)
    \tag{Using \eqref{eq:94}} \\
    & \quad +\frac{3\lrt^2N^2L\kbar(\kbar-1)\alpha}{5} \nn \\
    & \leq \frac{2\lrt^2\lr N^2L^2\kbar(\kbar-1)(\tau-1)}{3} \left[ \nbr{\nabla F(\wb^{(k,0)})} + \nu \right] + \frac{3\lrt^2N^2L\kbar(\kbar-1)}{5}\nbr{\nabla F(\wb^{(k,0)})} \tag{Recall $\nu = \gamma + \alpha$, \cref{as3}} \\
    & \quad +\frac{2\lrt^2\lr N^2 L^3(\tau-1)(\kbar-1)}{3}\left(\frac{31\lrt N \kbar^2}{50}\nbr{\nabla F(\wb^{(k,0)})} + \frac{31\lrt N \kbar^2\nu}{500} + \frac{\kbar(\kbar-1)\lrt N \alpha}{2} \right) \nn \\
    & \quad + \frac{3\lrt^2N^2L\kbar(\kbar-1)\alpha}{5} \nn \\
    & \leq \lrt N (\kbar-1) \left(\frac{2\lrt^2 NL^2\kbar}{3}+\frac{3\lrt NL\kbar}{5}+\frac{31\lrt^3N^2L^3\kbar^2}{75}\right)\nbr{\nabla F(\wb^{(k,0)})} \nn \\
    & \quad + \frac{2\lrt^2\lr N^2L^2\kbar(\kbar-1)(\tau-1)}{3} \left( 1 + \frac{31\lrt NL\kbar}{500} \right) \nu + \lrt^2N^2L\kbar(\kbar-1) \left(\frac{\lrt^2NL^2\kbar}{3}+\frac{3}{5}\right) \alpha \nn \\
    & \leq \frac{17\lrt N\kbar}{250}\nbr{\nabla F(\wb^{(k,0)})}+\frac{84\lrt^2\lr N^2L^2\kbar(\kbar-1)(\tau-1)}{125} \nu + \frac{16\lrt^2N^2L\kbar(\kbar-1)}{25} \alpha. \nn
\end{flalign}
where the last two bounds are due to $\lr \leq 1/(10 \tau N L \kbar))$.
\end{proof}

\begin{proof}[Proof of \cref{lem2-3}]
We have
\begin{align}
    & \frac{L}{2} \expt_k \left[\nbr{\wb^{(k,r)}-\wb^{(k,0)}}^2\right] \nn \\
    &= \frac{L}{2}\expt_k \left[\nbr{
    -\lrt \sum_{i=1}^r \dbl^{(k,i)}/\tau-\lrt \sum_{i=1}^r\qbl^{(k,i)}+\lrt^2\underbrace{\left(\sum_{i=1}^{r-1}\left(\prod_{j=i+2}^{r} (\ibd-\lrt\ssbl^{(k,j)})\right)\ssbl^{(k,i+1)}\sum_{j'=1}^i \qbt^{(k,j')}\right)}_{\defeq \rbt^{(k,r)}}}^2\right] \nn \\
    & \leq \frac{3L\lrt^2}{2\tau^2}\expt_k \left[\nbr{\sum_{i=1}^r \dbl^{(k,i)}}^2\right]+
     \frac{3L\lrt^2}{2}\expt_k \left[\nbr{\sum_{i=1}^r \qbl^{(k,i)}}^2\right]+\frac{3L\lrt^4}{2}\expt_k \left[\nbr{\rbt^{(k,r)}}^2\right]. \label{eq:5-1-1}
\end{align}
where in \Cref{eq:5-1-1} we use $\|a+b+c\|^2\leq3(\|a\|^2+\|b\|^2+\|c\|^2)$. We first bound the last term in \Cref{eq:5-1-1} as the following:
\begin{align}
    \frac{3L\lrt^4}{2}\expt_k \left[\nbr{\rbt^{(k,r)}}^2\right] &= \frac{3L\lrt^4}{2} \expt_k \left[\nbr{\sum_{i=1}^{r-1}\left(\prod_{j=i+2}^{r} (\ibd-\lrt\ssbl^{(k,j)})\right)\ssbl^{(k,i+1)}\sum_{j'=1}^i \qbt^{(k,j')}}^2\right] \nn \\
    &\leq \frac{54N^2L^3\lrt^4(r-1)}{25}\sum_{i=1}^{r-1}\expt_k \left[\frac{2}{\tau^2}\nbr{\sum_{j'=1}^i\dbl^{(k,j')}}^2+2\nbr{\sum_{j'=1}^i\qbl^{(k,j')}}^2\right], \label{eq:5-1-2}
\end{align}
which follows from \eqref{eq:3-0-4}, \eqref{eq:3-0-4-1}. Note that the two terms in \Cref{eq:5-1-2} are in similar forms as in the first two terms in \Cref{eq:5-1-1}. We will come back to bounding \Cref{eq:5-1-2} after bounding the first two terms in \Cref{eq:5-1-1}. 
\begin{flalign}
    &\expt_k \left[\nbr{\sum_{i=1}^r \dbl^{(k,i)}}^2\right]
    =\expt_k \left[\nbr{\sum_{i=1}^r \sum_{m\in\mathcal{S}^{(k,i)}}\db_m^{(k,i)}}^2\right]
    \leq r\sum_{i=1}^r\expt_k \left[\nbr{\sum_{m\in\mathcal{S}^{(k,i)}}\db_m^{(k,i)}}^2\right] \nn \tag{Using \Cref{lem:sum_of_squares}}\\
    & = r\sum_{i=1}^r \expt_k \nbr{\sum_{m\in\mathcal{S}^{(k,i)}} \sum_{l=0}^{\tau-1} \lp \nabla F_m(\wb_m^{(k,i-1,l)},\xi_m^{(k,i-1,l)}) - \nabla F_m(\wb_m^{(k,i-1,l)}) + \nabla F_m(\wb_m^{(k,i-1,l)}) -\nabla F_m(\wb^{(k,i-1)}) \rp}^2 \nn \\
    &\leq \frac{rN^2\tau\kbar}{M}\sum_{i=1}^r\sum_{m\in\sigma(i)}\left(\frac{103\lr^2L^2(\tau-1)\tau\sigma^2}{200}+\frac{69\lr^2 L^4 \tau^2 (\tau-1)}{50}\expt_k \left[\nbr{\wb^{(k,i-1)}-\wb^{(k,0)}}^2\right]\right. \nn \\
    & \qquad \qquad \left. + \frac{69\lr^2L^2\tau^2(\tau-1)}{50}\expt_k \left[\nbr{\nabla F_m(\wb^{(k,0)})}^2\right]\right)+r^2N\tau\sigma^2 \tag{Using the bound in \eqref{eq:5-1-0}} \\
    &= \frac{69rN^2\tau^3(\tau-1)\kbar \lr^2 L^2}{50M}\sum_{i=1}^r\sum_{m\in\sigma(i)}\nbr{\nabla F_m(\wb^{(k,0)})}^2+\frac{69rN^2\tau^3(\tau-1)\kbar\lr^2 L^4}{50M}\sum_{i=1}^r\sum_{m\in\sigma(i)}\expt_k \left[\nbr{\wb^{(k,i-1)}-\wb^{(k,0)}}^2\right] \nn \\
    & \quad + r^2 N \tau \sigma^2\left(1+\frac{103N\tau(\tau-1)L^2\lr^2}{200}\right)
    \label{eq:5-1-2-0}
\end{flalign}
For the second term in \Cref{eq:5-1-1}, similar to how we got \Cref{eq:3-2-2}, we have that
\begin{flalign}
    &\expt_k \left[\nbr{\sum_{i=1}^r \qbl^{(k,i)}}^2\right]\leq \frac{1}{3} N r^2 \left[ 2 N \lp \alpha^2 + \nbr{\nabla F(\wb^{(k,0)})}^2 \rp + \gamma^2 \left( \frac{M/\kbar-N}{M/\kbar-1} \right) \right]. \label{eq:5-1-2-1}
\end{flalign}
We can use \Cref{eq:5-1-2-0} and \Cref{eq:5-1-2-1} to bound \Cref{eq:5-1-2} as following:
\begin{flalign}
    & \frac{3L\lrt^4}{2}\expt_k \left[\nbr{\rbt^{(k,r)}}^2\right] \nn \\
    & \leq \frac{54N^2L^3\lrt^4(r-1)}{25}\sum_{i=1}^{r-1}\left(\frac{69iN^2\tau(\tau-1)\kbar \lr^2 L^2}{25M}\sum_{j'=1}^i\sum_{m\in\sigma(j')}\nbr{\nabla F_m(\wb^{(k,0)})}^2\right. \nn \\
    & \quad \left.+\frac{69iN^2\tau(\tau-1)\kbar\lr^2 L^4}{25M}\sum_{j'=1}^i\sum_{m\in\sigma(j')}\expt_k \left[\nbr{\wb^{(k,j'-1)}-\wb^{(k,0)}}^2\right]+\frac{2i^2N\sigma^2}{\tau}\left(1+\frac{103N\tau(\tau-1)L^2\lr^2}{200}\right)\right. \nn \\
    & \quad \left. + 8 N^2 i^2 \nbr{\nabla F(\wb^{(k,0)})}^2 + 4 N i^2 \gamma^2\left(\frac{M/\kbar-N}{M/\kbar-1}\right)+8N^2i^2\alpha^2 \right) \nn \\
    &\leq\frac{3N^4L^5\tau(\tau-1)\kbar\lr^2\lrt^4r(r-1)^2}{M}\sum_{j'=1}^r\sum_{m\in\sigma(j')}\nbr{\nabla F_m(\wb^{(k,0)})}^2+\frac{29N^4L^3\lrt^4 (r-1)^2r^2}{5}\nbr{\nabla F(\wb^{(k,0)})}^2 \nn \\
    & \quad +\frac{3N^4L^7\tau(\tau-1)\kbar\lr^2 \lrt^4 (r-1)^2r}{M}\sum_{j'=1}^r\sum_{m\in\sigma(j')}\expt_k \left[\nbr{\wb^{(k,j'-1)}-\wb^{(k,0)}}^2\right]+\frac{72N^3L^3\lrt^4(r-1)^2r^2\gamma^2}{25}\left(\frac{M/\kbar-N}{M/\kbar-1}\right) \nn \\
    & \quad +\frac{36N^3L^3\lrt^4(r-1)^2r^2\sigma^2}{25\tau}\left(1+\frac{103N\tau(\tau-1)L^2\lr^2}{200}\right)+\frac{29N^4L^3\lrt^4 (r-1)^2r^2\alpha^2}{5}. \label{eq:5-1-2-2}
\end{flalign}
Finally, plugging in \Cref{eq:5-1-2-0}, \Cref{eq:5-1-2-1}, and \Cref{eq:5-1-2-2} to \Cref{eq:5-1-1} we have 
\begin{align}
    &\expt_k \left[\nbr{\wb^{(k,r)}-\wb^{(k,0)}}^2\right] \nn \\
    & \leq \frac{2 \lrt^2 \lr^2 L^2 N^2 \tau(\tau-1)\kbar r}{M}\left(2+3\lrt^2L^2N^2(r-1)^2\right)\sum_{j'=1}^r\sum_{m\in\sigma(j')}\nbr{\nabla F_m(\wb^{(k,0)})}^2 \nn \\
    & \quad + \frac{2\lrt^2\lr^2L^4N^2\tau(\tau-1)\kbar r}{M}\left(2+3\lrt^2L^2N^2(r-1)^2\right)\sum_{i=1}^r\sum_{m\in\sigma(i)}\expt_k \left[\nbr{\wb^{(k,i-1)}-\wb^{(k,0)}}^2\right] \nn \\
    & \quad +\frac{2\lrt^2Nr^2\sigma^2}{\tau}\left(1+\frac{103\lr^2N\tau(\tau-1)L^2}{200}\right)\left(\frac{3}{2}+\frac{36\lrt^2L^2N^2(r-1)^2}{25}\right) \nn \\
    & \quad +2\lrt^2Nr^2\gamma^2 \left(\frac{M/\kbar-N}{M/\kbar-1}\right)\left(3+\frac{72\lrt^2L^2N^2(r-1)^2}{25}\right) \nn \\
    & \quad +\lrt^2N^2r^2\nbr{\nabla F(\wb^{(k,0)})}^2\left(12+\frac{58\lrt^2L^2N^2(r-1)^2}{5}\right)+\lrt^2N^2r^2\alpha^2\left(12+\frac{58\lrt^2L^2N^2(r-1)^2}{5}\right) \nn 
\end{align}
With rearrangement of the terms, we have
\begin{align}
    \Rightarrow & \expt_k \left[\nbr{\wb^{(k,r-1)}-\wb^{(k,0)}}^2\right]\leq \frac{2\lrt^2\lr^2L^2N^2\tau(\tau-1)\kbar (r-1)}{M} \left(2 + 3 \lrt^2L^2N^2\kbar^2\right)\sum_{j'=1}^\kbar\sum_{m\in\sigma(j')}\nbr{\nabla F_m(\wb^{(k,0)})}^2 \nn \\
    & \quad + \frac{2 \lrt^2 \lr^2 L^4 N^2\tau(\tau-1)\kbar (r-1)}{M}\left(2+3\lrt^2L^2N^2\kbar^2\right)\sum_{i=1}^\kbar\sum_{m\in\sigma(i)}\expt_k \left[\nbr{\wb^{(k,i-1)}-\wb^{(k,0)}}^2\right] \nn \\
    & \quad +\frac{2\lrt^2N(r-1)^2\sigma^2}{\tau}\left(1+\frac{103\lr^2N\tau(\tau-1)L^2}{200}\right)\left(\frac{3}{2}+\frac{36\lrt^2L^2N^2\kbar^2}{25}\right) \nn \\
    & \quad + 2\lrt^2N(r-1)^2\gamma^2 \left(\frac{M/\kbar-N}{M/\kbar-1}\right)\left(3+\frac{72\lrt^2L^2N^2\kbar^2}{25}\right) \nn \\
    & \quad + \lrt^2N^2r^2\nbr{\nabla F(\wb^{(k,0)})}^2\left(12+\frac{58\lrt^2L^2N^2r^2}{5}\right)+\lrt^2N^2(r-1)^2\alpha^2\left(12+\frac{58\lrt^2L^2N^2r^2}{5}\right) \nn \\
    & \leq \frac{102\lrt^2\lr^2L^2N^2\tau(\tau-1)\kbar (r-1)}{25M}\sum_{j'=1}^\kbar\sum_{m\in\sigma(j')}\nbr{\nabla F_m(\wb^{(k,0)})}^2+\frac{61\lrt^2N^2(r-1)^2}{5}\nbr{\nabla F(\wb^{(k,0)})}^2 \nn \\
    & \quad + \frac{102\lrt^2\lr^2L^4N^2\tau(\tau-1)(r-1)}{25}\sum_{i=1}^\kbar\expt_k \left[\nbr{\wb^{(k,i-1)}-\wb^{(k,0)}}^2\right]+\frac{61\lrt^2N(r-1)^2\sigma^2}{20\tau} \nn \\
    & \quad + \frac{303\lrt^2N(r-1)^2\gamma^2}{50} \left(\frac{M/\kbar-N}{M/\kbar-1}\right)+\frac{61\lrt^2N^2(r-1)^2\alpha^2}{5} \nn  ~~~~~~~~{(\because~\lr \leq 1/(10 \tau N L \kbar))}
    \end{align}
Summing over $r=1,...,\kbar$ we have
\begin{align}
    \Rightarrow & \sum_{r=1}^\kbar \expt_k \left[\nbr{\wb^{(k,r-1)}-\wb^{(k,0)}}^2\right]\leq \frac{51\lrt^2\lr^2L^2N^2\tau(\tau-1)\kbar^2(\kbar-1}{50M}\sum_{j'=1}^\kbar\sum_{m\in\sigma(j')}\nbr{\nabla F_m(\wb^{(k,0)})}^2 \nn \\
    & \quad + \frac{61\lrt^2N^2\kbar^2(\kbar-1)}{15}\nbr{\nabla F(\wb^{(k,0)})}^2+\frac{51\lrt^2\lr^2L^4N^2\tau(\tau-1)\kbar^2}{50}\sum_{i=1}^\kbar\expt_k \left[\nbr{\wb^{(k,i-1)}-\wb^{(k,0)}}^2\right] \nn \\
    & \quad + \frac{61\lrt^2N\kbar^2(\kbar-1)\sigma^2}{60\tau}+\frac{101\lrt^2N\kbar^2(\kbar-1)\gamma^2}{50} \left(\frac{M/\kbar-N}{M/\kbar-1}\right)+\frac{61\lrt^2N^2\kbar^2(\kbar-1)\alpha^2}{15} \nn \\
    & \quad \therefore \left(1-\frac{51\lrt^2\lr^2L^4N^2\tau(\tau-1)\kbar^2}{50}\right)\sum_{r=1}^\kbar\expt_k \left[\nbr{\wb^{(k,r-1)}-\wb^{(k,0)}}^2\right] \nn \\
    & \leq \frac{51\lrt^2\lr^2L^2N^2\tau(\tau-1)\kbar^2(\kbar-1)}{50M}\sum_{j'=1}^\kbar\sum_{m\in\sigma(j')}\nbr{\nabla F_m(\wb^{(k,0)})}^2+\frac{61\lrt^2N^2\kbar^2(\kbar-1)}{15}\nbr{\nabla F(\wb^{(k,0)})}^2 \nn \\
    & \quad + \frac{61\lrt^2N\kbar^2(\kbar-1)\sigma^2}{60\tau}+\frac{101\lrt^2N\kbar^2(\kbar-1)\gamma^2}{50} \left(\frac{M/\kbar-N}{M/\kbar-1}\right)+\frac{61\lrt^2N^2\kbar^2(\kbar-1)\alpha^2}{15}
\end{align}
Again using $\lr \leq 1/(10 \tau N L \kbar))$, we have
\begin{flalign}
    & \sum_{r=1}^\kbar\expt_k \left[\nbr{\wb^{(k,r-1)}-\wb^{(k,0)}}^2\right] \nn \\
    & \leq \frac{103\lrt^2\lr^2L^2N^2\tau(\tau-1)\kbar^2(\kbar-1)}{100M}\sum_{j'=1}^\kbar\sum_{m\in\sigma(j')}\nbr{\nabla F_m(\wb^{(k,0)})}^2 \nn \\
    &  \quad + \frac{41\lrt^2N^2\kbar^2(\kbar-1)}{10}\nbr{\nabla F(\wb^{(k,0)})}^2+\frac{51\lrt^2N\kbar^2(\kbar-1)\sigma^2}{50\tau}+ \frac{203\lrt^2N\kbar^2(\kbar-1)\gamma^2}{100} \left(\frac{M/\kbar-N}{M/\kbar-1}\right) \nn \\
    & \quad +\frac{41\lrt^2N^2\kbar^2(\kbar-1)\alpha^2}{10} \nn \\
    & \leq \frac{83\lrt^2N^2\kbar^2(\kbar-1)}{20}\nbr{\nabla F(\wb^{(k,0)})}^2+\frac{51\lrt^2N\kbar^2(\kbar-1)\sigma^2}{50\tau}+\frac{3\lr^2\kbar\tau(\tau-1)\nu^2}{100}+\frac{51\lrt^2N\kbar^2(\kbar-1)\gamma^2}{25} \left(\frac{M/\kbar-N}{M/\kbar-1}\right) \nn \\
    & \quad +\frac{41\lrt^2N^2\kbar^2(\kbar-1)\alpha^2}{10}. \nn 
\end{flalign} 
\end{proof}

\begin{proof}[Proof of \cref{lem2-6}]
\begin{flalign}
    & \frac{3L\lrt^2}{2\tau^2}\expt_k \left[\nbr{\sum_{i=1}^\kbar \dbl^{(k,i)}}^2\right] = \frac{3L\lrt^2}{2\tau^2}\expt_k \left[\nbr{\sum_{i=1}^\kbar \sum_{m\in\mathcal{S}^{(k,i)}}\db_m^{(k,i)}}^2\right] \leq \frac{3L\lrt^2\kbar}{2\tau^2}\sum_{i=1}^\kbar\expt_k \left[\nbr{\sum_{m\in\mathcal{S}^{(k,i)}} \db_m^{(k,i)}}^2\right] \nn \\
    &=\frac{3L\lrt^2\kbar}{2\tau^2}\sum_{i=1}^\kbar\expt_k \left[\nbr{\sum_{m\in\mathcal{S}^{(k,i)}} \sum_{l=0}^{\tau-1} (\nabla F_m(\wb_m^{(k,i-1,l)},\xi_m^{(k,i-1,l)}) - \nabla F_m(\wb_m^{(k,i-1,l)}) + \nabla F_m(\wb_m^{(k,i-1,l)})-\nabla F_m(\wb^{(k,i-1)}))}^2\right] \nn \\
    &=\frac{3L\lrt^2\kbar}{2\tau^2}\sum_{i=1}^\kbar\expt_k \left[\nbr{\sum_{m\in\mathcal{S}^{(k,i)}} \sum_{l=0}^{\tau-1} (\nabla F_m(\wb_m^{(k,i-1,l)},\xi_m^{(k,i-1,l)}) -\nabla F_m(\wb^{(k,i-1,l)}))}^2\right] \tag{cross terms are zero} \\
    & \quad + \frac{3L\lrt^2\kbar}{2\tau^2}\sum_{i=1}^\kbar\expt_k \left[\nbr{\sum_{m\in\mathcal{S}^{(k,i)}} \sum_{l=0}^{\tau-1} (\nabla F_m(\wb_m^{(k,i-1,l)})-\nabla F_m(\wb^{(k,i-1)}))}^2\right] \nn \\
    & \leq \frac{3L\lrt^2\kbar}{2\tau^2}\sum_{i=1}^\kbar\expt_k \left[\sum_{m\in\mathcal{S}^{(k,i)}}\nbr{\sum_{l=0}^{\tau-1} (\nabla F_m(\wb_m^{(k,i-1,l)},\xi_m^{(k,i-1,l)}) -\nabla F_m(\wb^{(k,i-1,l)}))}^2\right] \tag{independence of stochastic gradients across clients} \\
    & \quad + \frac{3L\lrt^2\kbar N}{2\tau^2}\sum_{i=1}^\kbar\expt_k \left[\sum_{m\in\mathcal{S}^{(k,i)}} \nbr{\sum_{l=0}^{\tau-1} (\nabla F_m(\wb_m^{(k,i-1,l)})-\nabla F_m(\wb^{(k,i-1)}))}^2\right] \nn \\
    & \leq \frac{3L\lrt^2\kbar}{2\tau^2}\sum_{i=1}^\kbar\expt_k \left[\sum_{m\in\mathcal{S}^{(k,i)}} \sum_{l=0}^{\tau-1} \nbr{\nabla F_m(\wb_m^{(k,i-1,l)},\xi_m^{(k,i-1,l)}) -\nabla F_m(\wb^{(k,i-1,l)})}^2\right] \tag{unbiasedness of stochastic gradients} \\
    & \quad + \frac{3L\lrt^2\kbar N}{2\tau}\sum_{i=1}^\kbar\expt_k \left[\sum_{m\in\mathcal{S}^{(k,i)}} \sum_{l=0}^{\tau-1} \nbr{\nabla F_m(\wb_m^{(k,i-1,l)})-\nabla F_m(\wb^{(k,i-1)})}^2\right] \nn \\
    &\leq \frac{3L\lrt^2\kbar^2 N}{2M\tau^2}\sum_{i=1}^\kbar\sum_{m\in\sigma(i)}\sum_{l=0}^{\tau-1} \sigma^2 \nn \tag{\Cref{as5}}\\
    & \quad + \frac{3L\lrt^2\kbar^2 N^2}{2M\tau}\sum_{i=1}^\kbar\sum_{m\in\sigma(i)}\sum_{l=0}^{\tau-1} \expt_k \left[\nbr{\nabla F_m(\wb_m^{(k,i-1,l)})-\nabla F_m(\wb^{(k,i-1)})}^2\right]  \nn \\
    & \leq \frac{3L\lrt^2\kbar^2 N}{2\tau}\sigma^2 + \frac{3L\lrt^2\kbar^2 N^2}{2M\tau}\sum_{i=1}^\kbar\sum_{m\in\sigma(i)}\sum_{l=0}^{\tau-1} \expt_k \left[\nbr{\nabla F_m(\wb_m^{(k,i-1,l)})-\nabla F_m(\wb^{(k,i-1)})}^2\right] \label{eq:5-0-7-5}
\end{flalign}

From \Cref{eq:5-0-7-5} we have that
\begin{align}
    & \sum_{l=0}^{\tau-1} \expt_k \left[\nbr{\nabla F_m(\wb_m^{(k,i-1,l)})-\nabla F_m(\wb^{(k,i-1)})}^2\right] \nn \\
    & \leq L^2\sum_{l=0}^{\tau-1} \expt_k \left[\nbr{\wb_m^{(k,i-1,l)}-\wb^{(k,i-1)})}^2\right] \nn \tag{\Cref{as1}}\\ 
    & \leq \frac{103\lr^2L^2(\tau-1)\tau\sigma^2}{200}+\frac{69\lr^2L^4\tau^2(\tau-1)}{50}\expt\left[\nbr{\wb^{(k,i-1)}-\wb^{(k,0)}}^2\right]+\frac{69\lr^2L^2\tau^2(\tau-1)}{50}\expt\left[\nbr{\nabla F_m(\wb^{(k,0)})}^2\right].
    \label{eq:5-1-0}
\end{align}
Plugging in \Cref{eq:5-1-0} to \Cref{eq:5-0-7-5} we have
\begin{flalign}
    & \frac{3L\lrt^2}{2\tau^2}\expt\left[\nbr{\sum_{i=1}^\kbar \dbl^{(k,i)}}^2\right] \nn \\
    & \leq  \frac{3L\lrt^2\kbar^2 N}{2\tau}\sigma^2 + \frac{3L\lrt^2\kbar^2 N^2}{2M\tau}\sum_{i=1}^\kbar\sum_{m\in\sigma(i)}\left(\frac{103\lr^2L^2(\tau-1)\tau\sigma^2}{200} \right. \nn \\
    & \quad \left. \frac{69 \lr^2 L^4 \tau^2 (\tau-1)}{50} \expt \left[\nbr{\wb^{(k,i-1)}-\wb^{(k,0)}}^2\right] + \frac{69\lr^2L^2\tau^2(\tau-1)}{50}\expt \left[\nbr{\nabla F_m(\wb^{(k,0)})}^2\right]\right) \nn \\
    &= \frac{8L\lrt^2\kbar^2 N\sigma^2}{5\tau} + \frac{52L^5\lrt^2\lr^2\kbar N^2\tau(\tau-1)}{25}\sum_{i=1}^\kbar\expt\left[\nbr{\wb^{(k,i-1)}-\wb^{(k,0)}}^2\right] \nn \\
    & \quad +\frac{52}{25}\lrt^2\lr^2L^3\kbar^2N^2\tau(\tau-1)(2\nbr{\nabla F(\wb^{(k,0)})}^2+2\nu^2) \nn \\
    &= \frac{8L\lrt^2\kbar^2 N\sigma^2}{5\tau} + \frac{52L^5\lrt^2\lr^2\kbar N^2\tau(\tau-1)}{25}\left(\frac{83\lrt^2N^2\kbar^2(\kbar-1)}{20}\nbr{\nabla F(\wb^{(k,0)})}^2+\frac{51\lrt^2N\kbar^2(\kbar-1)\sigma^2}{50\tau}\right. \nn \\
    & \quad \left. + \frac{3\lr^2\kbar\tau(\tau-1)\nu^2}{100}+\frac{51\lrt^2N\kbar^2(\kbar-1)\gamma^2}{25} \left(\frac{M/\kbar-N}{M/\kbar-1}\right)+\frac{41\lrt^2N^2\kbar^2(\kbar-1)\alpha^2}{10}\right) \nn \\
    & \quad + \frac{52}{25} \lrt^2 \lr^2L^3\kbar^2N^2\tau(\tau-1)(2\nbr{\nabla F(\wb^{(k,0)})}^2+2\nu^2) \nn \\
    &= \lrt^2\lr^2L^3\kbar^2 N^2\tau(\tau-1)\left(\frac{83\times 52 \lrt^2 L^2 N^2 \kbar^2}{25\times20} + \frac{104}{25}\right)\nbr{\nabla F(\wb^{(k,0)})}^2 \nn \\
    & \quad + \frac{L \lrt^2 N \kbar^2 \sigma^2}{\tau}\left(\frac{8}{5}+\frac{26\times 51 L^4\lrt^2\lr^2\kbar^2N^2\tau(\tau-1)}{25^2}\right)+\lrt^2\lr^2L^3\kbar^2N^2\tau(\tau-1)\left(\frac{104}{25}+\frac{156\lr^2L^2\tau(\tau-1)}{2500}\right)\nu^2 \nn \\ 
    & \quad + \frac{52 \times 51 L^5\lrt^4\lr^2\kbar^4N^3\tau(\tau-1)\gamma^2}{25^2}\left(\frac{M/\kbar-N}{M/\kbar-1}\right)+\frac{52\times 41 L^5\lrt^4\lr^2\kbar^3(\kbar-1)N^4\tau(\tau-1)\alpha^2}{250} \nn \\
    & \quad \leq\frac{\lrt N\kbar}{200}\nbr{\nabla F(\wb^{(k,0)})}^2+\frac{81L\lrt^2\kbar^2N\sigma^2}{50\tau}+\frac{209\lrt^2\lr^2L^3\kbar^2N^2\tau(\tau-1)\nu^2}{50}+\frac{L^3\lrt^4\kbar^4N^3\gamma^2}{20}\left(\frac{M/\kbar-N}{M/\kbar-1}\right) \nn \\
    & \quad +\frac{9 L^3\lrt^4\kbar^3(\kbar-1)N^4\alpha^2}{100}~~~~(\because  \lr=\log{(MK^2)}/\mu\tau N \kbar K,~K\geq 10\kappa\log{(MK^2)}). \nn
\end{flalign}
\end{proof}

\begin{proof}[Proof of \cref{lem2-7}]
\begin{align}
    \frac{3L\lrt^4}{2}\expt\left[\nbr{\rbt^{(k,0)}}^2\right] &= \frac{3L\lrt^4}{2}\expt\left[\nbr{\sum_{i=1}^{\kbar-1}\left(\prod_{j=i+2}^{\kbar} (\ibd-\lrt\ssbl^{(k,j)})\right)\ssbl^{(k,i+1)}\sum_{j'=1}^i \qbt^{(k,j')}}^2\right] \nn \\
    &\leq \frac{3L\lrt^4\kbar}{2}\sum_{i=1}^{\kbar-1}\expt\left[\nbr{\left(\prod_{j=i+2}^{\kbar} (\ibd-\lrt\ssbl^{(k,j)})\right)\ssbl^{(k,i+1)}\sum_{j'=1}^i \qbt^{(k,j')}}^2\right] \nn \tag{Using \Cref{lem:sum_of_squares}}\\
    & \leq\frac{54\lrt^4\kbar N^2L^3}{25 }\sum_{i=1}^{\kbar-1}\expt\left[\nbr{\sum_{j'=1}^i \qbt^{(k,j')}}^2\right], \label{eq:5-0-8}
\end{align}
where \Cref{eq:5-0-8} is derived in the same was as \Cref{eq:3-2-0}. Now we bound the following in \Cref{eq:5-0-8}
\begin{align}
    \expt\left[\nbr{\sum_{j'=1}^i \qbt^{(k,j')}}^2\right] &= \expt\left[\nbr{\sum_{j'=1}^i \dbl^{(k,j')}/\tau+\qbl^{(k,j')}}^2\right]\leq 2\expt\left[\nbr{\sum_{j'=1}^i \dbl^{(k,j')}/\tau}^2\right]+2\expt\left[\nbr{\sum_{j'=1}^i\qbl^{(k,j')}}^2\right] \nn \\
    &= \frac{2}{\tau^2} \expt \left[\nbr{\sum_{j'=1}^i \sum_{m\in\mathcal{S}^{(k,j')}}\db_m^{(k,j')}}^2\right]+2\expt\left[\nbr{\sum_{j'=1}^i\qbl^{(k,j')}}^2\right]. \label{eq:5-0-8-1-1}
\end{align}
We have that
\begin{flalign}
    & \frac{2}{\tau^2}\expt\left[\nbr{\sum_{j'=1}^i \sum_{m \in \mathcal{S}^{(k,j')}} \db_m^{(k,j')}}^2 \right]\leq\frac{2i}{\tau^2}\sum_{j'=1}^i\expt\left[\nbr{\sum_{m\in\mathcal{S}^{(k,j')}}\db_m^{(k,j')}}^2\right] \nn \\
    &= \frac{2i}{\tau^2}\sum_{j'=1}^i\expt\left[\nbr{\sum_{m\in\mathcal{S}^{(k,j')}}\sum_{l=0}^{\tau-1} (\nabla F_m(\wb_m^{(k,{j'}-1,l)},\xi_m^{(k,{j'}-1,l)})-\nabla F_m(\wb_m^{(k,{j'}-1,l)})+\nabla F_m(\wb_m^{(k,{j'}-1,l)})-\nabla F_m(\wb^{(k,{j'}-1)}))}^2\right] \nn \\
    &= \frac{2i}{\tau^2}\sum_{j'=1}^i\expt\left[\nbr{\sum_{m\in\mathcal{S}^{(k,j')}}\sum_{l=0}^{\tau-1} \nabla F_m(\wb_m^{(k,{j'}-1,l)},\xi_m^{(k,{j'}-1,l)})-\nabla F_m(\wb_m^{(k,{j'}-1,l)})}^2\right] \nn \\
    & \quad +\frac{2i}{\tau^2}\sum_{j'=1}^i\expt\left[\nbr{\sum_{m\in\mathcal{S}^{(k,j')}}\sum_{l=0}^{\tau-1}\nabla F_m(\wb_m^{(k,{j'}-1,l)})-\nabla F_m(\wb^{(k,{j'}-1)}))}^2\right] \nn \tag{Cross-terms are zero} \\
    & \leq \frac{2i}{\tau^2}\sum_{j'=1}^i\expt\left[\sum_{m\in\mathcal{S}^{(k,j')}}\sum_{l=0}^{\tau-1}\nbr{\nabla F_m(\wb_m^{(k,{j'}-1,l)},\xi_m^{(k,{j'}-1,l)})-\nabla F_m(\wb_m^{(k,{j'}-1,l)})}^2\right] \nn \\
    & \quad + \frac{2iN}{\tau} \sum_{j'=1}^i \expt \left[\sum_{m\in\mathcal{S}^{(k,j')}}\sum_{l=0}^{\tau-1}\nbr{\nabla F_m(\wb_m^{(k,{j'}-1,l)})-\nabla F_m(\wb^{(k,{j'}-1)}))}^2\right] \nn \tag{Using \Cref{lem:sum_of_squares}} \\
    & \leq \frac{2 i N \kbar}{M \tau^2} \sum_{j'=1}^i \sum_{m\in\sigma(j')} \sum_{l=0}^{\tau-1} \sigma^2 \nn \tag{Using \Cref{as5}}
    \\ & \quad + \frac{2 i N^2 \kbar}{M \tau } \sum_{j'=1}^i \sum_{m\in\sigma(j')} \sum_{l=0}^{\tau-1} \expt \left[\nbr{\nabla F_m(\wb_m^{(k,{j'}-1,l)})-\nabla F_m(\wb^{(k,{j'}-1)}))}^2\right] \nn \\
    & \leq \frac{2i^2N\sigma^2}{\tau} + \frac{2iN^2\kbar}{M\tau }\sum_{j'=1}^i\sum_{m\in\sigma(j')}\left(\frac{103\lr^2L^2(\tau-1)\tau\sigma^2}{200} + \frac{69\lr^2L^4\tau^2(\tau-1)}{50}\expt\left[\nbr{\wb^{(k,j'-1)}-\wb^{(k,0)}}^2\right]\right. \nn \\
    & \quad \left. + \frac{69\lr^2L^2\tau^2(\tau-1)}{50}\expt\left[\nbr{\nabla F_m(\wb^{(k,0)})}^2\right]\right)  \nn \tag{Using \Cref{eq:5-1-0}} \\
    & = \frac{2i^2N\sigma^2}{\tau} + \frac{103i^2N^2\lr^2L^2(\tau-1)\sigma^2}{100} + \frac{69 i N^2 \lr^2 L^4 \tau(\tau-1)}{25} \sum_{j'=1}^i\expt\left[\nbr{\wb^{(k,j'-1)}-\wb^{(k,0)}}^2\right] \nn \\
    & \quad + \frac{69 i N^2 \kbar \lr^2 L^2 \tau (\tau-1)}{25M} \sum_{j'=1}^i\sum_{m\in\sigma(j')}\expt\left[\nbr{\nabla F_m(\wb^{(k,0)})}^2\right] \nn \\
    & \leq \frac{2i^2N\sigma^2}{\tau} + \frac{103i^2N^2\lr^2L^2(\tau-1)\sigma^2}{100} + \frac{69iN^2\kbar\lr^2L^2\tau(\tau-1)}{25M} \sum_{j'=1}^\kbar \sum_{m\in\sigma(j')} \expt \left[\nbr{\nabla F_m(\wb^{(k,0)})}^2\right] \nn \\
    & \quad + \frac{69 i N^2 \lr^2 L^4 \tau(\tau-1)}{25} \sum_{j'=1}^\kbar \expt\left[\nbr{\wb^{(k,j'-1)}-\wb^{(k,0)}}^2\right] \nn \\
    & \leq \frac{2i^2N\sigma^2}{\tau} + \frac{103i^2N^2\lr^2L^2(\tau-1)\sigma^2}{100}
    + \frac{69 i N^2 \kbar\lr^2L^2\tau(\tau-1)}{25}(2\nbr{\nabla F(\wb^{(k,0)})}^2+2\nu^2) \nn \\
    & \quad + \frac{69iN^2\lr^2L^4\tau(\tau-1)}{25}\left(\frac{83\lrt^2N^2\kbar^2(\kbar-1)}{20}\nbr{\nabla F(\wb^{(k,0)})}^2+\frac{51\lrt^2N\kbar^2(\kbar-1)\sigma^2}{50\tau}+\frac{3\lr^2\kbar\tau(\tau-1)\nu^2}{100}\right. \nn \\
    & \quad \left. + \frac{51 \lrt^2 N \kbar^2 (\kbar-1) \gamma^2}{25} \left(\frac{M/\kbar-N}{M/\kbar-1}\right) + \frac{41\lrt^2N^2\kbar^2(\kbar-1)\alpha^2}{10} \right) \tag{Using \Cref{lem1} and \Cref{lem2-3}} \nn \\
    & \leq i^2N\sigma^2\left(\frac{2}{\tau}+\frac{103N\lr^2L^2(\tau-1)}{100}\right)+
    iN^2\lr^2L^2\tau(\tau-1)\kbar\nbr{\nabla F(\wb^{(k,0)})}^2\left(\frac{138}{25}+\frac{69\times 83 L^2\lrt^2N^2\kbar^2}{1250}\right) \nn \\
    & \quad + \frac{69\times51 iN^3\lr^2L^4(\tau-1)\lrt^2\kbar^3\sigma^2}{1250}+iN^2\lr^2L^2\tau(\tau-1)\kbar\nu^2\left(\frac{138}{25}+\frac{207L^2\lr^2\tau(\tau-1)}{2500}\right) \nn \\
    & \quad +\frac{69\times 51 i N^3\lr^2L^4(\tau-1)\lrt^2\kbar^3\gamma^2}{25^2}\left(\frac{M/\kbar-N}{M/\kbar-1}\right) + \frac{69\times 41 i\lrt^2\lr^2N^4L^4\tau(\tau-1)\kbar^2(\kbar-1)\alpha^2}{250} \nn \\
    & \leq i^2N\sigma^2\left(\frac{2}{\tau}+\frac{103N\lr^2L^2(\tau-1)}{100}\right)+
    \frac{28iN^2\lr^2L^2\tau(\tau-1)\kbar}{5}\nbr{\nabla F(\wb^{(k,0)})}^2
    +\frac{141 iN^3\lr^2L^4(\tau-1)\lrt^2\kbar^3\sigma^2}{50} \nn \\
    & \quad + \frac{111 iN^2\lr^2L^2\tau(\tau-1)\kbar\nu^2}{20} + \frac{113i N^3\lr^2L^4(\tau-1)\lrt^2\kbar^3\gamma^2}{20}\left(\frac{M/\kbar-N}{M/\kbar-1}\right) \nn \\
    & \quad + \frac{23 i \lrt^2 \lr^2 N^4 L^4 \tau(\tau-1)\kbar^2(\kbar-1)\alpha^2}{2}. \label{eq:7-1-0}
\end{flalign}
where in the last two bounds we use $\lr \leq 1/(10 \tau N L \kbar))$. Plugging in \Cref{eq:7-1-0} to \Cref{eq:5-0-8-1-1} in \Cref{eq:5-0-8} we have
\begin{flalign}
    & \frac{3L\lrt^4}{2} \expt \left[\nbr{\rbt^{(k,0)}}^2\right]  \nn \\
    & \leq \frac{54\lrt^4\kbar N^2L^3}{25 }\sum_{i=1}^{\kbar-1}\left(i^2N\sigma^2\left(\frac{2}{\tau}+\frac{103N\lr^2L^2(\tau-1)}{100}\right) + \frac{28 i N^2 \lr^2 L^2 \tau(\tau-1)\kbar}{5}\nbr{\nabla F(\wb^{(k,0)})}^2 \right. \nn \\ 
    & \quad \left.+\frac{141 iN^3\lr^2L^4(\tau-1)\lrt^2\kbar^3\sigma^2}{50} + \frac{111 i N^2 \lr^2 L^2 \tau(\tau-1) \kbar \nu^2}{20} + \frac{113 i N^3 \lr^2 L^4 (\tau-1) \lrt^2 \kbar^3 \gamma^2}{20} \left(\frac{M/\kbar-N}{M/\kbar-1}\right)\right. \nn \\
    &  \quad \left. + \frac{23 i \lrt^2 \lr^2N^4L^4\tau(\tau-1)\kbar^2(\kbar-1)\alpha^2}{2}+2\expt\left[\nbr{\sum_{j'=1}^i\qbl^{(k,j')}}^2\right]\right) \nn \\
    & \leq \frac{18\lrt^4\kbar^3(\kbar-1) N^3L^3\sigma^2}{25 }\left(\frac{2}{\tau}+\frac{103N\lr^2L^2(\tau-1)}{100}\right) + \frac{54\times 14\lrt^4\lr^2\kbar^3(\kbar-1) N^4L^5(\tau-1)\tau}{125 }\nbr{\nabla F(\wb^{(k,0)})}^2 \nn \\
    & \quad + \frac{27\times 141\lrt^6\lr^2\kbar^5(\kbar-1) N^5L^7(\tau-1)\sigma^2}{1250}+\frac{27\times111\lrt^4\lr^2\kbar^3(\kbar-1)N^4L^5(\tau-1)\tau\nu^2}{500} \nn \\
    & \quad + \frac{27 \times 113\lrt^6\lr^2\kbar^5(\kbar-1) N^5L^7(\tau-1)\gamma^2}{500 }\left(\frac{M/\kbar-N}{M/\kbar-1}\right) \nn \\
    & \quad + \frac{54\times23\lrt^6\lr^2N^6L^7\tau(\tau-1)\kbar^4(\kbar-1)^2\alpha^2}{100}+\frac{108\lrt^4\kbar N^2L^3}{25 }\sum_{i=1}^{\kbar-1}\expt\left[\nbr{\sum_{j'=1}^i\qbl^{(k,j')}}^2\right] \nn \\
    &\leq \frac{18\lrt^4\kbar^3(\kbar-1) N^3L^3\sigma^2}{25 }\left(\frac{2}{\tau}+\frac{103N\lr^2L^2(\tau-1)}{100}\right) + \frac{54\times 14\lrt^4\lr^2\kbar^3(\kbar-1) N^4L^5(\tau-1)\tau}{125 }\nbr{\nabla F(\wb^{(k,0)})}^2 \nn \\
    & \quad +\frac{27\times 141\lrt^6\lr^2\kbar^5(\kbar-1) N^5L^7(\tau-1)\sigma^2}{1250}+\frac{27\times111\lrt^4\lr^2\kbar^3(\kbar-1)N^4L^5(\tau-1)\tau\nu^2}{500} \nn \\
    & \quad + \frac{27 \times 113\lrt^6\lr^2\kbar^5(\kbar-1) N^5L^7(\tau-1)\gamma^2}{500 }\left(\frac{M/\kbar-N}{M/\kbar-1}\right)+\frac{54\times23\lrt^6\lr^2N^6L^7\tau(\tau-1)\kbar^4(\kbar-1)^2\alpha^2}{100} \nn \\
    & \quad + \frac{108\lrt^4\kbar N^2L^3}{25 }\left(2 \kbar^3 N\left(\frac{M/\kbar-N}{M/\kbar-1}\right)\gamma^2 +4N^2\kbar^3\nbr{\nabla F(\wb^{(k,0)}}^2+\frac{4N^2\kbar^2(\kbar-1)\alpha^2}{3}\right) \nn \tag{Using \Cref{eq:3-2-2}}\\
    & \leq \lrt N\kbar(\lrt^3\kbar^3N^3L^3)\left(\frac{108\times4}{25}+\frac{54\times 14 \lrt^2 L^2}{125}\right)\nbr{\nabla F(\wb^{(k,0)})}^2+\frac{27\times 141 \lrt^2L\kbar^2N\sigma^2}{1250\tau}(\lrt^4\lr^2\kbar^4N^4L^6\tau(\tau-1)) \nn \\
    & \quad + \frac{18\lrt^2 L\kbar^2N\sigma^2}{25\tau}\left(2\lrt^2\kbar^2N^2L^2+\frac{103\lrt^2\kbar^2N^3L^4\lr^2\tau(\tau-1)}{100}\right) \nn \\
    & \quad + \lrt^4\kbar^4N^3L^3\left(\frac{M/\kbar-N}{M/\kbar-1}\right)\gamma^2\left(\frac{216}{25} + \frac{27\times 113 \lrt^2\lr^2\kbar^2N^3L^4(\tau-1)}{500}\right) \nn \\
    & \quad +\frac{27\times 111\lrt^2\lr^2L^3\kbar^2N^2\tau(\tau-1)\nu^2}{500}(\lrt^2\kbar^2N^2L^2) \nn \\
    & \quad + {\lrt^4\kbar^3(\kbar-1)N^4L^3\alpha^2}\left(\frac{108\times 4}{75} + \frac{54 \times23\lrt^2\lr^2\kbar^2N^2L^4\tau(\tau-1)}{100}\right) \nn \\
    & \leq \frac{9\lrt N\kbar}{500}\nbr{\nabla F(\wb^{(k,0)})}^2+\frac{3\lrt^2L\kbar^2N\sigma^2}{200\tau}+\frac{173\lrt^4\kbar^4N^3L^3\gamma^2}{20}\left(\frac{M/\kbar-N}{M/\kbar-1}\right) + \frac{3\lrt^2\lr^2L^3\kbar^2N^2\tau(\tau-1)\nu^2}{50} \nn \\
    & \quad +\frac{59\lrt^4\kbar^3(\kbar-1)N^4L^3\alpha^2}{10}. \nn
\end{flalign}
where, again, in the last two bounds we use $\lr \leq 1/(10 \tau N L \kbar))$.
\end{proof}

\newpage

\section{Proofs for the CyCP+Shuffled SGD Case}
\label{app:ccp+ssgd}
Now let us extend our analysis to clients locally performing shuffled SGD. Recall that for shuffled SGD, as shown in \Cref{algo1} we have that each client in $m\in\mathcal{S}^{(k,i)}$ receives the global model $\wb^{(k,i-1)}$ and initializes its local model as the global model i.e., $\wb_m^{(k,i-1,0)}=\wb^{(k,i-1)}$. Then the client performs shuffled SGD over its $B$ components sequentially with update rule $\wb_m^{(k,i-1,l+1)}=\wb_m^{(k,i-1,l)}-\lr\nabla F_{m,\pi_m^k(l)}(\wb^{(k,i-1,l)}),~l\in[0,...,B-1]$ where $\pi_m^k\sim\text{Unif}(\mathcal{P}_B)$ and $F_{m,\pi_m^k(l)}(\wb)$ is the $\pi_m^k(l)^{\text{th}}$ component of the local loss of client $m$ such that the sum of all the components for each client is equal to the local loss of that client i.e., $F_m(\wb)=\frac{1}{B}\sum_{l=0}^{B-1} F_{m,l}(\wb)$. Hence the update rule over the inner loop $i\in[\kbar]$ is as follows:
\begin{align}
\wb^{(k,i)}=\wb^{(k,i-1)}-\frac{\lr}{N}\sum_{m\in\mathcal{S}^{(k,i)}}\sum_{l=0}^{B-1}\nabla F_{m,\pi_m^k(l)}(\wb_m^{(k,i-1,l)}). \nn
\end{align}
Using
\begin{align}
    & \nabla F_{m,\pi_m^k(l)}(\wb_m^{(k,i-1,l)})=\nabla F_{m,\pi_m^k(l)}(\wb_m^{(k,i-1,l)})-\nabla F_{m,\pi_m^k(l)}(\wb^{(k,i-1)})+\nabla F_{m,\pi_m^k(l)}(\wb^{(k,i-1)}) \nn \\
    &= \nabla F_{m,\pi_m^k(l)}(\wb^{(k,i-1)})+\underbrace{\int_{0}^1\nabla^2 F_{m,\pi_m^k(l)}(\wb^{(k,i-1)}+t(\wb_m^{(k,i-1,l)}-\wb^{(k,i-1)}))dt}_{\defeq \hbbt_{m,l}^{(k,i-1)}}(\wb_m^{(k,i-1,l)}-\wb^{(k,i-1)}), \nn
\end{align}
we have that
\begin{align}
    \wb^{(k,i)}=\wb^{(k,i-1)}-\frac{\lr}{N}\sum_{m\in\mathcal{S}^{(k,i)}}\sum_{l=0}^{B-1} \left[ \nabla F_{m,\pi_m^k(l)}(\wb^{(k,i-1)})+\hbbt_{m,l}^{(k,i-1)} (\wb_m^{(k,i-1,l)}-\wb^{(k,i-1)}) \right]. \nn
\end{align}
Leveraging the fact that $\wb_m^{(k,i-1,0)}=\wb^{(k,i-1)}$ we can use recursion to get the update rule
\begin{align}
    \wb^{(k,i)}=\wb^{(k,i-1)}-\frac{\lr}{N}\sum_{m\in\mathcal{S}^{(k,i)}} \sum_{l=0}^{B-1} \left(\prod_{j=B-1}^{l+1}\left(\ib-\lr \hbbt_{m,j}^{(k,i-1)}\right)\right)\nabla F_{m,\pi_m^k(l)}(\wb^{(k,i-1)}) \label{eqn:locshuffsgd}
\end{align}
Similarly, we can define
\begin{align}
    \nabla F_{m,\pi_m^k(l)}(\wb^{(k,i-1)}) = \nabla F_{m,\pi_m^k(l)}(\wb^{(k,0)})+\underbrace{\int_{0}^1\nabla^2F_{m,\pi_m^k(l)}(\wb^{(k,0)}+t(\wb^{(k,i-1)}-\wb^{(k,0)})dt}_{\defeq \hbbl_{m,l}^{(k,i-1)}}(\wb^{(k,i-1)}-\wb^{(k,0)}), \nn
\end{align}
to get 
\begin{flalign}
    \wb^{(k,i)}&=\wb^{(k,i-1)}-\frac{\lr}{N}\sum_{m\in\mathcal{S}^{(k,i)}} \sum_{l=0}^{B-1} \left(\prod_{j=B-1}^{l+1}\left(\ib-\lr\hbbt_{m,j}^{(k,i-1)}\right)\right)\left(\nabla F_{m,\pi_m^k(l)}(\wb^{(k,0)})+\hbbl_{m,l}^{(k,i-1)}(\wb^{(k,i-1)}-\wb^{(k,0)})\right) \nn \\
    &\begin{aligned}
    =\wb^{(k,i-1)}-\lr\times\underbrace{\frac{1}{N}\sum_{m\in\mathcal{S}^{(k,i)}}
    \sum_{l=0}^{B-1} \left(\prod_{j=B-1}^{l+1}\left(\ib-\lr\hbbt_{m,j}^{(k,i-1)}\right)\right)\nabla F_{m,\pi_m^k(l)}(\wb^{(k,0)})}_{\defeq \tb^{(k,i)}}\\-\lr\times\underbrace{\frac{1}{N}\sum_{m\in\mathcal{S}^{(k,i)}}
    \sum_{l=0}^{B-1} \left(\prod_{j=B-1}^{l+1}\left(\ib-\lr\hbbt_{m,j}^{(k,i-1)}\right)\right)\hbbl_{m,l}^{(k,i-1)}}_{\defeq \tbb^{(k,i)}}(\wb^{(k,i-1)}-\wb^{(k,0)})
    \end{aligned} \label{eq:9-0-0}
\end{flalign}
Unrolling \Cref{eq:9-0-0} we can obtain the update rule for the outer loop as follows:
\begin{align}
    \wb^{(k+1,0)}=\wb^{(k,0)}-\lr\sum_{i=1}^\kbar\left(\prod_{j=\kbar}^{i+1}(\ibd-\lr\tbb^{(k,j)})\right)\tb^{(k,i)} \label{eq:9-0-1}
\end{align}
Applying summation by parts to $\prod_{j=\kbar}^{i+1}(\ibd-\lr\tbb^{(k,j)})$ and $\tb^{(k,i)}$ in \Cref{eq:9-0-1} and then again to $\prod_{j=B-1}^{l+1}\left(\ib-\lr\hbbt_{m,j}^{(k,i-1)}\right)$ and $\nabla F_{m,\pi_m^k(l)}(\wb^{(k,0)})$ in $\tb^{(k,i)}$ we can rewrite \Cref{eq:9-0-1} as 
\begin{flalign}
    &\begin{aligned}
        &\wb^{(k+1)}=\wb^{(k,0)}-\frac{\lr}{N}\sum_{i=1}^\kbar\sum_{m\in\st^{(k,i)}}\sum_{l=0}^{B-1}\nabla F_{m,\pi_m^k(l)}(\wb^{(k,0)})\\
        &+\lr^2
        \times\underbrace{\frac{1}{N}\sum_{i=1}^\kbar\sum_{m\in\st^{(k,i)}}\sum_{l=0}^{B-2}\left(\prod_{t=B-1}^{l+2}(\ibd-\lr \hbbt_{m,t}^{(k,i-1)})\right)\hbbt_{m,l+1}^{(k,i-1)}\sum_{j=0}^l \nabla F_{m,\pi_m^k(j)}(\wb^{(k,0)})}_{\defeq\rb_1^{(k,0)}}\\
        &+\lr^2
        \times\underbrace{\frac{1}{N}\sum_{i=1}^{\kbar-1}\left(\prod_{j=\kbar}^{i+2}(\ibd-\lr\tbb^{(k,j)})\right)\tbb^{(k,i+1)}\left(\sum_{j=1}^i\sum_{m\in\st^{(k,j)}}\sum_{l=0}^{B-1}\nabla F_{m,\pi_m^k(l)}(\wb^{(k,0)})\right)}_{\defeq\rb_2^{(k,0)}}\\
        &-\lr^3\times\underbrace{\frac{1}{N}\sum_{i=1}^{\kbar-1}\left(\prod_{j=\kbar}^{i+2}(\ibd-\lr\tbb^{(k,j)})\right)\tbb^{(k,i+1)}\left(\sum_{j=1}^i\sum_{m\in\st^{(k,j)}}\sum_{l=0}^{B-2}\left(\prod_{t=B-1}^{l+2}(\ibd-\lr\hbbt_{m,t}^{(k,i-1)})\right)\hbbt_{m,l+1}^{(k,i-1)}\sum_{j'=0}^l\nabla F_{m,\pi_m^k(j')}(\wb^{(k,0)})\right)}_{\defeq\rb_3^{(k,0)}} \nn
    \end{aligned}
    \\
    &=\wb^{(k,0)}-\frac{\lr}{N}\sum_{i=1}^\kbar\sum_{m\in\st^{(k,i)}}\sum_{l=0}^{B-1}\nabla F_{m,\pi_m^k(l)}(\wb^{(k,0)})+\lr^2\rb_1^{(k,0)}+\lr^2\rb_2^{(k,0)}-\lr^3\rb_3^{(k,0)}. \nn
\end{flalign}
Taking expectation conditioned on all the past till $\wb^{(k,0)}$ 
, we get
\begin{align}
    \expt_k[\wb^{(k+1,0)}]-\wb^{(k,0)}=-\lr\kbar B\nabla F(\wb^{(k,0)})+\lr^2\expt_k[\underbrace{\rb_1^{(k,0)}+\rb_2^{(k,0)}-\lr\rb_3^{(k,0)}}_{\defeq \rb^{(k,0)}}], \nn
\end{align}
which follows from the expectation over the selected client set $\st^{(k,i)},~i\in[\kbar]$. 

\subsection{Intermediate Results}
\label{sec:int_results_shuffleSGD}

\begin{lemma}[Bound on the Norm of the Error Term Arrising Due to \cycp]
\label{lem:error_norm_shuffleSGD}
\begin{align*}
    & \nbr{\expt_k[\rb^{(k,0)}]} \leq  \left(\frac{2e^{1/10}}{3}+ \frac{2\text{exp}(e^{1/10}/10)e^{1/5}\lr BL(\kbar-1)}{3}\right)\sqrt{8\log\left(\frac{4MBK}{\delta}\right)}\kbar L(B^{3/2}-1)\nub \\
    & \quad + \left(\frac{e^{1/10}}{2} + \frac{\text{exp}(e^{1/10}/10)e^{1/5}\lr BL(\kbar-1)}{2}\right)\kbar LB(B-1)\nu + \frac{\text{exp}(e^{1/10}/10)e^{1/10}B^2L\kbar(\kbar-1)\alpha}{2} \\
    & \quad +\left(\frac{e^{1/10}}{2}+ \frac{\text{exp}(e^{1/10}/10)e^{1/10}(\kbar-1)}{2} + \frac{\text{exp}(e^{1/10}/10)e^{1/5}\lr(B-1)L(\kbar-1)}{2}\right)\kbar LB^2\nbr{\nabla F(\wb^{(k,0)})}.
\end{align*}
\end{lemma}

\begin{lemma}[Bound on the Norm Square of the Error Term Arrising Due to \cycp]
\label{lem:error_normsq_shuffleSGD}
\begin{flalign}
    & 3 L \lr^4 \expt_k \left[ \nbr{\rb_1^{(k,0)}}^2 + \nbr{\rb_2^{(k,0)}}^2 + \nbr{\rb_3^{(k,0)}}^2 \right] \nn \\
    &\leq \frac{31\lr^3\kbar L^2 (B-1)^2}{10}\log{(4MBK/\delta)}\nub^2+\frac{\lr^3\kbar L^2B(B-1)^2\nu^2}{2}+\frac{3\lr\kbar B}{100}\nbr{\nabla F(\wb^{(k,0)})}^2 \nn \\
    & \quad +\frac{6(\kbar-1)\lr^4\kbar^3B^4L^3\text{exp}(e^{1/10}/5)e^{1/5}}{N}\left(\frac{M/\kbar-N}{M/\kbar-1}\right)\gamma^2+4\kbar^2(\kbar-1)^2B^4L^3\lr^4\alpha^2\text{exp}(e^{1/10}/5)e^{1/5} \nn.
\end{flalign}
\end{lemma}
\Cref{lem:error_norm_shuffleSGD} and \Cref{lem:error_normsq_shuffleSGD} bound the error that arrises due to \cycp. The exponential constants arrise due to the learning rate set as $\lr=\log(MK^2)/\tau\mu N\kbar K$ and the lower bound on the cycle-epoch $K\geq 10\kappa\log{(MK^2)}$. The bounds depend on the intra-component heterogeneity $\nub$ and intra-group and inter-group heterogeneity $\gamma$ and $\alpha$, where $\nu=\gamma+\alpha$.
\subsection{Proof for \Cref{theo3:ssgd}}

\begin{thm*}
With \Cref{as1}-\Cref{as3}, and \Cref{as6} we have that with step-size $\lr=\log(MBK^2)/\mu \kbar BK$ and $K\geq 10\kappa\log{(MBK^2)}$ where $\kappa=L/\mu$ we have that the convergence error is bounded as:
\begin{align}
   \expt[F(\wb^{(K,0)})]-F^* & \leq (1-\mu\lr\kbar B)^K(F(\wb^{(0,0)})-F^*) + \mco \lp \frac{\kappa\lr B\kbar}{N}\left(\frac{M/\kbar-N}{M/\kbar-1}\right)\gamma^2 \rp \nn \\
    & \quad + \tilde{\mco} \lp \lr^{2} (B-1) \kappa L \nub^2 + \lr^2 \kappa L(B-1)^2\nu^2 \rp  + \mco \lp (\kbar-1)^2B^2L\kappa\lr^2\alpha^2 \rp. \nn
\end{align}
With $\lr=\log(MBK^2)/\mu B\kbar K$, we get
\begin{align}
    \expt[F(\wb^{(K,0)})]-F^* & \leq \frac{F(\wb^{(0,0)})-F^*}{MBK^2}+\bigto{\frac{\kappa\gamma^2}{\mu N K}\left(\frac{M/\kbar-N}{M/\kbar-1}\right)} + \bigto{\frac{(B-1)}{B^2} \frac{\kappa^2\nub^2}{\mu \kbar^2 K^2}} \nn \\
    & \quad +\bigto{\frac{\kappa^2 (B-1)^2\nu^2}{\mu B^2\kbar^2K^2}} + \bigto{\frac{\kappa^2(\kbar-1)^2\alpha^2}{\mu \kbar^2  K^2}}, \nn
\end{align}
where $\nu=\gamma+\alpha$ and $\tilo(\cdot)$ subsumes all log-terms and constants.
\end{thm*}

\begin{coro}
Recall that the total communication rounds translates to $T=\kbar K$ and therefore in terms of $T$ we have that the bound becomes
\begin{align*}
  \expt[F(\wb^{(K,0)})]-F^* & \leq \frac{\kbar^2(F(\wb^{(0,0)})-F^*)}{MBT^2} + \bigto{\frac{\kappa^2 (B-1)^2\nu^2}{\mu B^2 T^2}}+\bigto{\frac{\kappa^2(\kbar-1)^2\alpha^2}{\mu T^2}} \\
    & \quad + \bigto{\frac{\kappa^2 \overline{\nu}^2}{\mu T^2} \frac{(B-1)}{B^2}} + \bigto{\frac{\kappa\kbar\gamma^2}{\mu N T}\left(\frac{M/\kbar-N}{M/\kbar-1}\right)}.
\end{align*}
\end{coro}
Since $K\geq 10\kappa\log{(MBK^2)}$, we have $T\geq10\kappa\kbar\log{(MB\kbar^2T^2)}$ and one cannot increase $\kbar$ without increasing $T$ accordingly due to its lower bound depending on $\kbar$.

\begin{proof}
Using the $L$-smoothness property of the global objective $F$ we have
\begin{align}
    & \expt_k[F(\wb^{(k+1,0)})]-F(\wb^{(k,0)}) \leq \inner{\nabla F(\wb^{(k,0)})}{\expt_k[\wb^{(k+1,0)}]-\wb^{(k,0)}}+\frac{L}{2}\expt\left[\nbr{\wb^{(k+1,0)}-\wb^{(k,0)}}^2\right] \nn \\
    & \leq -\lr\kbar B\nbr{\nabla F(\wb^{(k,0)})}^2+\lr^2\nbr{\expt_k[\rb^{(k,0)}]}\nbr{\nabla F(\wb^{(k,0)})}
    +\frac{L}{2}\expt_k \left[\nbr{\wb^{(k+1,0)}-\wb^{(k,0)}}^2\right]. \label{eq:9-0-2}
\end{align}

Plugging \Cref{lem:error_norm_shuffleSGD}, we bound the second term in the RHS of \Cref{eq:9-0-2}.
\begin{flalign}
    & \lr^2 \nbr{\expt_k [\rb^{(k,0)}]}\nbr{\nabla F(\wb^{(k,0)})} \nn \\
    & \leq \left(\frac{2e^{1/10}}{3}+ \frac{2\text{exp}(e^{1/10}/10)e^{1/5}\lr BL(\kbar-1)}{3}\right)\sqrt{8\log\left(\frac{4MBK}{\delta}\right)}\lr^2\kbar L(B^{3/2}-1)\nub\nbr{\nabla F(\wb^{(k,0)})} \nn \\
    & \quad + \left(\frac{e^{1/10}}{2} + \frac{\text{exp}(e^{1/10}/10) e^{1/5}\lr BL(\kbar-1)}{2}\right)\lr^2\kbar LB(B-1)\nu\nbr{\nabla F(\wb^{(k,0)})} \nn \\
    & \quad + \frac{\lr^2\text{exp}(e^{1/10}/10)e^{1/10}B^2L\kbar(\kbar-1)\alpha}{2}\nbr{\nabla F(\wb^{(k,0)})}  \nn \\
    & \quad + \left(\frac{e^{1/10}}{2}+ \frac{\text{exp}(e^{1/10}/10)e^{1/10}(\kbar-1)}{2} + \frac{\text{exp}(e^{1/10}/10)e^{1/5}\lr(B-1)L(\kbar-1)}{2}\right)\lr^2\kbar LB^2\nbr{\nabla F(\wb^{(k,0)})}^2 \nn \\
    & \leq -\lr \kbar B \nbr{\nabla F(\wb^{(k,0)})}^2
    +\frac{L}{2}\expt_k \left[\nbr{\wb^{(k+1,0)}-\wb^{(k,0)}}^2\right] \nn \\
    & \quad + \frac{47}{20}\sqrt{\log\left(\frac{4MBK}{\delta}\right)}\lr^2\kbar L(B^{3/2}-1)\nub\nbr{\nabla F(\wb^{(k,0)})}+\frac{63\lr^2\kbar LB(B-1)\nu}{100}\nbr{\nabla F(\wb^{(k,0)})} \nn \\
    & \quad + \frac{31 \lr^2 B^2 L \kbar(\kbar-1) \alpha}{50}\nbr{\nabla F(\wb^{(k,0)})} \nn \\
    & \quad + \left( \frac{e^{1/10}}{20}+ \frac{\text{exp}(e^{1/10}/10)e^{1/10}}{20} + \frac{\text{exp}(e^{1/10}/10)e^{1/5}}{200}\right)\lr\kbar B\nbr{\nabla F(\wb^{(k,0)})}^2 \nn \\
    & \leq \frac{13}{100} \lr \kbar B\nbr{\nabla F(\wb^{(k,0)})}^2 + \frac{47}{20} \sqrt{\log \left( \frac{4MBK}{\delta} \right)} \lr^2 \kbar L(B^{3/2}-1) \nub \nbr{\nabla F(\wb^{(k,0)})} \nn \\
    & \quad + \frac{63\lr^2\kbar LB(B-1)\nu}{100}\nbr{\nabla F(\wb^{(k,0)})} + \frac{31\lr^2B^2L\kbar(\kbar-1)\alpha}{50}\nbr{\nabla F(\wb^{(k,0)})}. \label{eq:9-0-5}
\end{flalign}
where in the last two bounds we use $\lr=\log(MBK^2)/\mu \kbar BK,~K\geq 10\kappa\log{(MBK^2)}$. Next we bound the third term in the RHS of \Cref{eq:9-0-2} as follows:
\begin{flalign}
    &\begin{aligned}
    &\frac{L}{2}\expt_k \left[\nbr{\wb^{(k+1,0)}-\wb^{(k,0)}}^2\right]=\frac{L}{2}\expt_k \left[\nbr{-\frac{\lr}{N}\sum_{i=1}^\kbar\sum_{m\in\st^{(k,i)}}\sum_{l=0}^{B-1}\nabla F_{m,\pi_m^k(l)}(\wb^{(k,0)})+\lr^2\rb^{(k,0)}}^2\right]
    \end{aligned} \nn \\
    &\begin{aligned}
    &\leq L\expt_k \left[\nbr{-\frac{\lr}{N}\sum_{i=1}^\kbar\sum_{m\in\st^{(k,i)}}\sum_{l=0}^{B-1}\nabla F_{m,\pi_m^k(l)}(\wb^{(k,0)})}^2\right]+L\expt_k \left[\nbr{\lr^2\rb^{(k,0)}}^2\right]
    \end{aligned} \nn \\
    &\begin{aligned}
    &\leq L\lr^2B^2 \kbar^2 \expt_k \left[\nbr{\frac{1}{\kbar} \sum_{i=1}^\kbar \frac{1}{N} \sum_{m\in\st^{(k,i)}} \nabla F_{m}(\wb^{(k,0)})}^2\right] + 3L \lr^4 \expt_k \left[ \nbr{\rb_1^{(k,0)}}^2 + \nbr{\rb_2^{(k,0)}}^2 + \lr^2\nbr{\rb_3^{(k,0)}}^2 \right]
    \end{aligned} \label{eq:9-0-5-1} \\
    & \leq L\lr^2B^2\kbar^2 \left[ \frac{1}{N} \left( \frac{M/\kbar-N}{M/\kbar-1} \right) \gamma^2 + \nbr{\nabla F(\wb^{(k,0)})}^2 \right] + 3 L\lr^4 \expt_k \left[ \nbr{\rb_1^{(k,0)}}^2 + \nbr{\rb_2^{(k,0)}}^2 + \lr^2\nbr{\rb_3^{(k,0)}}^2 \right]. \label{eq:9-0-6} 
\end{flalign} 
where \Cref{eq:9-0-5-1} uses $\|a+b+c\|^2\leq3(\|a\|^2+\|b\|^2+\|c\|^2)$ and \Cref{eq:9-0-6} uses \Cref{eq:3-1-6}. Plugging the bound from \Cref{lem:error_normsq_shuffleSGD} in \Cref{eq:9-0-6} we have
\begin{flalign}
    & \frac{L}{2}\expt_k \left[\nbr{\wb^{(k+1,0)}-\wb^{(k,0)}}^2\right] 
    \leq \frac{21 L \lr^2 B^2 \kbar^2}{10N} \left(\frac{M/\kbar-N}{M/\kbar-1}\right)\gamma^2+\frac{31\lr^3\kbar L^2 (B-1)^2}{10}\log{(4MBK/\delta)}\nub^2 \nn \\
    & \quad + \frac{\lr^3\kbar L^2B(B-1)^2\nu^2}{2} + \frac{\lr\kbar B}{20}\nbr{\nabla F(\wb^{(k,0)})}^2+\frac{31\kbar(\kbar-1)^2B^3L^2\lr^3\alpha^2}{50}. \label{eq:9-2-4}
\end{flalign}
Finally, substituting \eqref{eq:9-0-5} and \eqref{eq:9-0-6} into \Cref{eq:9-0-2} we get
\begin{flalign}
    & \expt_k[F(\wb^{(k+1,0)})]-F(\wb^{(k,0)}) \nn \\
    & \leq -\frac{82}{100} \lr \kbar B \nbr{\nabla F(\wb^{(k,0)})}^2+\frac{21L\lr^2B^2\kbar^2}{10N}\left(\frac{M/\kbar-N}{M/\kbar-1}\right)\gamma^2 \nn \\
    & \quad + \left( \frac{\lr^{1/2} \kbar^{1/2}B^{1/2}}{10}\nbr{\nabla F(\wb^{(k,0)})}\right)\left(\frac{47\lr^{3/2}\kbar^{1/2}L(B^{3/2}-1)\sqrt{\log\left(\frac{4MBK}{\delta}\right)}\nub}{2B^{1/2}}\right) \nn \\
    & \quad + \left( \frac{\lr^{1/2}\kbar^{1/2}B^{1/2}}{10}\nbr{\nabla F(\wb^{(k,0)})}\right)\left(\frac{63\lr^{3/2}\kbar^{1/2}LB^{1/2}(B-1)\nu}{10}\right) \nn \\
    & \quad + \left( \frac{\lr^{1/2} \kbar^{1/2}B^{1/2}}{25} \nbr{\nabla F(\wb^{(k,0)})} \right) \left( \frac{31\lr^{3/2}\kbar^{1/2}(\kbar-1)LB^{3/2}\alpha}{2}\right) \nn \\
    & \quad + \frac{31\lr^3\kbar L^2 (B-1)^2}{10} \log{(4MBK/\delta)} \nub^2+\frac{\lr^3\kbar L^2B(B-1)^2\nu^2}{2}+\frac{31\kbar(\kbar-1)^2B^3L^2\lr^3\alpha^2}{50} \nn \\
    \Rightarrow & \expt[F(\wb^{(k+1,0)})]-F^* \leq (1-\mu\lr B\kbar)(\expt[F(\wb^{(k,0)})]-F^*)+\frac{21L\lr^2B^2\kbar^2}{10N}\left(\frac{M/\kbar-N}{M/\kbar-1}\right)\gamma^2 \nn \\
    & \quad + \left(\frac{47^2(B^{3/2}-1)^2}{8B}+\frac{31(B-1)^2}{10}\right)\lr^{3}\kbar L^2\log\left(\frac{4MBK}{\delta}\right)\nub^2+\left(\frac{63^2}{200}+\frac{1}{2}\right)\lr^3\kbar L^2B(B-1)^2\nu^2 \nn \\
    & \quad + 121\kbar(\kbar-1)^2B^3L^2\lr^3\alpha^2. \label{eq:9-3-0}~~~~\text{(}\because~\Cref{lem:Young})
\end{flalign}
Unrolling \Cref{eq:9-3-0} we have
\begin{flalign}
    &\begin{aligned} 
    &\expt[F(\wb^{(k+1,0)})]-F^*\leq (1-\mu\lr\kbar B)^K(F(\wb^{(0,0)})-F^*)+\frac{21\kappa\lr B\kbar}{10N}\left(\frac{M/\kbar-N}{M/\kbar-1}\right)\gamma^2\\
    &+\left(\frac{47^2(B^{3/2}-1)^2}{8B^2}+\frac{31(B-1)^2}{10B}\right)\lr^{2}\kappa L\log\left(\frac{4MBK}{\delta}\right)\nub^2+\left(\frac{63^2}{200}+\frac{1}{2}\right)\lr^2\kappa L(B-1)^2\nu^2\\
    &+121(\kbar-1)^2B^2L\kappa\lr^2\alpha^2
    \end{aligned}
\end{flalign}
With $\lr=\log(MBK^2)/\mu B\kbar K$, we have
\begin{flalign}
    &\expt[F(\wb^{(K,0)})]-F^* \nn \\
    &\begin{aligned} 
    &\leq \frac{F(\wb^{(0,0)})-F^*}{MBK^2}+\frac{21\kappa \log(MBK^2)}{10\mu N K}\left(\frac{M/\kbar-N}{M/\kbar-1}\right)\gamma^2 \nn \\
    &\quad +\left(\frac{47^2(B^{3/2}-1)^2}{8B^4}+\frac{31(B-1)^2}{10B^3}\right)\frac{\kappa^2 \log\left(4MBK/{\delta}\right)\log^2(MBK^2)\nub^2}{\mu \kbar^2 K^2}+\frac{21\kappa^2 (B-1)^2\log^2(MBK^2)\nu^2}{10\mu B^2\kbar^2K^2}\\
    &\quad +\frac{121(\kbar-1)^2\kappa^2\log^2(MBK^2)\alpha^2}{\mu\kbar^2K^2}
    \end{aligned}\\
    &\begin{aligned} 
    &= \frac{F(\wb^{(0,0)})-F^*}{MBK^2}+\bigto{\frac{\kappa\gamma^2}{\mu N K}\left(\frac{M/\kbar-N}{M/\kbar-1}\right)}+\bigto{\left(\frac{(B^{3/2}-1)^2}{B^4}+\frac{(B-1)^2}{B^3}\right)\frac{\kappa^2\nub^2}{\mu \kbar^2 K^2}} \nn \\
    &\quad +\bigto{\frac{\kappa^2 (B-1)^2\nu^2}{\mu B^2\kbar^2K^2}}+\bigto{\frac{\kappa^2(\kbar-1)^2\alpha^2}{\mu \kbar^2  K^2}},
    \end{aligned}
\end{flalign}
which concludes the proof.
\end{proof}

\subsection{Proof for \Cref{theo:compssgdlrr}}
Recall that the total cost for CyCP+Shuffled SGD and LocalRR respectively is as follows:
\begin{flalign}
&\begin{aligned}C_{\text{SSGD}}(\epsilon)=
\bigto{
\frac{c_{\text{SSGD}} \kappa}{\sqrt{\mu \epsilon}}\left(\frac{\kbar \sqrt{\mu}}{\kappa \sqrt{MB}}+{\nu}+\kbar\alpha+\frac{\overline{\nu}}{\sqrt{B}}\right)} 
\end{aligned}  \\
&C_{\text{LocalRR}}(\epsilon)=\bigto{\frac{\kbar c_{\text{SSGD}}}{\sqrt{\epsilon}}\left(\frac{1}{\sqrt{MB}}+\nu+\frac{\overline{\nu}}{\sqrt{B}}\right)} 
\end{flalign}

For $C_{\text{SSGD}}(\epsilon)<C_{\text{LocalRR}}(\epsilon)$ to be true, we need to have
Hence, equivalently, we need to have
\begin{align}
(1-\kbar)\nu+\kbar\alpha+(1-\kbar)\frac{\nub}{\sqrt{B}}<0. \nn
\end{align}
Since we have that $\kbar=M/N$, we have
\begin{align}
\left(\frac{N-M}{N}\right)\nu+\frac{M\alpha}{N}+\frac{(N-M)\nub}{N\sqrt{B}}<0\\
({N-M})\nu+{M\alpha}+\frac{(N-M)\nub}{\sqrt{B}}<0 \nn \\
\left(\alpha-\nu-\frac{\nub}{\sqrt{B}}\right)M<-N\nu-\frac{N\nub}{\sqrt{B}} \nn \\
\left(-\gamma-\frac{\nub}{\sqrt{B}}\right)M<-N\left(\alpha+\gamma+\frac{\nub}{\sqrt{B}}\right) \nn \\
M>N\left(\alpha+\gamma+\frac{\nub}{\sqrt{B}}\right)/\left(\gamma+\frac{\nub}{\sqrt{B}}\right)=N\left(1+\frac{\alpha}{\gamma+\frac{\nub}{\sqrt{B}}}\right),  \nn 
\end{align}
completing the proof.
\comment{
\begin{align}
\frac{\kbar-1}{\sqrt{MB}}+\left(\frac{\kbar}{\sqrt{M}}-1\right)\nu-\kbar\left(1-\frac{1}{\sqrt{M}}\right)\nu-\frac{(\kbar-1)\nub}{\sqrt{B}}\\
=\frac{\kbar-1}{\sqrt{MB}}+\frac{2\kbar\nu}{\sqrt{M}}-(\kbar+1)\nu-\frac{(\kbar-1)\nub}{\sqrt{B}}\\
=\left(\frac{M}{N}-1\right)\frac{1}{\sqrt{MB}}-\left(\frac{M}{N}+1-\frac{2M}{N\sqrt{M}}\right)\nu-\left(\frac{M}{N}-1\right)\frac{\nub}{\sqrt{B}}~~~~~(\because \kbar=M/N)\\
=\frac{\sqrt{M}}{N\sqrt{B}}-\frac{1}{\sqrt{MB}}-\frac{M\nu}{N}-\nu+\frac{2\sqrt{M}\nu}{N}-\frac{M\nub}{N\sqrt{B}}+\frac{\nub}{\sqrt{B}}<0\\
\therefore \sqrt{M}\left(\frac{1}{N\sqrt{B}}+\frac{2\nu}{N}\right)<M\left(\frac{\nub}{N\sqrt{B}}+\frac{\nu}{N}\right)+\frac{1}{\sqrt{MB}}+\nu-\frac{\nub}{\sqrt{B}}\\
\leq M\left(\frac{\nub}{N\sqrt{B}}+\frac{\nu}{N}\right)+\frac{1}{\sqrt{B}}+\nu-\frac{\nub}{\sqrt{B}}~~~~~(\because M\geq 1)\\
\frac{1}{N\sqrt{B}}+\frac{2\nu}{N} \leq \frac{M}{\sqrt{M}}\left(\frac{\nub}{N\sqrt{B}}+\frac{\nu}{N}\right)+\frac{1}{\sqrt{M}}\left(\frac{1-\nub}{\sqrt{B}}+\nu\right)\\
\frac{1}{N\sqrt{B}}+\frac{2\nu}{N} \leq \frac{M}{N}\left(\frac{\nub}{N\sqrt{B}}+\frac{\nu}{N}\right)+\frac{1}{N}\left(\frac{1-\nub}{\sqrt{B}}+\nu\right)~~~~~(\because M>N^2)\\
\therefore N\left(\frac{\frac{1}{\sqrt{B}}+{2\nu}-\nu+\frac{\nub-1}{\sqrt{B}}}{\frac{\nub}{\sqrt{B}}+{\nu}}\right)=N\left(\frac{\nu+\frac{\nub}{\sqrt{B}}}{\frac{\nub}{\sqrt{B}}+{\nu}}\right)=N\leq M
\end{align}
}

\subsection{Proofs on Intermediate Lemmas}
\label{sec:proofs_int_results_shuffleSGD}

\begin{proof}[Proof of \cref{lem:error_norm_shuffleSGD}]
In the following, we bound all three components of $\rb^{(k,0)}$ separately as follows.
\begin{align}
    & \nbr{\expt_k[\rb_1^{(k,0)}]} = \nbr{\expt_k \left[\frac{1}{N}\sum_{i=1}^\kbar\sum_{m\in\st^{(k,i)}}\sum_{l=0}^{B-2}\left(\prod_{t=B-1}^{l+2}(\ibd-\lr\hbbt_{m,t}^{(k,i-1)})\right)\hbbt_{m,l+1}^{(k,i-1)}\sum_{j=0}^l \nabla F_{m,\pi_m^k(j)}(\wb^{(k,0)})\right]} \nn \\
    &= 
    \nbr{\frac{\kbar}{M}\sum_{i=1}^\kbar\sum_{m\in\sigma(i)}\sum_{l=0}^{B-2} \expt_k \left[ \left(\prod_{t=B-1}^{l+2}(\ibd-\lr\hbbt_{m,t}^{(k,i-1)})\right)\hbbt_{m,l+1}^{(k,i-1)}\sum_{j=0}^l \nabla F_{m,\pi_m^k(j)}(\wb^{(k,0)}) \right]} \tag{$\because$ unbiased client sampling} \\
    & \leq \frac{\kbar}{M}\sum_{i=1}^\kbar\sum_{m\in\sigma(i)}\sum_{l=0}^{B-2}\nbr{\left(\prod_{t=B-1}^{l+2}(\ibd-\lr\hbbt_{m,t}^{(k,i-1)})\right)\hbbt_{m,l+1}^{(k,i-1)}\sum_{j=0}^l \nabla F_{m,\pi_m^k(j)}(\wb^{(k,0)})} \nn \\
    & \leq \frac{\kbar}{M}\sum_{i=1}^\kbar\sum_{m\in\sigma(i)}\sum_{l=0}^{B-2}\nbr{\left(\prod_{t=B-1}^{l+2}(\ibd-\lr\hbbt_{m,t}^{(k,i-1)})\right)}\nbr{\hbbt_{m,l+1}^{(k,i-1)}}\nbr{\sum_{j=0}^l \nabla F_{m,\pi_m^k(j)}(\wb^{(k,0)})} \nn ~~~~\text{(}\because\text{Submultiplicativity of Norms)}\\
    & \leq \frac{(1+\lr L)^B\kbar L}{M}\sum_{i=1}^\kbar\sum_{m\in\sigma(i)}\sum_{l=0}^{B-2}\nbr{\sum_{j=0}^l \nabla F_{m,\pi_m^k(j)}(\wb^{(k,0)})} \tag{$\because \nbr{\hbbt_{m,l+1}^{(k,i-1)}} \leq L$} \\
    & \leq  \frac{e^{1/10}\kbar L}{M}\sum_{i=1}^\kbar\sum_{m\in\sigma(i)}\sum_{l=0}^{B-2}\nbr{\sum_{j=0}^l \nabla F_{m,\pi_m^k(j)}(\wb^{(k,0)})} \tag{Using $\lr \leq \frac{1}{10 B L}$} \\
    & \leq \frac{e^{1/10}\kbar L}{M}\sum_{i=1}^\kbar\sum_{m\in\sigma(i)}\sum_{l=0}^{B-2}\left(\nub\sqrt{8(l+1)\log\left(\frac{4MBK}{\delta}\right)}+(l+1)\nbr{\nabla F_m(\wb^{(k,0)})}\right) \tag{Lemma 8 in \cite{yun2022shuf}} \\
    & \leq \frac{2e^{1/10}\kbar L(B^{3/2}-1)\nub}{3}\sqrt{8\log\left(\frac{4MBK}{\delta}\right)}+\frac{e^{1/10}\kbar LB(B-1)}{2M}\sum_{i=1}^\kbar\sum_{m\in\sigma(i)}\nbr{\nabla F_m(\wb^{(k,0)})} \nn \\ 
    & \leq \frac{2e^{1/10}\kbar L(B^{3/2}-1)\nub}{3}\sqrt{8\log\left(\frac{4MBK}{\delta}\right)} +\frac{e^{1/10}\kbar LB(B-1)\nu}{2} \left[ \nu + \nbr{\nabla F(\wb^{(k,0)})} \right]. \label{eq:r1_norm_bd}
\end{align}
where in the last bound we use \Cref{lem1}. We can similarly bound the next noise term as
\begin{align}
    & \nbr{\expt_k[\rb_2^{(k,0)}]} = \nbr{\expt_k \left[\frac{1}{N}\sum_{i=1}^{\kbar-1}\left(\prod_{j=\kbar}^{i+2}(\ibd-\lr\tbb^{(k,j)})\right)\tbb^{(k,i+1)}\left(\sum_{j=1}^i\sum_{m\in\st^{(k,j)}}\sum_{l=0}^{B-1}\nabla F_{m,\pi_m^k(l)}(\wb^{(k,0)})\right)\right]} \nn \\
    & \leq \frac{1}{N}\sum_{i=1}^{\kbar-1}\nbr{\expt_k \left[\left(\prod_{j=\kbar}^{i+2}(\ibd-\lr\tbb^{(k,j)})\right)\tbb^{(k,i+1)}\left(\sum_{j=1}^i\sum_{m\in\st^{(k,j)}}\sum_{l=0}^{B-1}\nabla F_{m,\pi_m^k(l)}(\wb^{(k,0)})\right)\right]} \tag{Using \Cref{lem:jensens}} \\
    & \leq \frac{1}{N}\sum_{i=1}^{\kbar-1}\nbr{\expt_k \left[\left(\prod_{j=\kbar}^{i+2}(\ibd-\lr\tbb^{(k,j)})\right)\right]}\nbr{\expt_k \left[\tbb^{(k,i+1)}\right]}\nbr{\expt_k \left[\sum_{j=1}^i\sum_{m\in\st^{(k,j)}}\sum_{l=0}^{B-1}\nabla F_{m,\pi_m^k(l)}(\wb^{(k,0)})\right]} \label{eq:9-0-3}
\end{align}
Since
\begin{align}
    \nbr{\tbb^{(k,i)}} &= \nbr{\frac{1}{N}\sum_{m\in\mathcal{S}^{(k,i)}}
    \sum_{l=0}^{B-1} \left(\prod_{j=B-1}^{l+1}\left(\ib-\lr\hbbt_{m,j}^{(k,i-1)}\right)\right)\hbbl_{m,l}^{(k,i-1)}} \nn \\
    & \leq \frac{1}{N}\sum_{m\in\mathcal{S}^{(k,i)}}
    \sum_{l=0}^{B-1}\nbr{\left(\prod_{j=B-1}^{l+1}\left(\ib-\lr\hbbt_{m,j}^{(k,i-1)}\right)\right)\hbbl_{m,l}^{(k,i-1)}} \nn \\
    & \leq \frac{1}{N}\sum_{m\in\mathcal{S}^{(k,i)}}
    \sum_{l=0}^{B-1}\nbr{\prod_{j=B-1}^{l+1}\left(\ib-\lr\hbbt_{m,j}^{(k,i-1)}\right)}\nbr{\hbbl_{m,l}^{(k,i-1)}} \nn \\
    & \leq \frac{1}{N}\sum_{m\in\mathcal{S}^{(k,i)}}
    \sum_{l=0}^{B-1}(1+\lr L)^BL \leq e^{1/10}BL, \label{eq:bd_norm_T}
\end{align}
since $\nbr{\hbbl_{m,l}^{(k,i-1)}} \leq L$ and $\lr\leq1/(10LB\kbar)$. We can bound \Cref{eq:9-0-3} as 
\begin{align}
    & \nbr{\expt_k[\rb_2^{(k,0)}]} \leq (1+e^{1/10}\lr BL)^\kbar e^{1/10}BL \sum_{i=1}^{\kbar-1}\nbr{\expt_k \left[\sum_{j=1}^i \frac{1}{N} \sum_{m\in\st^{(k,j)}}\sum_{l=0}^{B-1}\nabla F_{m,\pi_m^k(l)}(\wb^{(k,0)})\right]} \nn \\
    &= (1+e^{1/10}\lr BL)^\kbar e^{1/10}BL \sum_{i=1}^{\kbar-1}\nbr{\sum_{j=1}^i \frac{\kbar}{M} \sum_{m\in\sigma^{(j)}}\sum_{l=0}^{B-1}\nabla F_{m,\pi_m^k(l)}(\wb^{(k,0)})} \tag{$\because$ unbiased client sampling} \\
    &= \frac{(1+e^{1/10}\lr BL)^\kbar e^{1/10}B^2L\kbar}{M}\sum_{i=1}^{\kbar-1}\nbr{\sum_{j=1}^i\sum_{m\in\sigma^{(j)}}\nabla F_{m}(\wb^{(k,0)})} \nn \\
    & \leq {\text{exp}(e^{1/10}/10)e^{1/10}B^2L}\sum_{i=1}^{\kbar-1}\left(i\nbr{\nabla F(\wb^{(k,0)})}+i\alpha\right) \tag{using \cref{lem0-0}, and $\lr \leq \frac{1}{10 L B \kbar}$} \\
    & \leq \text{exp}(e^{1/10}/10)e^{1/10}B^2L \frac{\kbar(\kbar-1)}{2} \left[ \nbr{\nabla F(\wb^{(k,0)})} + \alpha \right]. \label{eq:r2_norm_bd}
\end{align}
For the last noise term, we have
\begin{flalign}
    &\begin{aligned}
        &\nbr{\expt_k[\rb_3^{(k,0)}]}\leq\\
        &\nbr{\expt_k \left[\frac{1}{N}\sum_{i=1}^{\kbar-1}\left(\prod_{j=\kbar}^{i+2}(\ibd-\lr\tbb^{(k,j)})\right)\tbb^{(k,i+1)}\left(\sum_{j=1}^i\sum_{m\in\st^{(k,j)}}\sum_{l=0}^{B-2}\left(\prod_{t=B-1}^{l+2}(\ibd-\lr\hbbt_{m,t}^{(k,i-1)})\right)\hbbt_{m,l+1}^{(k,i-1)}\sum_{j'=0}^l\nabla F_{m,\pi_m^k(j')}(\wb^{(k,0)})\right)\right]} \nn
    \end{aligned}\\
    &\begin{aligned}
        & \leq \sum_{i=1}^{\kbar-1}\nbr{\expt_k \left[\prod_{j=\kbar}^{i+2}(\ibd-\lr\tbb^{(k,j)})\right]}\nbr{\expt_k \left[\tbb^{(k,i+1)}\right]}\\
        & \qquad \times \nbr{\expt_k \left[\sum_{j=1}^i \frac{1}{N} \sum_{m\in\st^{(k,j)}} \sum_{l=0}^{B-2}\left(\prod_{t=B-1}^{l+2}(\ibd-\lr\hbbt_{m,t}^{(k,i-1)})\right)\hbbt_{m,l+1}^{(k,i-1)}\sum_{j'=0}^l\nabla F_{m,\pi_m^k(j')}(\wb^{(k,0)})\right]} \nn
    \end{aligned}\\
    & \leq (1+\lr e^{1/10}BL)^\kbar e^{1/10}BL \sum_{i=1}^{\kbar-1}\nbr{\sum_{j=1}^i \frac{\kbar}{M} \sum_{m\in\sigma^{(j)}}\sum_{l=0}^{B-2}\left(\prod_{t=B-1}^{l+2}(\ibd-\lr\hbbt_{m,t}^{(k,i-1)})\right)\hbbt_{m,l+1}^{(k,i-1)}\sum_{j'=0}^l\nabla F_{m,\pi_m^k(j')}(\wb^{(k,0)})} \tag{Using \eqref{eq:bd_norm_T}} \\
    & \leq \frac{\text{exp}(e^{1/10}/10)e^{1/10}BL\kbar}{M}\sum_{i=1}^{\kbar-1}\sum_{j=1}^i\sum_{m\in\sigma^{(j)}}\sum_{l=0}^{B-2} \expt_k \left[ \nbr{\prod_{t=B-1}^{l+2}(\ibd-\lr\hbbt_{m,t}^{(k,i-1)})}\nbr{\hbbt_{m,l+1}^{(k,i-1)}}\nbr{\sum_{j'=0}^l\nabla F_{m,\pi_m^k(j')}(\wb^{(k,0)})} \right] \tag{Using $\lr \leq \frac{1}{10 B L \kbar}$} \\
    &\leq \frac{(1+\lr L)^B\text{exp}(e^{1/10}/10)e^{1/10}BL^2\kbar(\kbar-1)}{M} \sum_{j=1}^\kbar \sum_{m\in\sigma^{(j)}} \sum_{l=0}^{B-2} \expt_k \nbr{\sum_{j'=0}^l \nabla F_{m,\pi_m^k(j')}(\wb^{(k,0)})} \tag{Using $\nbr{\hbbt_{m,l+1}^{(k,j)}} \leq L$, for all $j$}\\
    &\leq \frac{\text{exp}(e^{1/10}/10)e^{1/5}BL^2\kbar(\kbar-1)}{M}\sum_{j=1}^\kbar\sum_{m\in\sigma^{(j)}}\sum_{l=0}^{B-2}\left(\nub\sqrt{8(l+1)\log\left(\frac{4MBK}{\delta}\right)}+(l+1)\nbr{\nabla F_m(\wb^{(k,0)})}\right) \tag{Lemma 8 in \cite{yun2022shuf}} \\
    &\begin{aligned}
        & \leq \frac{2\text{exp}(e^{1/10}/10)e^{1/5}B(B^{3/2}-1)L^2\kbar(\kbar-1)}{3}\nub\sqrt{8\log\left(\frac{4MBK}{\delta}\right)} \nn \\
        & \quad +
        \frac{\text{exp}(e^{1/10}/10)e^{1/5}B^2(B-1)L^2\kbar(\kbar-1)}{2M}\sum_{j=1}^\kbar\sum_{m\in\sigma^{(j)}}\nbr{\nabla F_m(\wb^{(k,0)})}
    \end{aligned}\\
    &\begin{aligned}
        & \leq \text{exp}(e^{1/10}/10)e^{1/5}B L^2\kbar(\kbar-1) \left[\frac{2(B^{3/2}-1)}{3}\nub\sqrt{8\log\left(\frac{4MBK}{\delta}\right)}+
        \frac{B(B-1)\nu}{2} +
        \frac{B(B-1)}{2}\nbr{\nabla F(\wb^{(k,0)})} \right].
    \end{aligned}
    \label{eq:r3_norm_bd}
\end{flalign}
Finally, using \eqref{eq:r1_norm_bd}, \eqref{eq:r2_norm_bd}, \eqref{eq:r3_norm_bd}, in
\begin{flalign}
& \nbr{\expt_k[\rb^{(k,0)}]}\leq \nbr{\expt_k[\rb_1^{(k,0)}]}+\nbr{\expt_k[\rb_2^{(k,0)}]}+\lr\nbr{\expt_k[\rb_3^{(k,0)}]} \nn 
\end{flalign}
we get the final bound.
\end{proof}

\begin{proof}[Proof of \cref{lem:error_normsq_shuffleSGD}]
Similar to how we bounded the norm of the noise terms, we can bound the norm square of the noise terms as following
\begin{flalign}
    & 3\expt_k \left[\nbr{\rb_1^{(k,0)}}^2\right] \nn \\
    &= 3\expt_k \left[\nbr{\frac{1}{N}\sum_{i=1}^\kbar\sum_{m\in\st^{(k,i)}}\sum_{l=0}^{B-2}\left(\prod_{t=B-1}^{l+2}(\ibd-\lr\hbbt_{m,t}^{(k,i-1)})\right)\hbbt_{m,l+1}^{(k,i-1)}\sum_{j=0}^l \nabla F_{m,\pi_m^k(j)}(\wb^{(k,0)})}^2\right] \nn \\
    & \leq 3 \kbar (B-1) \expt_k \left[\frac{1}{N}\sum_{i=1}^\kbar\sum_{m\in\st^{(k,i)}}\sum_{l=0}^{B-2} \nbr{\left(\prod_{t=B-1}^{l+2}(\ibd-\lr\hbbt_{m,t}^{(k,i-1)})\right)\hbbt_{m,l+1}^{(k,i-1)}\sum_{j=0}^l \nabla F_{m,\pi_m^k(j)}(\wb^{(k,0)})}^2\right] \tag{Using \cref{lem:jensens}, \cref{lem:sum_of_squares}}nn \\
    & \leq \frac{3\kbar^2(B-1)}{M}\sum_{i=1}^\kbar\sum_{m\in\sigma^{(i)}}\sum_{l=0}^{B-2}(1+\lr L)^{2B}L^2 \expt_k \nbr{\sum_{j=0}^l \nabla F_{m,\pi_m^k(j)}(\wb^{(k,0)})}^2 \nn \\
    & \leq \frac{6e^{1/5}\kbar^2L^2(B-1)}{M}\sum_{i=1}^\kbar\sum_{m\in\sigma^{(i)}}\sum_{l=0}^{B-2}\left((l+1)(8\log(4MBK/\delta)\nub^2+(l+1)^2\nbr{\nabla F_m(\wb^{(k,0)})}^2\right) \tag{Using Lemma 8 in \cite{yun2022shuf} and $\lr\leq1/(10BL\kbar)$} \nn \\
    & \leq 4e^{1/5}\kbar^2L^2B(B-1)^2 \left[ 6 \log(4MBK/\delta) \nub^2 + B \nu^2 + B \nbr{\nabla F(\wb^{(k,0)})}^2 \right], \label{eq:9-1-0}
\end{flalign}
and
\begin{flalign}
    &3\expt_k \left[\nbr{\rb_2^{(k,0)}}^2\right] \nn \\
    & =3\expt_k \left[\nbr{\frac{1}{N}\sum_{i=1}^{\kbar-1}\left(\prod_{j=\kbar}^{i+2}(\ibd-\lr\tbb^{(k,j)})\right)\tbb^{(k,i+1)}\left(\sum_{j=1}^i\sum_{m\in\st^{(k,j)}}\sum_{l=0}^{B-1}\nabla F_{m,\pi_m^k(l)}(\wb^{(k,0)})\right)}^2\right] \nn \\
    &\leq\frac{3(\kbar-1)}{N^2}\sum_{i=1}^{\kbar-1}\expt_k \left[ \nbr{\prod_{j=\kbar}^{i+2}(\ibd-\lr\tbb^{(k,j)})}^2 \nbr{\tbb^{(k,i+1)}}^2 \nbr{\sum_{j=1}^i\sum_{m\in\st^{(k,j)}}\sum_{l=0}^{B-1}\nabla F_{m,\pi_m^k(l)}(\wb^{(k,0)})}^2 \right] \nn \tag{\Cref{lem:jensens} and submultiplicativity of norms}\\
    &\leq\frac{3(\kbar-1)}{N^2}\sum_{i=1}^{\kbar-1}\expt_k \left[(1+\lr BLe^{1/10})^{2\kbar}(B^2L^2e^{1/5})\nbr{\sum_{j=1}^i\sum_{m\in\st^{(k,j)}}\sum_{l=0}^{B-1}\nabla F_{m,\pi_m^k(l)}(\wb^{(k,0)})}^2\right] \nn \tag{Using \eqref{eq:bd_norm_T}} \\
    &\leq\frac{3(\kbar-1)B^4L^2\text{exp}(e^{1/10}/5)e^{1/5}}{N^2}\sum_{i=1}^{\kbar-1}\expt_k \left[\nbr{\sum_{j=1}^i\sum_{m\in\st^{(k,j)}}\nabla F_{m}(\wb^{(k,0)})}^2\right] \nn \tag{Using $\lr\leq1/(10BL\kbar)$} \\
    &\leq\frac{3(\kbar-1)B^4L^2\text{exp}(e^{1/10}/5)e^{1/5}}{N^2} \left( \frac{1}{3} (\kbar-1) \kbar^2 N \left[ \left(\frac{M/\kbar-N}{M/\kbar-1}\right)\gamma^2 + 2N \left(\nbr{\nabla F(\wb^{(k,0)})}^2+{\alpha^2}\right) \right] \right) \nn \tag{Using \Cref{eq:3-2-2}}\\
    &= \frac{(\kbar-1)^2 \kbar^2 B^4L^2\text{exp}(e^{1/10}/5)e^{1/5}}{N} \left[ \left(\frac{M/\kbar-N}{M/\kbar-1}\right)\gamma^2 + 2N \left(\nbr{\nabla F(\wb^{(k,0)})}^2+{\alpha^2}\right) \right]. \label{eq:9-1-1}
\end{flalign}
and using the same techniques to bound \Cref{eq:9-1-0} and \Cref{eq:9-1-1} we have the following:
\begin{flalign}
    &\begin{aligned}
    &3\lr^2\expt_k \left[\nbr{\rb_3^{(k,0)}}^2\right]=3\lr^2\expt_k \left[\left\|\frac{1}{N}\sum_{i=1}^{\kbar-1}\left(\prod_{j=\kbar}^{i+2}(\ibd-\lr\tbb^{(k,j)})\right)\tbb^{(k,i+1)}\right.\right.\\
    & \qquad \left.\left.\times\left(\sum_{j=1}^i\sum_{m\in\st^{(k,j)}}\sum_{l=0}^{B-2}\left(\prod_{t=B-1}^{l+2}(\ibd-\lr\hbbt_{m,t}^{(k,i-1)})\right)\hbbt_{m,l+1}^{(k,i-1)}\sum_{j'=0}^l\nabla F_{m,\pi_m^k(j')}(\wb^{(k,0)})\right)\right\|^2\right] \nn
    \end{aligned}\\
    &\begin{aligned}
    &\leq\frac{3\lr^2(\kbar-1)}{N^2}\sum_{i=1}^{\kbar-1}\expt_k \left[\left(\left\|\prod_{j=\kbar}^{i+2}(\ibd-\lr\tbb^{(k,j)})\right\|\left\|\tbb^{(k,i+1)}\right\|\right.\right.\\
    & \qquad \left.\left.\times\left\|\sum_{j=1}^i\sum_{m\in\st^{(k,j)}}\sum_{l=0}^{B-2}\left(\prod_{t=B-1}^{l+2}(\ibd-\lr\hbbt_{m,t}^{(k,i-1)})\right)\hbbt_{m,l+1}^{(k,i-1)}\sum_{j'=0}^l\nabla F_{m,\pi_m^k(j')}(\wb^{(k,0)})\right\|\right)^2\right] \nn
    \end{aligned}\\
    &\begin{aligned}
    &\leq\frac{3\lr^2(\kbar-1)}{N^2}\sum_{i=1}^{\kbar-1}(1+\lr BLe^{1/10})^{2\kbar}(B^2L^2e^{1/5})\\
    & \qquad \times (iN(B-1))\sum_{j=1}^i \expt_k \left[ \sum_{m\in\st^{(k,j)}} \sum_{l=0}^{B-2}\left\|\left(\prod_{t=B-1}^{l+2}(\ibd-\lr\hbbt_{m,t}^{(k,i-1)})\right)\hbbt_{m,l+1}^{(k,i-1)}\sum_{j'=0}^l\nabla F_{m,\pi_m^k(j')}(\wb^{(k,0)})\right\|^2\right] \nn
    \end{aligned}\\
    &\begin{aligned}
    &\leq\frac{3\lr^2(\kbar-1)\kbar B^2(B-1)L^2\text{exp}(e^{1/10}/5)e^{1/5}}{M}\\
    & \qquad \times\sum_{i=1}^{\kbar-1}i\sum_{j=1}^i\sum_{m\in\sigma^{(j)}}\sum_{l=0}^{B-2}\expt_k \left[\left\|\left(\prod_{t=B-1}^{l+2}(\ibd-\lr\hbbt_{m,t}^{(k,i-1)})\right)\hbbt_{m,l+1}^{(k,i-1)}\sum_{j'=0}^l\nabla F_{m,\pi_m^k(j')}(\wb^{(k,0)})\right\|^2\right] \nn
    \end{aligned}\\
    &\begin{aligned}
    &\leq\frac{3\lr^2(\kbar-1)\kbar B^2(B-1)L^2\text{exp}(e^{1/10}/5)e^{1/5}}{M}\\
    & \qquad \times\sum_{i=1}^{\kbar-1}i\sum_{j=1}^i\sum_{m\in\sigma^{(j)}}\sum_{l=0}^{B-2}(1+\lr L)^{2B}L^2\expt_k \left[\left\|\sum_{j'=0}^l\nabla F_{m,\pi_m^k(j')}(\wb^{(k,0)})\right\|^2\right] \nn
    \end{aligned}\\
    &\begin{aligned}
    &\leq\frac{6\lr^2(\kbar-1)\kbar B^2(B-1)L^4\text{exp}(e^{1/10}/5)e^{2/5}}{M}\\
    & \qquad \times\sum_{i=1}^{\kbar-1}i\sum_{j=1}^i\sum_{m\in\sigma^{(j)}}\sum_{l=0}^{B-2}\left(8(l+1)\log{(4MBK/\delta)}\nub^2+(l+1)^2\nbr{\nabla F_m(\wb^{(k,0)})}^2\right) \nn
    \end{aligned}\\
    &\begin{aligned}
    &\leq 12\lr^2(\kbar-1)^2\kbar^2 B^3(B-1)^2L^4\text{exp}(e^{1/10}/5)e^{2/5}\log{(4MBK/\delta)}\nub^2 \\
    & \qquad +\frac{\lr^2(\kbar-1)^2\kbar^2 B^4(B-1)^2L^4\text{exp}(e^{1/10}/5)e^{2/5}}{M}\sum_{j=1}^\kbar\sum_{m\in\sigma^{(j)}}\nbr{\nabla F_m(\wb^{(k,0)})}^2 \nn
    \end{aligned}\\
    &\begin{aligned}
    &\leq 2\lr^2(\kbar-1)^2\kbar^2 B^3(B-1)^2L^4\text{exp}(e^{1/10}/5)e^{2/5} \left[ 6 \log{(4MBK/\delta)}\nub^2 + B \lp \nu^2 + \nbr{\nabla F(\wb^{(k,0)})}^2 \rp \right].
    \end{aligned}
    \label{eq:9-1-2}
\end{flalign}
Using \Cref{eq:9-1-0}, \Cref{eq:9-1-1}, and \Cref{eq:9-1-2} we have
\begin{flalign}
    & L\lr^4\left(3\expt_k \left[\nbr{\rb_1^{(k,0)}}^2\right]+3\expt_k \left[\nbr{\rb_2^{(k,0)}}^2\right]+3\lr^2\expt_k \left[\nbr{\rb_3^{(k,0)}}^2\right]\right) \nn \\
    &\leq \left(24e^{1/5}\lr\kbar LB+\frac{2\lr\kbar BL\text{exp}(e^{1/10}/5)e^{2/5}}{25} \right)\lr^3\kbar L^2 (B-1)^2\log{(4MBK/\delta)}\nub^2 \nn \\
    & \quad +\left(4e^{1/5}\lr\kbar LB+\frac{\lr(\kbar-1) BL\text{exp}(e^{1/10}/5)e^{2/5}}{50}\right)\lr^3\kbar L^2B(B-1)^2\nu^2 \nn \\
    & \quad + \left(4e^{1/5}\lr^3\kbar L^3B^3+12\lr^3\kbar^3B^3L^3\text{exp}(e^{1/10}/5)e^{1/5}+\frac{\lr^3\kbar B^3L^3\text{exp}(e^{1/10}/5)e^{2/5}}{50}\right)\lr\kbar B\nbr{\nabla F(\wb^{(k,0)})}^2 \nn \\
    & \quad + \frac{6(\kbar-1)\lr^4\kbar^3B^4L^3\text{exp}(e^{1/10}/5)e^{1/5}}{N}\left(\frac{M/\kbar-N}{M/\kbar-1}\right)\gamma^2+4\kbar^2(\kbar-1)^2B^4L^3\lr^4\alpha^2\text{exp}(e^{1/10}/5)e^{1/5} \nn \\
    &\leq \frac{31\lr^3\kbar L^2 (B-1)^2}{10}\log{(4MBK/\delta)}\nub^2+\frac{\lr^3\kbar L^2B(B-1)^2\nu^2}{2}+\frac{3\lr\kbar B}{100}\nbr{\nabla F(\wb^{(k,0)})}^2 \nn \\
    & \quad +\frac{6(\kbar-1)\lr^4\kbar^3B^4L^3\text{exp}(e^{1/10}/5)e^{1/5}}{N}\left(\frac{M/\kbar-N}{M/\kbar-1}\right)\gamma^2+4\kbar^2(\kbar-1)^2B^4L^3\lr^4\alpha^2\text{exp}(e^{1/10}/5)e^{1/5}, \label{eq:9-2-3}
\end{flalign}
which completes the proof.
\end{proof}

\end{document}